\documentclass[lettersize,journal]{IEEEtran}

\IEEEoverridecommandlockouts

\usepackage{amsmath,graphicx,epstopdf,amssymb,amsthm}
\usepackage[most]{tcolorbox}

\usepackage{dsfont}
\usepackage{url}
\usepackage{color}
\usepackage{epstopdf}
\usepackage{verbatim}
\usepackage{cuted}
\usepackage{lipsum,graphicx,subcaption}
\usepackage[labelformat=simple]{subcaption}
\usepackage[utf8]{inputenc}
\usepackage[english]{babel}
\usepackage{caption}
\usepackage{enumitem}
\DeclareCaptionLabelFormat{lc}{\MakeLowercase{#1}~#2}
\usepackage{balance}
\usepackage{float}
\usepackage{mathtools}

\DeclarePairedDelimiter\floor{\lfloor}{\rfloor}

\newtheorem{assumption}{Assumption}

\allowdisplaybreaks

\usepackage[font=small]{caption}
\usepackage{enumitem}

\usepackage{tabularx,booktabs}
\newcolumntype{Y}{>{\centering\arraybackslash}X}
\usepackage{multirow}
\usepackage{algorithm}
\usepackage{algorithmic}
\allowdisplaybreaks

\usepackage{amsmath,bm,thmtools}

\newcolumntype{M}[1]{>{\centering\arraybackslash}m{#1}}

\begin{document}

% \title{Towards Serverless Semi-Decentralized Federated Learning in Edge/Fog Networks}
\title{Towards Serverless Semi-Decentralized Federated Learning with Heterogeneous Optimizers}
%% manipulating topology in severless semi-decentralized federated learning with heterogeneous optimizers 

% \title{Serverless-Training in Federated Learning with Heterogeneous Optimizers: A Semi-Decentralized Approach}
% \title{Towards Serverless-Training with Heterogeneous Optimizers: a Semi-Decentralized Federated Learning Approach}

% \title{Semi-Decentralized Federated Learning with Heterogeneous Optimizers via Cluster Formation}
% Semi-Decentralized Federated Learning with Heterogeneous Optimizers: Achieving Serverless Training via Cluster Formation

\author{Su Wang~\IEEEmembership{Member,~IEEE}, Mung Chiang,~\IEEEmembership{Fellow,~IEEE}, and H. Vincent Poor,~\IEEEmembership{Life Fellow,~IEEE}
\thanks{Su Wang and H. Vincent Poor are with the Department of Electrical and Computer Engineering, Princeton University, Princeton, NJ, USA. Email: \{hw5731,~poor\}@princeton.edu.
Mung Chiang is with the Department of Electrical and Computer Engineering, Purdue University, West Lafayette, IN, USA. Email: chiang@purdue.edu.} 
}

\maketitle

\begin{abstract}
    We investigate cluster formation, involving the number and composition of clusters, in decentralized federated learning (FL) with heterogeneous machine learning (ML) optimizers. 
    While clustering in centralized FL has enabled scalability and resource savings, its value and development in fully decentralized environments have yet to be explored. 
    Optimizing cluster formation in such environments is challenging, especially due to the complex coupling between network graph structures, local data heterogeneity, and different local ML model optimizers. 
    To address these challenges, we propose serverless semi-decentralized FL (SSD-FL), a methodology requiring no persistent server infrastructure. 
    In SSD-FL, cluster formation occurs via a lightweight, one-time device-to-device (D2D) initialization phase, after which actual ML model training (alongside consensus and convergence processes) is fully serverless. 
    Functionally, SSD-FL segments global rounds into intra-cluster and inter-cluster regimes, ensuring global convergence and consensus through novel "effective loss functions" that integrate device-specific ML optimizers with network graph-based regularization.
    Next, SSD-FL leverages the consensus gap via the Cheeger inequality to develop an iterative clustering algorithm evaluated against our derived convergence and consensus bounds, which incorporate a unique scoring metric to quantify data and optimizer heterogeneity across devices. 
    Finally, experimental evaluation against three categories of decentralized FL methodologies validate that SSD-FL improves both convergence speeds and communication efficiency across various network graphs, datasets, and local optimizer regimes. 

\end{abstract}

% \begin{IEEEkeywords}
%     XXX
% \end{IEEEkeywords}

\section{Introduction}
\label{sec:intro}

%%%%%% some key feedback points from sigmetrics
% {\color{blue} Need to change - focus and distill the feedback from the sigmetrics reviews}

% {\color{blue} Some strategies: 

% \begin{itemize}
%     \item Change the system model to centralized FL, every device has a connection to the central server
%     \item The premise would have to be that device-to-server transmissions cost far more than device-to-device communications 
%     \item System model would become 
%     \item Advantages: immediately (i) fixes the global consensus problem, (ii) obsoletes the assumption on starting consensus error
%     \item Advantages: (iii) prevents people from challenging the centrally determined cluster formation (prior to FL operation) 
%     \item Disadvantages: (i) it assumes a server, so it returns to a relatively standard initial star topology -- this minimizes the contribution, but only a bit
%     \item Disadvantages: (ii) method's contribution would be mainly communication savings (not necessarily better performance for all time, but better performance per global communication round)
% \end{itemize}
% }

%%% what I dislike so far, is that you haven't established a clear and pressing problem that needs to be addressed

%% high level intro if you want
Based on the edge/fog network, FL methodologies~\cite{yazdinejad2024robust,hallaji2024decentralized,wang2024device} are partitioned into centralized and decentralized FL~\cite{pei2024review}, as shown in Fig.~\ref{fig:intro_diag}.
% FL methodologies~\cite{} can be partitioned into centralized and decentralized FL, both shown in Fig.~\ref{}, based on the edge/fog network architecture. 
While centralized FL relies on a server to coordinate ML model training processes~\cite{wang2023toward,wang2019adaptive}, decentralized FL~\cite{yuan2024decentralized} relies on D2D communications to incrementally propagate ML model updates, eventually yielding both consensus and convergence. 
% Common to both classes of FL, the underlying edge/fog network is characterized by unbalanced and highly heterogeneous D2D links
%%% introduce terminology - non-hierarchical clustering in centralized FL referred to as semi-decentralized FL 
However, in large-scale edge/fog networks, both classes of FL may struggle as devices may be far from each other and the server, specifically for centralized FL, leading to scalability challenges in terms of latency, convergence, and consensus.

In response, existing works~\cite{yemini2022semi,lin2021semi,ali2025sdflmq} introduce cluster formation, in which devices are grouped based on data distributions or network properties, to improve the scalability of FL in large-scale edge/fog networks. 
Referred to as semi-decentralized FL (SD-FL), these methodologies demonstrate faster convergence and improved efficiency by leveraging pre-defined clusters. But as clusters are given a priori, we still do not understand the properties for effective cluster formation, i.e., the number of clusters and the devices within them. 
% there is limited intuition for cluster construction, i.e., the number of clusters and the devices within them. %how many clusters and which devices in each cluster.

Effective cluster formation requires balance between (i) macro-level network properties, such as graph connectivity and varying device densities, and (ii) micro-level device properties, such as heterogeneous datasets and ML optimizers (both of which influence collaborative ML model training in FL~\cite{wang2023toward,wang2021novel,liu2025accelerating}). 
These challenges are exacerbated in decentralized edge/fog networks, as such networks lack continuous central server synchronization. 
% The issues facing effective cluster formation are exacerbated in decentralized FL, as the lack of continuous central server synchronization makes it harder to balance between (i) macro-level network properties, such as graph connectivity, and (ii) micro-level device properties, such as heterogeneous datasets and ML optimizers. 
% owing to the lack of central server synchronization. 
Instead, decentralized edge/fog networks are commonly treated as a single cluster~\cite{tang2022gossipfl,nguyen2024decentralized}, which may be inefficient as large-scale edge/fog networks exhibit extensive heterogeneity. 
% For instance, it may be more efficient to have multiple clusters so that overall cluster data distributions are similar to each other or so that clusters have a similar degree of graph connectivity. 
% These observations motivate a transformation from decentralized FL to serverless semi-decentralized FL (SSD-FL), where the network is partitioned into multiple clusters with intra-cluter and inter-cluster mechanics. %that interact via cluster-to-cluster (C2C) synchronization. 
% This transformation is particularly relevant in practical systems such as decentralized energy grids [11], [12], where localized consensus reduces the overall cost of energy-trading decisions, or ad hoc wireless sensor networks (such as those for disaster recovery or multi-domain unmanned vehicle networks), where clustering can account for highly non-uniform node densities. 
It may be more efficient to have multiple clusters so that overall cluster data distributions are similar to each other or so that clusters have a similar degree of graph connectivity. 
To contextualize these ideas, consider the following potential applications: 
\begin{itemize}
    \item \textbf{Decentralized Energy Grids} rely on D2D communications without global/central control, such as those involving D2D solar energy trading~\cite{tushar2020peer,li2023peer}. 
    % Therefore, to obtain ML-based insights over such edge/fog networks, devices could rely on decentralized FL~\cite{husnoo2024decentralized,liu2022federated}. 
    % However, decentralized FL does not distinguish among D2D links. 
    Leveraging decentralized FL in these types of edge/fog networks can be problematic, owing to highly heterogeneous D2D links as well as densities, e.g., new home neighborhoods with local energy storage vs older subdivisions. %% faster local/relevant consensus + simplified network-wide coordination 
    By carefully designing device clusters, SSD-FL can enable both (i) faster localized/relevant consensus, minimizing the costs of frequent long-distance or expensive D2D links, and (ii) simplify network-wide coordination, as integrating (locally) synchronized clusters may be easier than a larger number of uncoordinated edge/fog devices. 
    % By carefully designing device clusters, SSD-FL can enable both (i) localized consensus, minimizing the costs of long-distance or expensive D2D links, and (ii) simplified network-wide coordination, as synchronization over a smaller number of clusters - and subsequent cluster-to-cluster (C2C) synchronization - is easier than over a significantly larger number of uncoordinated edge/fog devices. 
    % The proposed SSD-FL methodology would enable smart grid intelligence with minimal oversight, allowing individual devices to run at their most optimal structural settings via careful cluster formation. 
    \item 
    \textbf{Ad Hoc Wireless Sensor Networks} for disaster recovery communications~\cite{pogkas2005ad,wang2024information} or multi-domain unmanned vehicle networks~\cite{qian2022joint,zhang2025fmd} are similarly massively distributed and reliant on highly heterogeneous D2D communication links. 
    In the case of natural disaster communications~\cite{pogkas2005ad,wang2024information}, edge/fog networks are characterized by regions of high and low density devices and D2D connections, such as earthquake hotspots interspersed between rural plains or UAVs/UGVs with relay devices covering their communication limitations. % groups communicating with each other. 
    % highly distributed - so D2D from one ``corner" of the network to another may be highly involved 
    SSD-FL, via careful cluster formation, can enable decentralized edge/fog networks to leverage periodic inter-cluster communications, rather than frequent and higher total latency global synchronizations, and thus improve ML training convergence rates overall. 
    % SSD-FL, via careful cluster formation, minimizes enables faster local updates within while maintaining the advantages of periodic inter-cluster connections. 
\end{itemize}

%% condensed applications (rather than expanded applications as in the itemized list above)
% In practical decentralized edge/fog networks, potential applications include next generation energy grids~\cite{tushar2020peer,li2023peer}, where localized consensus can reduce the overall cost of energy-trading decisions, or ad hoc wireless sensor networks (such as those for disaster recovery~\cite{pogkas2005ad,wang2024information} and multi-domain unmanned vehicle networks~\cite{qian2022joint,zhang2025fmd}), where clustering can account for highly heterogeneous device density. 

%%%% introductory diagram here - TBD
\begin{figure}[t]
    \centering
    \includegraphics[width=0.98\linewidth]{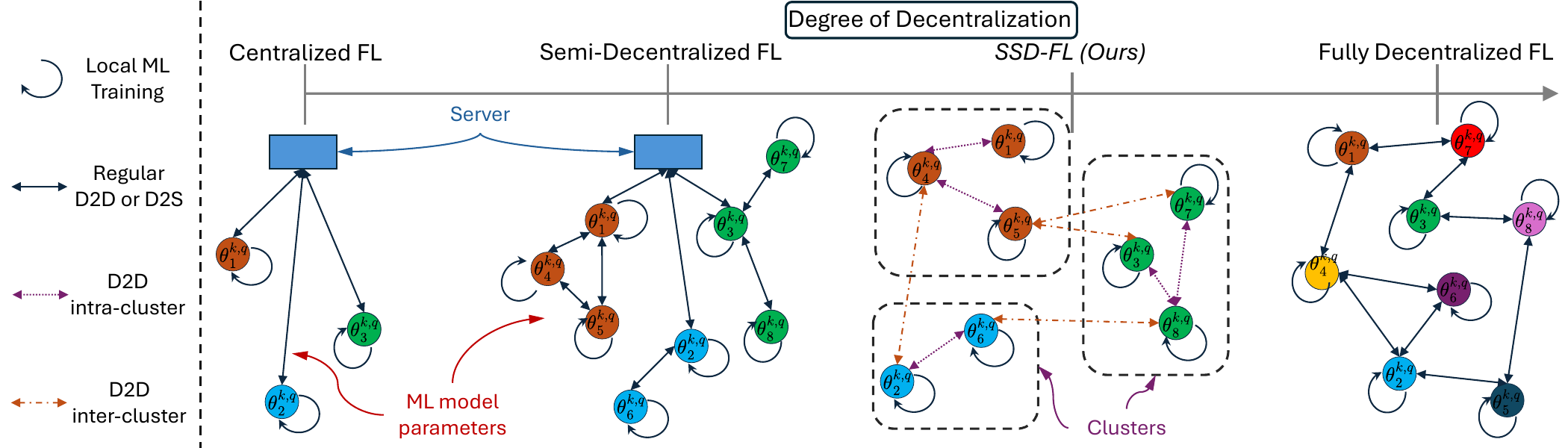}
    % Centralized, semi-decentralized, serverless semi-decentralized, and fully decentralized FL as a spectrum 
    \caption{FL architectures with respect to the degree of network decentralization. 
    From left to right, FL shifts from control by a centralized, global server to fully decentralized devices. SSD-FL introduces clusters in decentralized networks, offering heterogeneous D2D cooperation density in decentralized edge/fog networks. % a path for heterogeneous D2D coooperation density
    }
    \label{fig:intro_diag}
    \vspace{-4mm}
\end{figure}

To support these sample applications and beyond, we seek to answer how and when to form clusters in decentralized edge/fog networks. 
In this regard, cluster formation involves understanding two deeply coupled trade-offs: (i) global-level network structure, which influences the optimal number of clusters and (ii) local-level cluster composition, which determines the specific devices within each cluster. 
The number of clusters directly controls the rate of local convergence but at the cost of global consensus, e.g., more clusters means faster local training but requires many rounds of multi-hop D2D communications to attain global consensus. 
Conversely, device selection within each cluster defines the local communication topology, data distributions (degree of non-i.i.d.), and set of local ML model optimizers, all of which influence intra-cluster convergence properties. 
Thus, to achieve effective cluster formation, we must consider both the number of clusters and their composition jointly.

Our proposed SSD-FL methodology addresses these coupled trade-offs without relying on persistent server infrastructure. 
Here, serverless refers specifically to model training, where all coordination occurs entirely via D2D communications. 
While SSD-FL requires a lightweight, one-time coordination of network devices prior to training, this is in fundamental contrast to traditional centralized as well as SD-FL methodologies, which require continuous server management (i.e., the server is core to the distributed training process). 
Therefore, by formalizing this serverless cluster formation, SSD-FL bridges a core gap towards fully decentralized, serverless FL across large-scale edge/fog networks.

\subsection{Outline and Summary of Contributions}
In the following, we begin by reviewing relevant literature in Sec.~\ref{sec:related_work} and present SSD-FL's system model as well as theoretical background in Sec.~\ref{sec:prelim}. 
Then, we derive the convergence and consensus properties of our proposed SSD-FL methodology in Sec.~\ref{sec:theory} followed by the cluster algorithm of SSD-FL in Sec.~\ref{sec:cluster_formation}.
Subsequently, we validate SSD-FL relative to baselines experimentally in Sec.~\ref{sec:experiments} before summarizing the key takeaways in Sec.~\ref{sec:conclusion}.
We summarize our key contributions as follows:

\begin{itemize}
    % \item \textbf{Formulation of SSD-FL:} 
    \item \textbf{Cluster-driven approach to decentralized FL:}
    We introduce SSD-FL, a methodology that decomposes the structure of FL in serverless edge/fog networks, via cluster formation, into intra-cluster and inter-cluster regimes. 
    Towards principled cluster formation, SSD-FL proposes ``effective loss functions" with explicit terms for (i) network structure heterogeneity via regularization with cluster and global Laplacian matrices and (ii) different device ML optimizers (i.e., SGD, SGD with momentum, and proximal SGD). 

    \item \textbf{Integrated intra-cluster and inter-cluster regime convergence:} 
    We characterize the theoretical convergence rate and consensus gaps for intra-cluster and inter-cluster regimes, demonstrating (i) non-convex first-order stationary points with heterogeneous optimizers and (ii) connectivity-driven convergence as a result of regularization via graph Laplacian matrices.
    These results follow from our effective loss functions, whose simultaneous treatment of momentum and proximal terms as well as graph regularization requires extending standard smoothness arguments. 
    Finally, we show integrated (combined intra-cluster and inter-cluster) convergence for SSD-FL, in which the ML processes as well as network/clusters' graph structure are explicit.
    % These bounds explicitly depend on both ML processes as well as network/clusters graph structure. 

    % \item \textit{Training of heterogeneous clusters and optimizers:} %Heterogeneity of optimizers and  }
    % % Graph structure induces regularized loss functions 
    % % Heterogeneous optimizers just a feature of general edge/fog heterogeneity
    % To prevent devices' ML model parameters from diverging in decentralized edge/fog networks, SSD-FL leverages loss functions regularized with respect to the cluster-specific and network-wide Laplacian matrices. 
    % The theoretical convergence and consensus of SSD-FL are also derived with respect to Laplacian-regularized loss functions.

    % \item \textit{Convergence-driven cluster formation:} 
    \item \textbf{Consensus-convergence guided cluster formation:}
    Our proposed SSD-FL methodology determines both the optimal number of clusters and their constituents via a one-time, pre-deployment initialization step using network graph and device properties. 
    % SSD-FL is among the first to propose a methodology for cluster formation in decentralized FL that determines both the number of clusters and their constituent devices. %he first to propose a methodology for cluster formation in decentralized FL. As far as we are aware
    % We do so by leveraging the derived theoretical consensus conditions for SSD-FL, from which we develop cluster and graph conductance thresholds via Cheeger's inequality. 
    Leveraging our derived theoretical consensus conditions, we map these system characteristics to explicit cluster and graph conductance thresholds via Cheeger’s inequality, which are then applied to partition the network and refine clusters.
    % compositions.
    
    % By leveraging theoretical convergence conditions, we develop cluster and graph conductance thresholds, 
    % Theoretical results produce convergence condition -> convert convergence condition into a conductance factor, leveraging Cheeger's inequality, we can then produce an algorithm for cluster formation

    % \item \textit{SSD-FL over multiple empirical tasks:} 
    \item \textbf{Experimental validation of SSD-FL:}
    We evaluate SSD-FL in terms of ML training speed and quality in networks of varying architecture, size, connectivity, and heterogeneity (i.e., uniform and unique local ML optimizers). 
    These experiments demonstrate that SSD-FL offers both improved final accuracies as well as faster convergence relative to three families of decentralized FL baselines on FMNIST and CIFAR10 datasets. 
    
    %experiments involving both traditional image classification as well as regression tasks. 
    %\textit{Experimental Evaluation:} 

\end{itemize}

\section{Related Work}
\label{sec:related_work}

%%% how to structure?
%% clustering approaches for FL in centralized edge/fog networks
% semi-decentralized + hierarchical FL, etc...

%% current advances in decentralized FL
% and frame current literature in terms of 

We contextualize SSD-FL with respect to clustering methodologies for centralized FL and relevant advances in decentralized FL. % works on decentralized FL. 
In particular, we want to emphasize that existing literature has yet to develop methodologies for exact cluster formation (i.e., number of clusters and their devices) in FL, even in centralized edge/fog networks. 
Therefore, our research aims to understand effective clustering and subsequently bridge the gap between deliberate network structure manipulation in centralized FL and fully decentralized edge/fog scenarios. 

\subsection{Clustering for centralized FL}
%% clustering approaches for FL in centralized edge/fog networks
% semi-decentralized + hierarchical FL, etc...
The motivation for clustering in centralized FL stems from large-scale edge/fog networks, in which edge devices may be far away from the central server.  %~\cite{lin2021semi}
Rather than incur high latency from mandating device-to-server transmissions, existing literature proposed semi-decentralized FL~\cite{lin2021semi,ali2025sdflmq,yemini2022semi,weng2024coded} in which devices are grouped into clusters. 
Within these clusters, devices follow gossip-based protocols (similar to decentralized FL methodologies~\cite{sun2022decentralized,koloskova2020unified}) in order to achieve intra-cluster consensus, after which, a single device in every cluster would communicate with the server to complete global aggregations. 
% impact of this
In effect, these techniques extend the reach of centralized FL, connecting the ``edge" of large-scale networks while lowering latency for the FL process overall. 
% connecting more devices at the edge of large-scale networks, and lowering latency for the FL process overall. 

There is a similar line of research in hierarchical FL~\cite{liu2022hierarchical,chen2025hierarchical,wang2022uav,haghighifard2025hierarchical}. 
While these methodologies also involve cluster formation, their clusters locally function as a star topology, with one device managing and synchronizing other devices, and thereby reduce the total network device-to-server communication constraints.
However, the underlying scalability issues of large-scale edge/fog networks remain, especially if devices are distant to the ``central" device of their assigned cluster. 
Moreover, only a limited set of possible D2D connections are used, i.e., only those D2D connections involving the ``central" device of each cluster, and, therefore, there is an opportunity for further performance/latency gains by integrating D2D cooperation throughout clusters. 

More broadly, existing methodologies on semi-decentralized and hierarchical FL depend on restrictive assumptions for clusters. Typically, clusters are either pre-determined~\cite{lin2021semi,sun2022semi,weng2024coded} or derived from purely statistical properties (e.g., cosine similarity~\cite{ali2025sdflmq} or training progression~\cite{wang2022accelerating,gong2023towards}). 
Even though such approaches do yield improvements to latency and convergence relative to standard, centralized FL, they nonetheless neglect the structural heterogeneity aspects of large-scale edge/fog networks, i.e., the number and quality of their available D2D connections. 
% In such networks, devices not only differ in terms of their local computational and statistical characteristics but also in their strucuunderlying infrastructure, i.e., the number and quality of their available D2D connections. 
Clustering solely based on devices' computational and statistical characteristics overlooks these underlying network properties, which could otherwise be leveraged for performance and scalability improvements.
% As a result, clustering solely on statistical similarity risks overlooking structural properties of the underlying network that could otherwise be leveraged for performance and scalability gains.

In light of these limitations, our proposed SSD-FL methodology aims for cluster formation that integrates both network structure and devices' statistical heterogeneity, including choice of local ML optimizer, via its introduction of effective loss functions. 
% whereas prior work ignores network topology in cluster formation, SSD-FL explicitly integrates connectivity quality into its effective loss functions
Moreover, as a further distinction from the above lines of research, SSD-FL is designed for fully decentralized edge/fog networks, for which global synchronization is not possible. 
Our approach of interspersed intra-cluster and inter-cluster regimes also allows SSD-FL to leverage the advantages of clustering (i.e., faster local convergences and intermediate consensus) while maintaining both the flexibility and scalability of fully decentralized edge/fog networks. 
To better highlight these distinctions, we next describe SSD-FL in the context of decentralized FL methodologies.

%%%
% semi-decentralized works
% \cite{lin2021semi} the original work on semi-decentralized FL, devices communicate with their edge server. edge servers communicate with the central server
% \cite{sun2022semi} similar to~\cite{lin2021semi}, except there is no central server. Instead, the edge servers will use D2D (similar to decentralized FL) and gossiping to achieve the global aggregation
% \cite{ali2025sdflmq} develops a more efficient communication protocol for \cite{lin2021semi}, again relying on a central server and edge servers that resemble hierarchical FL. Specifically, devices communicate with edge server, edge server communicates with the global/central server 
% \cite{weng2024coded} central server but edge devices can engage in D2D parameter sharing, big difference relative to~\cite{yemini2022semi} is the introduction of quantization within the parameter communications
% \cite{yemini2022semi} identical architecture to that of \cite{weng2024coded}, it was just the first work to have such an architecture

\subsection{Current advances in decentralized FL}
In decentralized FL research, existing literature can be categorized broadly based on directed or undirected D2D links. 
% While research on directed networks for FL~\cite{nguyen2024decentralized,nedic2014distributed,taheri2020quantized} is an important area of study, 
While directed networks for FL~\cite{nguyen2024decentralized,nedic2014distributed,taheri2020quantized} are an important line of research, we focus on the intersection of decentralized FL and undirected networks, which capture the two-way nature of wireless communications and naturally enable cluster-based network reorganization as the two sample applications in Sec.~\ref{sec:intro} suggest. 
For such undirected edge/fog networks, existing methodologies~\cite{wang2022matcha,liu2022decentralized,hashemi2021benefits} view their underlying network as a single cluster, where performance improvements are achieved primarily through the design of D2D communication sequencing, e.g., gossip protocol manipulation~\cite{tang2022gossipfl,nedic2009distributed}, periodic D2D communication~\cite{hashemi2021benefits,sun2022decentralized}, or irregular D2D communication~\cite{koloskova2020unified,zehtabi2024decentralized} after devices' perform local ML model training. 
% and, thus, exert control over the ML model training process primarily via the design of D2D communication sequencing, 
In this regard, existing approaches can be organized into three main segments: (i) synchronous, (ii) periodic, and (iii) stochastic decentralized FL. % training. %in fully decentralized FL 

% %% by default, existing works in decentralized FL all assume a single cluster. The primary way methodologies manipulate decentralized FL thus far is via the D2D cooperation sequencing after local devices perform ML model training
% For such undirected edge/fog networks, fully decentralized FL methodologies can be structured into three main segments: (i) synchronous, (ii) periodic, and (iii) stochastic distributed training. 

%% Explain the three main segments - 1 para, last sentence, what common weakness do they share?
% synchronous decentralized FL 
In synchronous decentralized FL~\cite{nguyen2024decentralized,shi2015extra}, D2D communication happens at every iteration to synchronize local ML models. While these methodologies leverage frequent mixing to produce convergence guarantees, they incur substantial D2D communication overhead, limiting scalability in large-scale edge/fog networks. 
% is typically the default methodology in which, within the same time instance, devices will train and subsequently transmit their local ML model parameters/gradients to their neighbors~\cite{nguyen2024decentralized,shi2015extra}. 
% periodic decentralized FL 
Periodic decentralized FL~\cite{sun2022decentralized,hashemi2021benefits} reduces D2D communication overhead by propagating updates only after several rounds of local ML model training.
Such approaches maintain convergence under standard smoothness assumptions and offer communication cost savings~\cite{sun2022decentralized,liu2022decentralized}, but introduce greater drift across devices, which can overfit locally and require more training time overall~\cite{yuan2024decentralized}. 
% stochastic decentralized FL -  stochastic training + D2D communications 
By contrast, stochastic decentralized FL methodologies~\cite{koloskova2020unified,zehtabi2024decentralized,bornstein2023swift} rely on arbitrary, random, or asynchronous operations, where the timing of local device training and/or D2D communications are dictated by randomness or hardware constraints.
% we use categorize decentralized FL approaches that involve arbitrary, random, or asynchronous D2D communication as stochastic decentralized FL methodologies~\cite{koloskova2020unified,zehtabi2024decentralized,bornstein2023swift}. 
% Stochastic decentralized FL methodologies do not have a distinct synchronization or ML model training phase, instead relying on random or hardware factors to dictate timing for both local ML model training and/or D2D communications. 
As such, these methodologies enable more functionality and integration in large-scale edge/fog networks, at the cost of predictability, leading to cases of inefficient resource use and inconsistent training overall. 

%%% maybe reword this final paragraph
% by contrast, SSD-FL aims to provide a complementary perspective - that of viewing the network as a set of connected components, rather than a single entity. 
% to this end, SSD-FL partitions devices into clusters, which operate via interspersed intra-cluster and inter-cluster regimes. 
% Similar to periodic and stochastic decentralized FL methodologies, these two regimes 
% Thus offering a more intuitive/natural way to manage heterogeneity in large-scale edge/fog networks. 
% %% can end paragraph here or add another sentence - in this manner, SSD-FL leverages network structure control as a core component of effective decentralized FL. 
SSD-FL aims to provide a complementary perspective to these existing lines of research via restructuring the underlying edge/fog network. 
By partitioning devices into clusters and managing clusters via interspersed intra-cluster and inter-cluster regimes, SSD-FL not only reduces D2D communication overhead, similar to periodic decentralized FL approaches, but also provides a more intuitive/natural way to manage heterogeneity in large-scale edge/fog networks, rather than the more general and unpredictable frameworks underlying stochastic methodologies. %enabling general randomness as for stochastic . 
% In this manner, SSD-FL introduces network structure manipulation as a core component of effective decentralized FL design. 
Thus, SSD-FL introduces network structure control as a core component of effective decentralized FL design.

% Large-scale edge/fog networks are highly distributed, consistently waiting for communications to propagate from one end of the network to the other can be very costly/incur high latency

% decentralized FL works 
% \cite{wang2022matcha}, link importance (as measured towards convergence speed) of arbitrary topologies in decentralized FL 
% % \cite{liu2022decentralized}, many local training rounds before a communication round in decentralized FL 
% \cite{zehtabi2024decentralized}, sporadic decentralized FL, i.e., sporadic communication between devices as well as sporadic computation at devices 

%% old structure -- incomplete
% \subsection{Decentralized, Semi-Decentralized, and Hierarchical FL}

%% main idea: existing methodologies 

%% remove clustering in networks, because too broad and we use it more as an application rather than an explicit or theoretical contribution
% \subsection{Clustering in Networks}
% bring in some graph theoretical works (i.e., expander graphs, consensus clusters, etc...) 

% \subsection{Meta-Learning}

% \subsection{Heterogeneous ML Optimizers}

\section{System Model}
\label{sec:prelim}

% {\color{blue}Because algorithm needs to estimate alpha and alphas, may need some coordinator - INITIALLY - before full decentralized ML processes go forth.}

In the following, we first describe our network model in Sec.~\ref{ssec:net_model}, the ML model training components in Sec.~\ref{ssec:ml_mechs}, and theoretical background in Sec.~\ref{ssec:theory_back}.

%%%% notes
\subsection{Network model} \label{ssec:net_model}
% Intra-cluster regimes and inter-cluster regimes induce different network architectures. 
% We depict both regimes in Fig.~\ref{fig:XXX}, with their ML model training processes explained in Sec.~\ref{ssec:ml_mechs}.
We model the edge/fog network as a graph $G = \{\mathcal{N}, \mathcal{E}\}$, where $\mathcal{N} = \{1, \cdots, N \}$ denotes the set of devices/nodes and $\mathcal{E}$ represents the set of weighted active D2D edges/links, with $(i,j) \in \mathcal{E}$ if device $i$ is able and willing to share ML model parameters with device $j$ and vice versa. 
Given any $i,j \in \mathcal{N}$, we assume that if the D2D link $(i,j)$ exists, then so does $(j,i)$. %though the weights of these two links may differ. 
Since SSD-FL follows different D2D communication structures within and across clusters, we use separate graphs: $\tilde{G} = \{ \mathcal{N}, \tilde{\mathcal{E}} \}$ for intra-cluster regimes and the full graph $G$ for inter-cluster regimes.
% , visualized in Fig.~\ref{fig:xxx}. 
For the rest of this paper, non-calligraphic font represents the size of the corresponding set, e.g., $N = \vert \mathcal{N} \vert$. 
SSD-FL aims to partition the network graph $G$ into a set of clusters or subgraphs $\mathcal{S} = \{1, \cdots, S \}$, with the subgraph for clusters $s \in \mathcal{S}$ defined as $\tilde{G}_s = \{ \mathcal{N}_s, \tilde{\mathcal{E}}_s \}$ and their union denoted by $\tilde{G}  = \{\mathcal{N}, \tilde{\mathcal{E}} \} \equiv \{\cup_{s \in \mathcal{S}} N_s, \cup_{s \in \mathcal{S}} \tilde{\mathcal{E}}_s \}$. %\cup_{s \in \mathcal{S}} G_s = 
Here, $\mathcal{N}_s \subset \mathcal{N}$, represents a subset of the network's nodes, while $\mathcal{E}_s \subset \mathcal{E}$ represents the weighted set of edges $(i,j), \forall i,j \in \mathcal{N}_s$. 
% Since edges are weighted, $\tilde{\mathcal{E}} \neq \mathcal{E}$ 
Moreover, the set of clusters $\mathcal{S}$ is connected such that, given any two clusters $s, s^{'} \in \mathcal{S}$, $s \neq s^{'}$, there is a path from at least one device $i \in \mathcal{N}_s$ to another device $k \in \mathcal{N}_{s^{'}}$ through other clusters $\hat{s} \in \mathcal{S} \setminus \{ s, s^{'} \}$ if necessary. 
% i.e., the set of clusters $\mathcal{S}$ is connected. 
% During intra-cluster training rounds 

Within each cluster $s \in \mathcal{S}$, the set of D2D links $\tilde{\mathcal{E}}_s$ is represented by adjacency matrix $\tilde{\boldsymbol{A}}_s \in \mathbb{R}^{N_s \times N_s}$, where $\tilde{\boldsymbol{A}}_s = [\tilde{a}_{i,j}]_{1 \leq i,j \leq N}$ with $\tilde{a}_{i,j}=0$ if $(i,j) \neq \mathcal{E}_s$ and $0 < \tilde{a}_{i,j} \leq 1$ otherwise. 
As in existing literature~\cite{jiang2017collaborative,nguyen2024decentralized}, we consider these adjacency matrices to be doubly stochastic, i.e., $\tilde{A}_s \boldsymbol{1} = \tilde{A}_s^T \boldsymbol{1} = \boldsymbol{1}$, with symmetry, i.e., $\tilde{a}_{i,j} = \tilde{a}_{j,i}$, being the result of undirected graphs.  % as in undirected graphs. 
% In practice, such adjacency matrices follow Metropolis-Hastings weights or manipulations of Sinkhorn's algorithm. % not needed, because we're going for general framework
We stack these cluster adjacency matrices $\tilde{\boldsymbol{A}}_s$ $\forall s \in \mathcal{S}$ diagonally, leading to a block diagonal adjacency matrix $\tilde{\boldsymbol{A}} \in \mathbb{R}^{N \times N}$ such that
\begin{equation}
    \tilde{\boldsymbol{A}} = 
    \begin{bmatrix}
    \tilde{\boldsymbol{A}}_1 & \cdots & \mathbf{0} \\
    \vdots & \ddots & \vdots \\
    \mathbf{0} & \cdots & \tilde{\boldsymbol{A}}_S
    \end{bmatrix}
\end{equation}
for the full graph $\tilde{G}$ during intra-cluster regimes. 
Moreover, as each block $\tilde{\boldsymbol{A}}_s$ is doubly stochastic, $\tilde{\boldsymbol{A}}$ is also doubly stochastic and symmetric. 
On the other hand, for inter-cluster regimes, the weighted set of D2D edges $\mathcal{E}$ induces a separate doubly stochastic adjacency matrix $\boldsymbol{A} = [a_{i,j}]_{1 \leq i,j \leq N}$ with $a_{i,j} \neq \tilde{a}_{i,j}$, $a_{i,j} = 0$ if $(i,j) \notin \mathcal{E}$, and $0 < a_{i,j} \leq  1$ otherwise. 
% { A must be symmetric, otherwise the follow-up modeling of $\nabla f(\theta) = (I-A)\theta$ no longer works - this is because $(I-A)\theta$ defines a non-conservative vector field, as therefore such a gradient cannot be the gradient of a scalar loss function }
% With the goal of effective cluster formation, 
With this structure, we next analyze the  conductance of each cluster via the graph conductance $\Phi(\tilde{G}_s)$ defined as
\begin{equation} \label{eq:graph_conduct_defn}
    % \Phi(\tilde{G}_s) = \min_{ \mathcal{V} \in \mathcal{N}, }
    \Phi(\tilde{G}_s) = \min_{\substack{ \mathcal{V} \subseteq \mathcal{N}_s \\ 0 < \text{vol}(\mathcal{V}) \leq \frac{1}{2} \text{vol}(\mathcal{N}_s)  }} \phi( \mathcal{V} ),
\end{equation}
where we define the volume of $\mathcal{V}$ as $\text{vol}(\mathcal{V}) = \sum_{i \in \mathcal{V}} d_i$, the degree of node $i$ as $d_i = \sum_{j \in \mathcal{N}_s} A_{i,j}$, and the cut conductance of $\mathcal{V}$ with respect to $\mathcal{N}_s$ as $\phi(\mathcal{V})$. 
Formally, this $\phi(\mathcal{V})$ is defined as
\begin{equation} \label{eq:cut_conductance_defn}
    \phi(\mathcal{V}) = \frac{\text{cut}(\mathcal{V}, \overline{\mathcal{V}})}{ \min \{ \text{vol}(\mathcal{V}), \text{vol}(\overline{\mathcal{V}})  \}   }, 
\end{equation}
where $\text{cut}(\mathcal{V}, \overline{\mathcal{V}}) = \sum_{i \in \mathcal{V}, j \in \overline{\mathcal{V}}} A_{i,j}$. 
In other words, the graph conductance $\Phi(\cdot)$ measures the smallest cut conductance, and thereby the strength of bottlenecks within a graph or cluster. 
This property is leveraged by SSD-FL for effective cluster formation, as discussed in Sec.~\ref{sec:cluster_formation}.
For brevity, we will refer to graph conductance simply as conductance throughout the rest of the manuscript.
% yields the bottleneck of the flow of the input graph, which, for SSD-FL, will be either subgraphs $\tilde{G}_s$ or the full network graph $G$.

% %%%% system model here - TBD
% \begin{figure}[t]
%     \centering
%     \includegraphics[width=0.95\linewidth]{diagrams/ssd_fl_diag2.pdf}
%     \caption{The operation of SSD-FL for the $k$-th global cycle, involving the $\tilde{k}$-th intra-cluster regime followed by the $\hat{k}$-th inter-cluster regime. 
%     SSD-FL operation across alternating intra-cluster and inter-cluster regimes. Within each cluster, devices.~\ref{XXX} }
%     \label{fig:system_model}
% \end{figure}

% timing - how many iterations within each cluster before a cluster-to-cluster communication round
Within this structure, we assume a total of $T$ operational instances, so that $\mathcal{T} = \{ 1, \cdots, T \}$, and organize $\mathcal{T}$ into a series of intra-cluster regimes of duration $\tau_a > 0$ followed by inter-cluster regimes of duration $\tau_r > 0$. 
Together, we combine intra-cluster and inter-cluster regimes into overarching global cycles of length $\tau_g = \tau_a + \tau_r$, thus leading to a total of $K = \floor{T / \tau_g}$ global cycles with $\mathcal{K} = \{ 0, \cdots, K-1 \}$.  % both of them combined into overarching cycles of length $\tau_g = \tau_a + \tau_r$. 
Given any global cycle $k \in \mathcal{K}$, we denote the intra-cluster regime as $\tilde{k}$ with $t \in \tilde{k} = \{k \tau_g, \cdots, k\tau_g + (\tau_a -1) \}$ and represent the inter-cluster regime using $\hat{k}$ such that $t \in \hat{k} = \{ k \tau_g + \tau_a, \cdots, k \tau_g + \tau_a + (\tau_r-1) \}$. 
Similarly, we denote the set of all intra-cluster and inter-cluster regimes as $\tilde{\mathcal{K}}$ and $\hat{\mathcal{K}}$ respectively. 
This form enables referencing the $q$-th step for both intra-cluster and inter-cluster regimes, e.g., $(\tilde{k},q) = k \tau_g + q$ for $q \in \{0, \cdots, \tau_a - 1 \}$, which we employ as superscripts within the ML mechanisms explained next. 
% When convenient, we will denote time-indexed quantities by the superscripts $(\tilde{k},q)$ and $(\hat{k},q)$ in intra-cluster and inter-cluster regimes respectively. 
% At a lower level, each global cycle $k \in \mathcal{K}$ is partitioned into an intra-cluster regime $t \in \tilde{k}$ where $\tilde{k} = \{k \tau_g, \cdots, k\tau_g + (\tau_a -1) \}$, and an inter-cluster regime $t \in \hat{k}$ where $\hat{k} = \{ k \tau_g + \tau_a, \cdots, k \tau_g + \tau_a + (\tau_r-1) \}$. 
% This form enables referencing the $q$-th step within intra-cluster and inter-cluster regimes, i.e., $\tilde{k}[q] \triangleq k \tau_g + q$, $q \in \{0, \cdots, \tau_a -1\}$, and $\hat{k}[q] \triangleq k \tau_g + \tau_a + q$, $q \in \{0, \cdots, \tau_r - 1\}$. 
% When convenient, we will denote time-indexed quantities by the superscripts $(\tilde{k},q)$ and $(\hat{k},q)$ in intra-cluster and inter-cluster regimes respectively. 
% For clarity in the subsequent analysis, we represent the $k$-th intra-cluster and inter-cluster regimes by $\tilde{k} = \{k\tau_g, \cdots, k\tau_g + (\tau_a-1) \}$ and $\hat{k} = \{ k \tau_g + \tau_a, \cdots, k \tau_g + \tau_a + (\tau_r-1) \}$, respectively. 

\subsection{ML model training mechanisms} \label{ssec:ml_mechs}
% visualization/overview first? then just go straight into intra-cluster mechanics
%% no separate subsections because the inter-cluster processes are quite simple/sparse - only a single paragraph 

% Within every global round $k \in \mathcal{K}$, devices $i \in \mathcal{N}_s$, $\forall s \in \mathcal{S}$, train their local ML models and cooperatively share updates, based on their local cluster's adjacency matrix $\tilde{\boldsymbol{A}}_s$, during the intra-cluster regime $\tilde{k}$.
% The subsequent inter-cluster regime $\hat{k}$ then enables network-wide propagation via the full graph adjacency matrix $\boldsymbol{A}$. 
% We first explain the specific ML model training mechanisms during intra-cluster regimes $\tilde{k} \in \tilde{\mathcal{K}}$ and then summarize the D2D consensus process in inter-cluster regimes $\hat{k} \in \hat{\mathcal{K}}$.

%%%%%%% \subsubsection{Intra-cluster processes} label{sss:intra-regimes}
%% devices' and their local ML models 
We first explain the ML model training and D2D communications for intra-cluster regimes $\tilde{k} \in \tilde{\mathcal{K}}$, then summarize the network-wide consensus process under inter-cluster regimes $\hat{k} \in \hat{\mathcal{K}}$.
During an intra-cluster regime $\tilde{k}$ and iteration $q$, all network devices $i \in \mathcal{N}$ locally train a set of ML model parameters $\boldsymbol{\theta}^{\tilde{k},q}_i \in \mathbb{R}^d$ with the goal of minimizing its local loss function $L_i(\boldsymbol{\theta}^{\tilde{k},q}_i \vert \mathcal{D}_i)$ defined as 
\begin{equation} \label{eq:local_loss_def}
    L_i(\boldsymbol{\theta}^{\tilde{k},q}_i \vert \mathcal{D}_i) = \frac{1}{D_i} \sum_{ h = 1}^{D_i} \ell_h(\boldsymbol{\theta}^{\tilde{k},q}_i \vert (x_h,y_h)),  
\end{equation}
where $\ell_h: \mathbb{R}^d \rightarrow \mathbb{R}$ is the loss function for the $h$-th datum with features $x_h \in \mathbb{R}^{w \times z}$ and label $y_h \in \mathbb{R}$, $\mathcal{D}_i$ denotes the dataset at device $i$, and $D_i$ denotes the dataset size. % of said dataset.
% (omitted for future expressions)d
% $L_i(\cdot): \mathbb{R}^d \rightarrow \mathbb{R}$ and $D_i$ denotes the size of the dataset at device $i$.
For future expressions, we will omit the $D_i$ dependence within the expression of $L_i(\cdot)$ as well as the $(x_h,y_h)$ dependence for expressions involving $\ell_h$. 
Moreover, similar to existing literature~\cite{jiang2017collaborative,chewi2023complexity}, $d$ is set to $1$ to clarify analysis through vector variables. %% or do a footnote for the commented out sentence
% For simplicity of analysis, $d$ is set to $1$ similar to existing works~\cite{jiang2017collaborative,chewi2023complexity}, thus yielding vector variables for both $\boldsymbol{\theta}^{\tilde{k},q}$ and $\hat{\boldsymbol{\theta}}^{\tilde{k},q}_s$ as well as functions thereof. 
% training specifics - % To this end, devices $i \in \mathcal{N}$ first compute loss function gradients as follows:
To minimize their local loss functions in~\eqref{eq:local_loss_def}, each device $i \in \mathcal{N}$ updates its local ML model parameters $\boldsymbol{\theta}^{\tilde{k},q}_i$ using the gradient of~\eqref{eq:local_loss_def}, expressed as  
\begin{equation} \label{eq:full_loss_grad}
    \nabla L_i(\boldsymbol{\theta}^{\tilde{k},q}_i) = \frac{1}{D_i} \sum_{h=1}^{D_i} \nabla \ell_h(\boldsymbol{\theta}^{\tilde{k},q}_i).
\end{equation}
% Since SSD-FL uses the stochastic gradient,~\eqref{eq:full_loss_grad} is then approximated as:
In practice, the full gradient in~\eqref{eq:full_loss_grad} is often approximated by a stochastic gradient,
\begin{equation} \label{eq:stochastic_local_grad}
    g_i(\boldsymbol{\theta}^{\tilde{k},q}_i) = \frac{1}{M} \sum_{h \in \mathcal{D}_i^{\mathsf{B},\tilde{k},q}} \nabla \ell_h(\boldsymbol{\theta}^{\tilde{k},q}_i),
\end{equation}
where $\mathcal{D}_i^{\mathsf{B},\tilde{k},q}$ denotes a randomly sampled mini-batch of $M$ data from $\mathcal{D}_i$ at the $q$-th instance during the $\tilde{k}$ intra-cluster regime. 
Using~\eqref{eq:stochastic_local_grad}, devices $i \in \mathcal{N}$ then leverage stochastic gradient descent (SGD) approaches with heterogeneous optimizers, resulting in standard SGD, proximal SGD, or SGD with momentum. 
Formally, these optimizers have the following structures:
% For the aforementioned optimizers, the local update process would yield effective gradients with the following structures:
\begin{itemize}
    \item Standard SGD: 
    \begin{equation} \label{eq:eff_sgd}
        \tilde{g}_i(\boldsymbol{\theta}^{\tilde{k},q}_i) = g_i(\boldsymbol{\theta}^{\tilde{k},q}_i), \forall q, i.
    \end{equation}
    
    \item Proximal SGD with $0 \leq \mu_i < 1$ being the proximal parameter~\cite{xiao2014proximal}:  
    \begin{equation} \label{eq:eff_prox_sgd}
        \tilde{g}_i(\boldsymbol{\theta}^{\tilde{k},q}_i) = g_i(\boldsymbol{\theta}^{\tilde{k},q}_i) + \mu_i (\boldsymbol{\theta}^{\tilde{k},q}_i - \boldsymbol{\theta}^{\tilde{k},0}_i), \forall q, i.
    \end{equation}
    
    \item SGD with momentum where $0 \leq \rho_i < 1$ is the momentum parameter~\cite{sutskever2013importance}:
    \begin{equation} \label{eq:eff_sgd_mtm}
        \tilde{g}_i(\boldsymbol{\theta}^{\tilde{k},q}_i) = g_i(\boldsymbol{\theta}^{\tilde{k},q}_i) + \sum_{p=0}^{q-1} \rho_i^{q-p} g_i(\boldsymbol{\theta}^{\tilde{k},p}_i), \forall q, i. 
    \end{equation}
    
\end{itemize}
Combined,~\eqref{eq:eff_sgd}-\eqref{eq:eff_sgd_mtm} yields an aggregate expression for the local stochastic gradients:
% the following overall local stochastic gradient: {\color{blue} need to add sentence that you cannot have both proximal SGD + momentum simultaneously - devices can only follow a single form, i.e., if $\rho_o > 0$ then $\mu_i = 0$ and vice versa.}
\begin{equation} \label{eq:eff_sgd_ovr} 
    \tilde{g}_i(\boldsymbol{\theta}^{\tilde{k},q}_i) = g_i(\boldsymbol{\theta}^{\tilde{k},q}_i) + \sum_{p=0}^{q-1} \rho_i^{q-p} g_i(\boldsymbol{\theta}^{\tilde{k},p}_i) + \mu_i (\boldsymbol{\theta}^{\tilde{k},q}_i - \boldsymbol{\theta}^{\tilde{k},0}_i), 
\end{equation}
which enables devices $i \in \mathcal{N}$ to choose their specific optimizer, e.g., $\mu_i = \rho_i = 0$ indicates standard SGD while $\mu_i >0 $ and $\rho_i = 0$ indicates SGD with momentum. 
% The local stochastic gradient $\tilde{g}_i$ is flexible - it enables devices $i \in \mathcal{N}$ to choose their preferred form of SGD independently. For example, device $i$ may engage in only standard SGD via $\mu_i = \rho_i = 0$ while device $j$ may leverage all three forms of SGD via $\mu_i > 0$ and $\rho_i > 0$ within~\eqref{eq:eff_sgd_ovr}. 

%ML model parameter sharing (in intra-cluster regimes),
% Including D2D communications (in intra-cluster regimes), the ML model update process for devices $i \in \mathcal{N}_s$, $\forall s \in \mathcal{S}$, takes the following form: 

Thus, given an intra-cluster regime $\tilde{k} \in \tilde{\mathcal{K}}$ and cluster $s \in \mathcal{S}$, each device $i \in \mathcal{N}_s$ would simultaneously update and share their local ML model among neighbors with active D2D links via
% all devices $i \in \mathcal{N}$ update their local ML models during intra-cluster regimes $\tilde{k} \in \tilde{\mathcal{K}}$ as follows:
\begin{equation} \label{eq:intraC_updates}
    \boldsymbol{\theta}^{\tilde{k},q+1}_i =  \sum_{\substack{j \in \mathcal{N}_s} } \tilde{a}^s_{j,i} \boldsymbol{\theta}^{\tilde{k},q}_j - \eta \tilde{g}_i(\boldsymbol{\theta}^{\tilde{k},q}_i), 
    % \left(g_i(\boldsymbol{\theta}^t_i) + \sum_{q=0}^{t-1} \rho_i^{t-q} g_i(\boldsymbol{\theta}^q_i) + \mu_i (\boldsymbol{\theta}^t_i - \boldsymbol{\theta}^0_i) \right), 
    %g_i (\boldsymbol{\theta}^t_i), 
\end{equation}
where $\tilde{a}^{s}_{j,i}$ is the $(j,i)$-th entry of the $s$-th cluster's adjacency matrix $\tilde{\boldsymbol{A}}_s$, $\sum_{j \in \mathcal{N}_s} \tilde{a}^s_{j,i} \boldsymbol{\theta}^{\tilde{k},q}_j$ represents the weighted sum of device $i$'s neighboring ML models from the $q$-th iteration and $\eta > 0$ is the learning rate. 
Combining the individual device update rule in~\eqref{eq:intraC_updates} for a cluster $s$, we then obtain 
\begin{equation} \label{eq:cluster_2gether_update}
    \hat{\boldsymbol{\theta}}^{\tilde{k},q+1}_s = \tilde{\boldsymbol{A}}_s \hat{\boldsymbol{\theta}}^{\tilde{k},q}_s - \eta \tilde{\boldsymbol{G}}_s(\hat{\boldsymbol{\theta}}^{\tilde{k},q}_s), 
\end{equation}
where 
\begin{equation} \label{eq:tildeg_cluster_wide}
\begin{aligned}
    & \tilde{\boldsymbol{G}}_s(\hat{\boldsymbol{\theta}}^{\tilde{k},q}_s) = \boldsymbol{G}_s(\hat{\boldsymbol{\theta}}^{\tilde{k},q}_s) + 
    \sum_{p=0}^{q-1} \boldsymbol{\rho}_s^{q-p} \odot \boldsymbol{G}_s(\hat{\boldsymbol{\theta}}^{\tilde{k},p}_s) 
    \\
    & + \boldsymbol{\mu}_s \odot \left( \hat{\boldsymbol{\theta}}^{\tilde{k},q}_s - \hat{\boldsymbol{\theta}}^{\tilde{k},0}_s \right), 
\end{aligned}
\end{equation}
$\boldsymbol{G}_s(\hat{\boldsymbol{\theta}}^{\tilde{k},q}_s) = \left[ g_i(\boldsymbol{\theta}^{\tilde{k},q}_i) \right]_{i \in \mathcal{N}_s}$, $\boldsymbol{\rho}^q_s = [\rho^q_i]_{i \in \mathcal{N}_s}$, $\boldsymbol{\mu}_s = [\mu_i]_{i \in \mathcal{N}_s}$, and $\odot$ denotes the Hadamard product. 

In~\eqref{eq:cluster_2gether_update}, both the active intra-cluster D2D links and local gradient induce changes to local ML model parameters, which can be better highlighted after the introduction of $\pm \hat{\boldsymbol{\theta}}^{\tilde{k},q}_s$ as follows:
% Introducing $\pm \hat{\boldsymbol{\theta}}^{\tilde{k},q}_s$ to~\eqref{eq:cluster_2gether_update} then rearranging enables: 
\begin{equation} \label{eq:expand_cluster_gd}
    \hat{\boldsymbol{\theta}}^{\tilde{k},q+1}_s = \hat{\boldsymbol{\theta}}^{\tilde{k},q}_s - \eta \underbrace{\left( \tilde{\boldsymbol{G}}_s(\hat{\boldsymbol{\theta}}^{\tilde{k},q}_s)  + \frac{1}{\eta} \left(  \mathbf{I}_s - \boldsymbol{A}_s \right) \hat{\boldsymbol{\theta}}^{\tilde{k},q}_s \right) }_{\triangleq \nabla \tilde{F}_s(\hat{\boldsymbol{\theta}}^{\tilde{k},q}_s)},
\end{equation}
where we use $\mathbf{I}_s$ to denote the identity matrix of identical dimension to $\boldsymbol{A}_s$.
%% manually verify that this expanded gradient is the unbiased estimate of the gradient of the lyapunov loss function...
%% maybe this incurs a story flow change from lyapunov first instead of lyapunov last? 
From~\eqref{eq:expand_cluster_gd}, the stochastic gradient update is consequently $\nabla \tilde{F}_s (\hat{\boldsymbol{\theta}}^{\tilde{k},q}_s)$, which indicates that SSD-FL minimizes the stochastic gradient of an ``effective" intra-cluster loss function. 
By reversing this gradient, $\nabla \tilde{F}_s(\cdot)$, we formally define the effective intra-cluster loss function as follows:
\begin{equation} \label{eq:eff_c_loss}
\begin{aligned}
    & \tilde{L}_s(\hat{\boldsymbol{\theta}}^{\tilde{k},q}_s) = \underbrace{\sum_{i \in \mathcal{N}_s} L_i(\boldsymbol{\theta}^{\tilde{k},q}_i)}_{(a)} + \underbrace{ \sum_{i \in \mathcal{N}_s} 
    \boldsymbol{\theta}^{\tilde{k},q}_i 
    \sum_{p=0}^{q-1} \rho^{q-p}_i g_i(\boldsymbol{\theta}^{\tilde{k},p}_i)}_{(b)} 
    \\ 
    & + \underbrace{ \sum_{i \in \mathcal{N}_s} \frac{\mu_i}{2} \Vert \boldsymbol{\theta}^{\tilde{k},q}_i - \boldsymbol{\theta}^{\tilde{k},0}_i \Vert^2 }_{(c)}
    + \underbrace{ \frac{1}{2 \eta } \Vert \hat{\boldsymbol{\theta}}^{\tilde{k},q}_s \Vert^2_{\mathbf{I}_s-\boldsymbol{A}_s}}_{(d)},
\end{aligned}
\end{equation}
where terms $(a)$ and $(b)$ assess the ML model qualities via functions of loss (with term $(b)$ active only for devices $i \in \mathcal{N}_s$ with $\rho_i > 0$), %primarily built from historical/cumulative sum of momentum-scaled stochastic gradients of loss functions $L_i(\cdot)$, 
term $(c)$ minimizes sudden or dynamic changes to local ML model parameters (i.e., smoother training for devices $i \in \mathcal{N}_s$ with $\mu_i >0$), and term $(d)$ integrates active D2D interaction via a graph regularization term:
% Fully expanding this graph regularization yields
\begin{equation} \label{eq:greg_full}
\begin{aligned}
    & \Vert \hat{\boldsymbol{\theta}}^{\tilde{k},q}_s \Vert^2_{\mathbf{I}_s - \tilde{\boldsymbol{A}}_s} = \left(\hat{\boldsymbol{\theta}}^{\tilde{k},q}_s\right)^T (\mathbf{I}_s - \tilde{\boldsymbol{A}}_s) \hat{\boldsymbol{\theta}}^{\tilde{k},q}_s \\
    & = \sum_{i,j \in \mathcal{N}_s} \tilde{a}_{j,i} (\boldsymbol{\theta}^{\tilde{k},q}_i - \boldsymbol{\theta}^{\tilde{k},q}_j)^2.
\end{aligned}
\end{equation}
The form of~\eqref{eq:greg_full} thus encourages consensus among devices $i \in \mathcal{N}_s$ belonging to the same cluster $s$ and with active D2D links $\tilde{a}_{j,i} > 0$.  
Moreover, since $\mathbb{E} \left[ \nabla \tilde{F}_s(\hat{\boldsymbol{\theta}}^{\tilde{k},q}_s) \right] = \nabla \hat{L}_s(\hat{\boldsymbol{\theta}}^{\tilde{k},q}_s)$, we can confirm that, in~\eqref{eq:expand_cluster_gd}, clusters $s \in \mathcal{S}$ are performing forms of gradient descent with respect to the effective loss function in~\eqref{eq:eff_c_loss}. %or should we say "performing local ML model training w.r.t.?

These properties for $\tilde{L}_s$ also hold at the global level, across the sum for all clusters. 
Since $\tilde{\boldsymbol{A}}$ consists of blocks $\tilde{\boldsymbol{A}}_s$ with all $s \in \mathcal{S}$, the effective global loss function can be expressed as the sum of~\eqref{eq:eff_c_loss} over all clusters $s \in \mathcal{S}$ as follows:
\begin{equation} \label{eq:eff_g_loss}
\begin{aligned}
    & \tilde{L}(\boldsymbol{\theta}^{\tilde{k},q}) = \sum_{s \in \mathcal{S}} 
    \sum_{i \in \mathcal{N}_s} L_i(\boldsymbol{\theta}^{\tilde{k},q}_i) 
    \\
    & + \sum_{i \in \mathcal{N}} \boldsymbol{\theta}^{\tilde{k},q}_i  \sum_{p=0}^{q-1} \rho^{q-p}_i \nabla L_i(\boldsymbol{\theta}^{\tilde{k},p}_i) 
    \\
    & + \sum_{i \in \mathcal{N}} \frac{\mu_i}{2} \Vert \boldsymbol{\theta}^{\tilde{k},q}_i - \boldsymbol{\theta}^{\tilde{k},0}_i \Vert^2 
    + \frac{1}{2 \eta } \Vert \hat{\boldsymbol{\theta}}^{\tilde{k},q}_s \Vert^2_{\mathbf{I} - \tilde{\boldsymbol{A}}}, 
\end{aligned}
\end{equation}
where $\mathbf{I}$ is the identity matrix of similar dimension to $\tilde{\boldsymbol{A}}$. 
Thus, across all clusters $s \in \mathcal{S}$, network devices $i \in \mathcal{N}$ collectively aim to minimize their local loss (and local proximal terms if $\rho_i > 0$) while improving cluster-wide consensus, during intra-cluster regimes $\tilde{k} \in \tilde{\mathcal{K}}$. % $\forall k \in \mathcal{K}$. 
% Leveraging~\eqref{eq:eff_g_loss}, we will subsequently examine SSD-FL's intra-cluster regimes in Sec.~\ref{sec:theory}. 

% {\color{blue}New effect - penalty based on the number of clusters. Side effect - topology driven regularization.}
% {\color{blue}Conclusion: We can use this effective loss to evaluate convergence of the SSD-FL.}

%% you don't use $\nabla \tilde{F}$ at all and $\nabla \tilde{F}_s$ already defined above
% {\color{blue}Need to terms related to the stochastic regularized loss function $\tilde{F}_s$ and $\tilde{F}$ because the proofs will involve them rather than the full regularized loss functions $\tilde{L}_s$ and $\tilde{L}$.}

%% inter-cluster regimes
%%%%%% \subsubsection{Inter-cluster Regimes} \label{sss:inter_regimes} 
On the other hand, during inter-cluster regimes $\hat{k} \in \hat{\mathcal{K}}$, all network devices $i \in \mathcal{N}$ undergo simulated global (i.e., across all clusters) synchronizations by iterative D2D communications. 
% perform global (i.e., across all clusters) synchronizations. 
Since the goal during inter-cluster regimes extends beyond singular clusters, the update rule follows the full graph $G$ and thus adjacency matrix $\boldsymbol{A}$, yielding
\begin{equation} \label{eq:inter_c_updates}
    \boldsymbol{\theta}^{\hat{k},q+1} = \boldsymbol{A}^T \boldsymbol{\theta}^{\hat{k},q},
\end{equation}
for a total of $\tau_r$ iterations. 
These inter-cluster updates represent a diffusion process over the global network graph $G$, whose efficiency depends on the connectivity within $\boldsymbol{A}$. In highly connected edge/fog networks, such as dense or fully connected networks (i.e., those with complete graphs), D2D ML model parameter propagation happens fast, and the network acts as a single cluster. 
By contrast, for sparse or weakly connected $\boldsymbol{A}$ (i.e., edge/fog networks with highly heterogeneous link density), some devices may be poorly synchronized. 
As such, careful clustering, followed by intra-cluster consensus before inter-cluster communications, provides a structural remedy.
We examine these scenarios within our experiments in Sec.~\ref{sec:experiments}.

% To this end, during $\hat{k}$, devices $i \in \mathcal{N}$ will engage with D2D communications with their neighbors, i.e., all possible D2D links/edges,  following the adjacency matrix $\boldsymbol{A}$, yielding the following update rule: 
% \begin{equation} \label{eq:inter_c_updates}
%     \boldsymbol{\theta}^{\hat{k},q+1} = \boldsymbol{A}^T \boldsymbol{\theta}^{\hat{k},q}.
% \end{equation}
% The global D2D communication process in~\eqref{eq:inter_c_updates} aims to achieve consensus, with the needed communication rounds derived by the network's underlying structure. 

\subsection{Theoretical Background} \label{ssec:theory_back}
%%% structure
% %% definitions + assumptions
% smoothness, PL-condition, 
% % intermediary sentence - assumptions on the adjacency matrix for all time
% doubly stochastic, positive semi-definite
% % direct consequences of the adjacency matrix assumptions: (i) ordered eigenvalues 

We next define theoretical properties underpinning SSD-FL's convergence and consensus properties, which we present in Sec.~\ref{sec:theory}. 
To this end, we first explain assumptions on the device-level loss functions from~\eqref{eq:local_loss_def}, beginning with smoothness and bounded gradients.
%rely on the following assumptions on the loss functions. 
\begin{assumption}[Smoothness] \label{ass:smoothness}
The loss functions $L_i(\cdot)$ are $\gamma_i$-Lipschitz smooth, where $\gamma_i > 0$ and $\forall i \in \mathcal{N}$.
Formally, 
\begin{equation} \label{eq:ass_smoothness}
    \Vert \nabla L_i(\boldsymbol{\theta}_1) - \nabla L_i(\boldsymbol{\theta}_2) \Vert \leq \gamma_i \Vert \boldsymbol{\theta}_1 - \boldsymbol{\theta}_2 \Vert, 
\end{equation}
where $\boldsymbol{\theta}_1, \boldsymbol{\theta}_2 \in \mathbb{R}^{d}$. 
\end{assumption}

% Next, to allow for the analysis on non-convex loss functions, we leverage the Polyak-Lojasiewicz (PL) condition~\cite{karimi2016linear}.
% \begin{assumption}[PL Condition] \label{ass:PL_cond}
% % Given some optimal loss function value $L_i^{*}$ and PL parameter $\beta > 0$, the loss function $L_i(\cdot)$ satisfies the PL condition if:
% For all $i \in \mathcal{N}$ and $\boldsymbol{\theta} \in \mathbb{R}^d$, the loss functions $L_i(\boldsymbol{\theta})$ satisfy the PL condition if:
% \begin{equation} \label{eq:ass_PL_cond}
%     \frac{1}{2} \Vert \nabla L_i(\boldsymbol{\theta}) \Vert^2 \geq \beta_i (L_i(\boldsymbol{\theta}) - L_i^*), 
% \end{equation}
% where $L_i^*$ is the optimal value of $L_i(\boldsymbol{\theta})$ and $\beta >0$ is the PL parameter. 
% \end{assumption}

% %%% the PL-condition is a commonly used method to prove convergence for non-convex loss functions in more challenging settings~\cite{XXX,XXX}. 

\begin{assumption}[Bounded Gradients] \label{ass:bounded_grad}
The gradients of loss functions $\nabla L_i(\cdot)$ are bounded $\forall i \in \mathcal{N}$ and $\forall \boldsymbol{\theta}_i$ as follows: 
\begin{equation} \label{eq:ass_bound_grad}
    \Vert \nabla L_i(\boldsymbol{\theta}_i) \Vert \leq B, 
\end{equation}
where $0 < B < \infty$.  
\end{assumption}
% in our theoretical results (Section~\ref{sec:theory})
We will leverage Assumptions~\ref{ass:smoothness} and~\ref{ass:bounded_grad} together to simplify the effective global loss of~\eqref{eq:eff_g_loss} and subsequently prove convergence of SSD-FL. % with non-convex loss functions.
As such, we also need to formalize properties for the adjacency matrices $\tilde{\boldsymbol{A}}$ and $\boldsymbol{A}$ as follows:

% Next, we present the assumptions on the adjacency matrix. 
\begin{assumption}[Adjacency Matrix Properties] \label{ass:adjacency_matrix}
The adjacency matrices $\tilde{\boldsymbol{A}}$ and $\boldsymbol{A}$ are both assumed to have the following properties: (i) doubly stochastic such that $\tilde{\boldsymbol{A}} \boldsymbol{1} = {\boldsymbol{A}} \boldsymbol{1} = \boldsymbol{1}$ and $\tilde{\boldsymbol{A}}^T \boldsymbol{1} = {\boldsymbol{A}}^T \boldsymbol{1} = \boldsymbol{1}$, (ii) $\mathbf{I} \succeq \tilde{\boldsymbol{A}} \succ 0$ and $\mathbf{I} \succeq {\boldsymbol{A}} \succ 0$, where $\succeq$ and $\succ$ denote positive semi-definite and positive definite respectively, and (iii) symmetric such that $\tilde{\boldsymbol{A}} = \tilde{\boldsymbol{A}}^T$ and ${\boldsymbol{A}} = {\boldsymbol{A}}^T$. 
\end{assumption} 

As a consequence of Assumption~\ref{ass:adjacency_matrix}, we have that, via the doubly stochastic condition and the Perron-Frobenius Theorem~\cite{pillai2005perron}, the largest eigenvalue of both $\tilde{\boldsymbol{A}}$ and ${\boldsymbol{A}}$ are $1$, and that, via the positive definite property, the eigenvalues of $\tilde{\boldsymbol{A}}$ and $\boldsymbol{A}$ are real and strictly positive, i.e., $\lambda_1(\boldsymbol{A}) = 1 \geq \lambda_2(\boldsymbol{A}) \geq \cdots \geq \lambda_{i}(\boldsymbol{A}) > 0$ where $\lambda_{m}(\boldsymbol{A})$ denotes the $m$-th largest eigenvalue of $\boldsymbol{A}$.
The final assumption relates to variability in the effective intra-cluster loss functions defined in~\eqref{eq:eff_c_loss}.

% Finally, the non-singularity assumption enables asymmetric adjacency matrices, $\tilde{\boldsymbol{A}}$ and $\boldsymbol{A}$, and thus a more general class of D2D processes. 

\begin{assumption}[Bounded Gradient Variances] \label{ass:grad_var}
For any intra-cluster regime $\tilde{k}$, $k \in \mathcal{K}$, cluster $s \in \mathcal{S}$, and instance $0 \leq q < \tau_a-1$, 
% there are scalars $0 < \beta_1 \leq \beta_2$ such that 
% \begin{equation}
%     \beta_1 \Vert \nabla \tilde{L}_s(\hat{\boldsymbol{\theta}}^{\tilde{k},q}_s) \Vert^2 \leq  \mathbb{E} \left[ \nabla \tilde{F}_s(\hat{\boldsymbol{\theta}}^{\tilde{k},q}_s) \right]^T \nabla \tilde{L}_s(\hat{\boldsymbol{\theta}}^{\tilde{k},q}_s),  
% \end{equation}
% and 
% \begin{equation}
%     \left\Vert \mathbb{E}\left[ \nabla \tilde{F}_s(\hat{\boldsymbol{\theta}}^{\tilde{k},q}_s) \right] \right\Vert 
%     \leq \beta_2 \left\Vert \nabla \tilde{L}_s(\hat{\boldsymbol{\theta}}^{\tilde{k},q}_s) \right\Vert.
% \end{equation}
% Moreover, 
there exist scalars $\alpha, \alpha_s \geq 0$ such that 
\begin{equation} \label{eq:ass_bound_grad_var}
    \operatorname{Var}\left[ \nabla \tilde{F}_s(\hat{\boldsymbol{\theta}}^{\tilde{k},q}_s) \right] 
    \leq \alpha + \alpha_s \Vert \nabla \tilde{L}_s(\hat{\boldsymbol{\theta}}^{\tilde{k},q}_s) \Vert^2.
\end{equation}
\end{assumption}

Since $\nabla \tilde{F}_s(\hat{\boldsymbol{\theta}}^{\tilde{k},q}_s)$ is the unbiased estimate of the effective intra-cluster gradient $\nabla \tilde{L}_s(\hat{\boldsymbol{\theta}}^{\tilde{k},q}_s)$, Assumption~\ref{ass:grad_var} follows naturally. 
Finally, as a result of Assumption~\ref{ass:grad_var}, we have that
\begin{equation} \label{eq:bound_2nd_moment}
    \mathbb{E}\left[ \left\Vert 
    \nabla \tilde{F}_s(\hat{\boldsymbol{\theta}}^{\tilde{k},q}_s) 
    \right\Vert^2  \right] \leq  \alpha + \hat{\alpha}_s \left\Vert 
    \nabla \tilde{L}_s(\hat{\boldsymbol{\theta}}^{\tilde{k},q}_s) 
    \right\Vert^2,
\end{equation}
where $\hat{\alpha}_s = \alpha_s + 1$. 
In both~\eqref{eq:ass_bound_grad_var} and~\eqref{eq:bound_2nd_moment}, the constant $\alpha$ depicts baseline variance, i.e., variance floor when $\Vert \nabla \tilde{L}_s(\hat{\boldsymbol{\theta}}^{\tilde{k},q}_s) \Vert^2 \approx 0$, and thereby describes the gradient noise of the stochastic effective gradient $\nabla \tilde{F}_s(\hat{\boldsymbol{\theta}}^{\tilde{k},q}_s)$ independent of other variables such as specific intra-cluster regime $\tilde{k} \in \tilde{\mathcal{K}}$. % e.g., current training iteration
Meanwhile, $\alpha_s$ and, by extension, $\hat{\alpha}_s$ estimate relative gradient norm amplification, specifically how the variance of effective intra-cluster stochastic gradient grows with full intra-cluster gradient norm $\Vert \nabla \tilde{L}_s(\hat{\boldsymbol{\theta}}^{\tilde{k},q}_s) \Vert^2$. 
In practice, both $\alpha_s$ and thus $\hat{\alpha}_s$ are influenced by dataset and optimizer heterogeneity within each cluster $s \in \mathcal{S}$, and we develop a methodology for their estimation in Sec.~\ref{sec:cluster_formation}. 

\section{Theoretical Results}
\label{sec:theory}

%% overview paragraph 
In the following, we prove integrated (joint intra- and inter-cluster) convergence across global rounds $k \in \mathcal{K}$ for SSD-FL. 
This analysis presents several non-trivial challenges relative to existing decentralized FL convergence results. 
The heterogeneous optimizer structure in~\eqref{eq:eff_sgd_ovr} requires construction of an effective loss function in~\eqref{eq:eff_c_loss}, whose smoothness properties require new treatment of gradient gaps across momentum, proximal, and graph regularization terms simultaneously in Sec.~\ref{ssec:effect_loss_prove_props}.
Subsequently, in Sec.~\ref{ssec:main_conv_results}, we explain the integrated convergence of SSD-FL across intra- and inter-cluster regimes in Theorem~\ref{thm:combined_conv}, which leverages the results in Sec.~\ref{ssec:effect_loss_prove_props} and cannot be obtained by direct application or extension of single regime (intra- or inter-cluster) analysis.

% requires careful composition of both regimes' convergence and consensus properties beyond what single-regime analyses can provide.
% how it cannot be obtained by direct application of single-regime analyses. 

% To this end, we first establish core properties of effective intra-cluster loss functions in Sec.~\ref{ssec:effect_loss_prove_props}, before proving the main convergence and consensus results in Sec.~\ref{ssec:main_conv_results}. 

\subsection{Effective loss function properties} 
% Properties of the regularized global loss function} 
\label{ssec:effect_loss_prove_props}

Given any intra-cluster regime $\tilde{k} \in \tilde{\mathcal{K}}$, we bound the gradient gap for effective intra-cluster loss functions, considering the option for heterogeneous optimizers therein. 
% To this end, we first establish that cluster-wide regularized loss functions $\tilde{L}_s$, $\forall s \in \mathcal{S}$ are smooth with a constant gap even under the presence of heterogeneous and hybrid optimizers. 
%% its effectively regularized loss function smoothness
% but because of the constant factor, can call it something else
\begin{restatable}{proposition}{reglsmooth} {\normalfont (Gradient Gap of Effective Intra-cluster Loss)} 
\label{thm:smoothness_reg_loss} 
    Given two instances $q_1$ and $q_2$ such that $q_1 \neq q_2$ and $q_1, q_2 < \tau_a$ within any intra-cluster regime $\tilde{k} \in \tilde{\mathcal{K}}$, the cluster-level regularized loss functions $\tilde{L}_s(\hat{\boldsymbol{\theta}}^{\tilde{k},q_1}_s)$ and $\tilde{L}_s(\hat{\boldsymbol{\theta}}^{\tilde{k},q_2}_s)$ , $\forall s \in \mathcal{S}$ have bounded gradient gap as follows:
    \begin{equation} \label{eq:thm_regl_smooth_final}
        \left\Vert \nabla \tilde{L}_s(\hat{\boldsymbol{\theta}}^{\tilde{k},q_1}_s) - \nabla \tilde{L}_s(\hat{\boldsymbol{\theta}}^{\tilde{k},q_2}_s) \right\Vert \leq  \gamma_s^{\mathsf{eff}}  \Vert \hat{\boldsymbol{\theta}}^{\tilde{k},q_1}_s - \hat{\boldsymbol{\theta}}^{\tilde{k},q_2}_s \Vert
        + \tau_a B \sqrt{N_s}
    \end{equation}
    where 
    \begin{equation}
        \gamma_s^{\mathsf{eff}} = \left( \hat{\gamma}_s + 1 + \frac{1}{\eta} (1 - \lambda_{N_s}(\boldsymbol{A}_s)) \right),
    \end{equation}
    and $\hat{\gamma}_s = \max_{i \in \mathcal{N}_s} \gamma_i$. 
    Similarly, for global-level regularized loss functions $\tilde{L}(\boldsymbol{\theta}^{\tilde{k},q_1})$ and $\tilde{L}(\boldsymbol{\theta}^{\tilde{k},q_2})$, the gradient gap is
    \begin{equation} \label{eq:thm_regl_smooth_global_final}
        \left\Vert \nabla \tilde{L}(\boldsymbol{\theta}^{\tilde{k},q_1}) - \nabla \tilde{L}(\boldsymbol{\theta}^{\tilde{k},q_2}) \right\Vert 
        \leq 
        {\gamma}^{\mathsf{eff}} 
        \left\Vert {\boldsymbol{\theta}}^{\tilde{k},q_1} - {\boldsymbol{\theta}}^{\tilde{k},q_2} \right\Vert
        + 
        \tau_a B \sqrt{N},
    \end{equation}
    where 
    \begin{equation}
        {\gamma}^{\mathsf{eff}} = \left(\hat{\gamma} + 1 + \frac{1}{\eta} \left( 1- \lambda_{N}(\tilde{\boldsymbol{A}}) \right) \right),
    \end{equation}
    and $\hat{\gamma} = \max_{i \in \mathcal{N}} \gamma_i$. 
\end{restatable}
\begin{proof}
    See Appendix~\ref{app_ssec:smoothness_reg_loss}.
\end{proof}

The gradient gap in Proposition~\ref{thm:smoothness_reg_loss} extends the smoothness assumption with standard loss functions in~\eqref{eq:local_loss_def} to effective intra-cluster loss functions from~\eqref{eq:eff_c_loss}. 
With it, we subsequently establish a corresponding loss gap between any two iterations $q_1, q_2 \in \tilde{k}$, for all $\tilde{k} \in \tilde{\mathcal{K}}$ as follows:
% {\color{blue}Adjust this statement to include effective global losses as well}
\begin{restatable}{corollary}{resmooth}{\normalfont(Effective Intra-cluster Loss Gap)}
\label{coro:re_smooth}
Given two instances $q_1$ and $q_2$ such that $q_1 \neq q_2$ and $q_1, q_2 < \tau_a$ within any intra-cluster regime $\tilde{k} \in \tilde{\mathcal{K}}$, the cluster-level effective loss functions $\tilde{L}_s(\hat{\boldsymbol{\theta}}^{\tilde{k},q_1}_s)$ and $\tilde{L}_s(\hat{\boldsymbol{\theta}}^{\tilde{k},q_2}_s)$ , $\forall s \in \mathcal{S}$ have bounded gap as follows:
\begin{equation} \label{eq:coro_re_smooth_statement}
\begin{aligned}
    & \tilde{L}_s(\hat{\boldsymbol{\theta}}^{\tilde{k},q_1}_s) \leq \tilde{L}_s(\hat{\boldsymbol{\theta}}^{\tilde{k},q_2}_s) + 
    \left(\nabla \tilde{L}_s(\hat{\boldsymbol{\theta}}^{\tilde{k},q_2}_s)\right)^T \left( \hat{\boldsymbol{\theta}}^{\tilde{k},q_1}_s - \hat{\boldsymbol{\theta}}^{\tilde{k},q_2}_s \right) 
    \\
    & + \left( \frac{1}{2} \gamma^{\mathsf{eff}}_s + \tau_a B \sqrt{N_s} \right) \Vert \hat{\boldsymbol{\theta}}^{\tilde{k},q_1}_s - \hat{\boldsymbol{\theta}}^{\tilde{k},q_2}_s \Vert^2.
\end{aligned}
\end{equation}
% Similarly, for global-level regularized loss functions $\tilde{L}(\boldsymbol{\theta}^{\tilde{k},q_1})$ and $\tilde{L}(\boldsymbol{\theta}^{\tilde{k},q_2})$, the gradient gap is
% \begin{equation}
% \begin{aligned}
% \end{aligned}
% \end{equation}
%%% add global back if not enough content throughout the paper
\end{restatable}
\begin{proof}
    See Appendix~\ref{app_ssec:re_smooth}.
\end{proof}

Together, Proposition~\ref{thm:smoothness_reg_loss} and Corollary~\ref{coro:re_smooth} generalize conventional smoothness property of local loss functions to their effective intra-cluster loss function counterparts. 
In this regard, from~\eqref{eq:thm_regl_smooth_final},~\eqref{eq:thm_regl_smooth_global_final}, and~\eqref{eq:coro_re_smooth_statement}, we see that the smoothness of $\tilde{L}_s(\cdot)$ is preserved, with additive terms determined by the size of the cluster/network ($\sqrt{N_s}$ or $\sqrt{N}$) and the intra-cluster regime duration, $\tau_a$. 
% We next prove SSD-FL's main convergence and consensus results. 

%%%%%%%%% 
%%%%%%%%%
%%%%%%%%%
\subsection{SSD-FL convergence and consensus} 
\label{ssec:main_conv_results}
Towards proving SSD-FL's integrated global round convergence (i.e., across both intra- and inter-cluster regimes), we begin by leveraging effective intra-cluster loss function properties to demonstrate intra-cluster convergence as follows: 
\begin{restatable}{theorem}{intraConv} {\normalfont (Intra-cluster Convergence)} \label{thm:intra_conv}
    If $\eta < \frac{2}{\hat{\alpha}_s \Gamma_s }$, then, given any intra-cluster regime $\tilde{k} \in \tilde{\mathcal{K}}$ and cluster $s \in \mathcal{S}$, we bound the first-order stationary point as follows:
    % for that cluster $s$ as follows:
    \begin{equation} \label{eq:intra_conv_statement}
    \begin{aligned}
        \sum_{q=0}^{\tau_a-1} \left\Vert \nabla \tilde{L}_s(\hat{\boldsymbol{\theta}}^{\tilde{k},q}_s) \right\Vert^2 
        \leq
        \frac{\tilde{L}_s(\hat{\boldsymbol{\theta}}^{\tilde{k},0}_s) + 
        \frac{\alpha \tau_a \eta^2}{2} \Gamma_s}
        % {\eta -\hat{\alpha}_s\frac{\eta^2}{2}\left(\gamma^{\mathsf{eff}}_s + \tau_a B \sqrt{N_s}\right)}
        {\eta - \frac{\hat{\alpha}_s \eta^2}{2} \Gamma_s}
    \end{aligned}
    \end{equation} 
    where
    \begin{equation}
        \Gamma_s = \left(\gamma^{\mathsf{eff}}_s + \tau_a B \sqrt{N_s}\right).
    \end{equation}
    % Subsequently, the global (i.e., across all clusters $s \in \mathcal{S}$) first-order stationary point can be expressed as follows:
    % \begin{equation}
    %     \sum_{s \in \mathcal{S}} \sum_{q=0}^{\tau_a-1} \left\Vert 
    %     \nabla \tilde{L}_s(\hat{\boldsymbol{\theta}}^{\tilde{k},q}_s)
    %     \right\Vert^2 
    %     \leq 
    %     \sum_{s \in \mathcal{S}} \left[ \frac{
    %     \tilde{L}_s(\hat{\boldsymbol{\theta}}^{\tilde{k},0}_s) + 
    %     \frac{\alpha \tau_a \eta^2}{2} \Gamma_s }
    %     {\left( \eta  - \frac{\hat{\alpha}_s \eta^2 }{2} \Gamma_s \right) }
    %     \right]. 
    % \end{equation}
\end{restatable}
\begin{proof}
    See Appendix~\ref{app_ssec:intra_conv}.
\end{proof}
When the baseline variance of the effective intra-cluster stochastic gradient is near zero, i.e., $\alpha \approx 0$, such as when batches are the size of the full dataset per the discussion in Assumption~\ref{ass:grad_var}, then Theorem~\ref{thm:intra_conv} implies that $\sum_{q=0}^{\tau_a-1} \left\Vert \nabla \tilde{L}_s(\hat{\boldsymbol{\theta}}^{\tilde{k},q}_s) \right\Vert^2 \leq \frac{\tilde{L}_s(\hat{\boldsymbol{\theta}}^{\tilde{k},0}_s)}{\eta - \frac{\hat{\alpha}_s \eta^2}{2} \Gamma_s}$. 
% \begin{equation} \label{eq:intraconv_imply1}
%     \sum_{q=0}^{\tau_a-1} \left\Vert \nabla \tilde{L}_s(\hat{\boldsymbol{\theta}}^{\tilde{k},q}_s) \right\Vert^2 \leq 
%     \frac{\tilde{L}_s(\hat{\boldsymbol{\theta}}^{\tilde{k},0}_s)}
%     {\eta - \frac{\hat{\alpha}_s \eta^2}{2} \Gamma_s}.
% \end{equation}
In other words, the first-order stationary point becomes bounded by a constant independent of the intra-cluster regime duration $\tau_a$. 
Consequently, for large $\tau_a \rightarrow \infty$, the average first-order stationary point is bounded above by zero
\begin{equation} \label{eq:intraconv_imply1}
    \lim_{\tau_a \rightarrow \infty} \frac{1}{\tau_a} \sum_{q=0}^{\tau_a-1} \left\Vert \nabla \tilde{L}_s(\hat{\boldsymbol{\theta}}^{\tilde{k},q}_s) \right\Vert^2 \leq \lim_{\tau_a \rightarrow \infty} \frac{1}{\tau_a} \frac{\tilde{L}_s(\hat{\boldsymbol{\theta}}^{\tilde{k},0}_s)}{\eta - \frac{\hat{\alpha}_s \eta^2}{2} \Gamma_s} \rightarrow 0,
\end{equation}
and therefore $\nabla \tilde{L}_s(\hat{\boldsymbol{\theta}}^{\tilde{k},q}_s) \rightarrow 0$ for all $q \in \tilde{k}$. 
By contrast, for stochastic gradients with non-trivial batches, $\alpha > 0$ and thus the average first-order stationary point as $\tau_a \rightarrow \infty$ is bounded by a constant, specifically $\frac{ \frac{\alpha \eta^2}{2} \Gamma_s}{\eta - \frac{\hat{\alpha}_s \eta^2}{2} \Gamma_s}$. 
% As such, SSD-FL yields bounded average first-order stationary points in both cases for intra-cluster regimes $\tilde{k} \in \tilde{\mathcal{K}}$. 
As such, SSD-FL is able to yield bounded average first-order stationary points. 
With this characterization of intra-cluster regime convergence, we next examine the corresponding intra-cluster consensus gap properties. 
\begin{restatable}{lemma}{intraCons}{\normalfont(Intra-cluster consensus gap)} 
\label{thm:intra_cons}
For any intra-cluster regime $\tilde{k} \in \tilde{\mathcal{K}}$ and assuming that $\Delta^{\tilde{k},q}_s \perp \boldsymbol{1}_s$ and $\eta < 1 - \lambda_2(\boldsymbol{A}_s)$, the intra-cluster cluster consensus gap can be bounded above as follows:
% Given any instance $q$ within an intra-cluster regime $\tilde{k}$, $k \in \mathcal{K}$, at a cluster $s \in \mathcal{S}$, we bound the intra-cluster consensus gap as:
\begin{equation} \label{eq:intra_cons_feq}
    % \left\Vert \Delta^{\tilde{k},\tau_a-1}_s \right\Vert \leq \frac{2\eta B\tau_a \sqrt{N_s}}{1 - \eta - \lambda_2(\boldsymbol{A}_s)},
    \left \Vert \Delta^{\tilde{k},\tau_a}_s \right\Vert \leq \underbrace{(\lambda_2(\boldsymbol{A}_s)+\eta)^{\tau_a-1} \left \Vert \Delta^{\tilde{k},0}_s \right\Vert}_{(a)} + \underbrace{\frac{2 \eta \tau_a B \sqrt{N_s}}{1 - \eta - \lambda_2(\boldsymbol{A}_s)}}_{(b)}, 
\end{equation} 
where $\Delta^{\tilde{k},\tau_a}_s = \overline{\boldsymbol{\theta}}^{\tilde{k},\tau_a}_s \boldsymbol{1}_s - \hat{\boldsymbol{\theta}}^{\tilde{k},\tau_a}_s$, and $\overline{\boldsymbol{\theta}}^{\tilde{k},\tau_a}_s = \frac{1}{N_s} \sum_{i \in \mathcal{N}_s} \boldsymbol{\theta}^{\tilde{k},\tau_a}_s$. 
\end{restatable}
\begin{proof} 
    See Appendix~\ref{app_ssec:intra_cons}.
\end{proof} 
% {\color{blue}Maybe better idea to have $\Vert \Delta^{k,\tau_a-1}_s\Vert$ rather than $\Vert \Delta^{k,q}_s \Vert$. This way, we can clean things up quite nicely. }
% since this gap increases with larger $\tau_a$, this offers a point of motivation for inter-cluster regimes...
% while you have cluster-wide convergence, the clusters themselves do not necessarily reach consensus 

% talk about the two core terms within the intra-cluster consensus bound, 
% in particular, because of the effective intra-cluster loss functions and heterogeneous optimizer choices therein, consensus is not necessarily assured ...
% then frame them as motivation for inter-cluster regimes 
% and segue into consensus for inter-cluster regimes...

For individual clusters $s \in \mathcal{S}$, the two terms in Lemma~\ref{thm:intra_cons} highlight competing effects in intra-cluster regimes. While the initial intra-cluster disagreement decreases exponentially in term~\eqref{eq:intra_cons_feq}$(a)$ as a result of $\lambda_2(\boldsymbol{A}_s) + \eta < 1$, the $2\eta \tau_a B \sqrt{N_s}$ component of~\eqref{eq:intra_cons_feq}$(b)$ grows linearly with respect to the duration of the intra-cluster regime $\tau_a$. 
% from the effective intra-cluster loss function grows linearly with the duration of the intra-cluster regime $\tau_a$. 
Specifically, we refer to $2\eta \tau_a B \sqrt{N_s}$ as a cumulative gradient noise from bounding the gradient of the effective intra-cluster loss function $\tilde{L}_s(\hat{\boldsymbol{\theta}}^{\tilde{k},q}_s)$ and therein the heterogeneous optimizer choices embedded via $\boldsymbol{\rho}$ and $\boldsymbol{\mu}$.
As a result, Lemma~\ref{thm:intra_cons} implies that consensus is not assured within individual clusters $s \in \mathcal{S}$ even though they demonstrate convergence in Theorem~\ref{thm:intra_conv}. 
% as a result of Theorem~\ref{thm:intra_conv}.
% Therefore, since there is a need to synchronize all clusters $s \in \mathcal{S}$, inter-cluster regimes $\hat{k} \in \hat{\mathcal{K}}$ fulfill a secondary purpose in aiding intra-cluster consensus as well. 
This motivates inter-cluster regimes $\hat{k} \in \hat{\mathcal{K}}$ in SSD-FL, as the network can thus synchronize all devices $i \in \mathcal{N}$ as well as re-balance the cumulative gradient noise within clusters $s \in \mathcal{S}$.
In this regard, we next show the inter-cluster consensus:
\begin{restatable}{lemma}{interCons}{\normalfont(Inter-cluster consensus)}
\label{thm:inter_cons}
Given any instance $q$ within an inter-cluster regime $\hat{k} \in \hat{\mathcal{K}}$ and assuming that $\hat{\Delta}^{\hat{k},q} \perp \boldsymbol{1}_s$, we bound the inter-cluster consensus gap as
\begin{equation} \label{eq:inter_cons_statement}
    \left\Vert \hat{\Delta}^{\hat{k},\tau_r} \right\Vert \leq \lambda_2(\boldsymbol{A})^{\tau_r-1} \left\Vert \hat{\Delta}^{\hat{k},0} \right\Vert,
\end{equation}
where $\hat{\Delta}^{\hat{k},q} = \overline{\boldsymbol{\theta}}^{\hat{k},q} \boldsymbol{1} - {\boldsymbol{\theta}}^{\tilde{k},q}$, and $\overline{\boldsymbol{\theta}}^{\hat{k},q} = \frac{1}{N} \sum_{i \in \mathcal{N}} \boldsymbol{\theta}^{\hat{k},q}$. 
% \begin{equation} \label{eq:defn_mean_theta_diffs}
%     \hat{\Delta}^{\hat{k},q} = \overline{\boldsymbol{\theta}}^{\hat{k},q} \boldsymbol{1} - {\boldsymbol{\theta}}^{\tilde{k},q},
% \end{equation}
\end{restatable}
\begin{proof}
    See Appendix~\ref{app_ssec:inter_cons}.
\end{proof}
From~\eqref{eq:inter_cons_statement}, it is immediate that all devices $i \in \mathcal{N}$ reach consensus in exponential fashion, as Assumption~\ref{ass:adjacency_matrix} implies $\lambda_2(\boldsymbol{A})<1$. 
Thus, inter-cluster regimes $\hat{k} \in \hat{\mathcal{K}}$ integrate the intra-cluster regime convergences from Theorem~\ref{thm:intra_conv} across all clusters $s \in \mathcal{S}$.
Formally, we prove integrated convergence across full global rounds $k \in \mathcal{K}$, obtaining the following result. 
% first noting that $\eta \leq \min \{1-\lambda_2(\tilde{\boldsymbol{A}}_s), 2/\hat{\alpha}_s \Gamma_s \}$ as a result of , and 
%% this is a self-note
% From these last two theorems, have a condition
% \begin{equation}
%     \eta \leq \min\{1-\lambda_2(\boldsymbol{A}_s), \frac{2}{\hat{\alpha}_s \Gamma_s }\}
% \end{equation}
% Generally, $1 - \lambda_2(\boldsymbol{A}_s)$ tends to be static, while $\frac{2}{\hat{\alpha}_s \Gamma_s}$ can be re-arranged and expressed as
% \begin{equation}
%     \eta \leq \frac{2 - (\alpha_s + 1) (1 - \lambda_N(\tilde{A}) ) }{ (\alpha_s + 1) (\hat{\gamma}_s + 1 + \tau_a B \sqrt{N_s})}
% \end{equation}
% {\color{blue}Need to make sure that $L$ is $\tilde{L}$... here and in the proofs} -- done
\begin{restatable}{theorem}{combConv} {\normalfont (Integrated Convergence)} \label{thm:combined_conv}
    Let $\eta \leq \min_{s \in \mathcal{S}} \{1-\lambda_2(\tilde{\boldsymbol{A}}_s), \frac{2}{\hat{\alpha}_s \Gamma_s} \}$, then, for all global cycles $k \in \mathcal{K}$, we have bounded first-order stationary point as follows:
    \begin{equation} \label{eq:comb_conv_ovr_statement}
    \begin{aligned}
        &\sum_{q=0}^{\tau_r + \tau_a-1} \left\Vert \nabla \tilde{L}(\boldsymbol{\theta}^{k,q}) \right\Vert^2 
        \leq
        \frac{(2\tau_r - 1 ) \tilde{L}(\boldsymbol{\theta}^{k,0}) + \alpha C_1 \sum_{s \in \mathcal{S}} \Gamma_s}{ \eta - \frac{\hat{\alpha} \eta^2}{2} \Gamma } \\ 
        & 
        + 4 (\tau_r - 1) \left( \frac{ \gamma^{\mathsf{eff}} \left( 1 - \lambda_N({\boldsymbol{A}}) \right) }{1-\lambda_2({\boldsymbol{A}})} \right)^2 \Vert \hat{\Delta}^{k,0} \Vert^2  + 4 C_2
    \end{aligned}
    \end{equation} 
    where $\Gamma = \gamma^{\mathsf{eff}} + \tau_a B \sqrt{N}$, $C_1 = \frac{ \left(\tau_a + 2 \tau_r - 2 \right) \eta^2}{2}$, $C_2 = \tau_a^2 B^2 N \tau_r (\tau_a + \tau_r - 1)^2$, and $\hat{\alpha} = \max_{s \in \mathcal{S}} \hat{\alpha}_s$.
    % \begin{align}
    %     & \Gamma = \gamma^{\mathsf{eff}} + \tau_a B \sqrt{N}, \\
    %     & C_1 = \frac{ \left(\tau_a + 2 \tau_r - 2 \right) \eta^2}{2}, \\
    %     & C_2 = \frac{ \gamma^{\mathsf{eff}} }{1-\lambda_2(\tilde{\boldsymbol{A}})}, \\
    %     & C_3 = \tau_a^2 B^2 N \tau_r (\tau_a + \tau_r - 1)^2
    % \end{align}
\end{restatable}
\begin{proof}
    See Appendix~\ref{app_ssec:comb_conv}.
\end{proof}

Aside from $\tilde{L}(\boldsymbol{\theta}^{k,0})$, all other terms on the right hand side of~\ref{eq:comb_conv_ovr_statement} remain constants independent of the global round $k \in \mathcal{K}$. 
As such, since effective global loss from~\eqref{eq:eff_g_loss}
can be bounded by $\tilde{L}^{\max}$ for all $k \in \mathcal{K}$, we have that, as $K \rightarrow \infty$,
\begin{equation} \label{eq:avgd_combconv}
\begin{aligned}
    & \lim_{K \rightarrow \infty} \frac{1}{K} \sum_{k \in \mathcal{K}} \sum_{q = 0}^{\tau_a + \tau_r - 1} \left\Vert \nabla L(\boldsymbol{\theta}^{k,q}) \right\Vert^2 \\
    & \leq 
    \frac{(2\tau_r - 1 ) L(\boldsymbol{\theta}^{k,0}) + \alpha C_1 \sum_{s \in \mathcal{S}} \Gamma_s}{ \eta - \frac{\hat{\alpha} \eta^2}{2} \Gamma } \\
    & + 4 (\tau_r - 1) \left( \frac{ \gamma^{\mathsf{eff}} \left( 1 - \lambda_N({\boldsymbol{A}}) \right) }{1-\lambda_2({\boldsymbol{A}})} \right)^2 \Vert \hat{\Delta}^{k,0} \Vert^2  + 4 C_2.
\end{aligned}
\end{equation}
In other words, the average global first-order stationary point, integrated across intra-cluster regimes $\tilde{k} \in \tilde{\mathcal{K}}$ and inter-cluster regimes $\hat{k} \in \hat{\mathcal{K}}$, is bounded above by a constant independent of the global round $k$. 
Therefore, Theorem~\ref{thm:combined_conv} indicates that SSD-FL yields bounded global convergence, with finite effective global loss function gradients. % and overall stable ML model training across global cycles. % even in the presence of heterogeneous optimizers. 
% ensuring that the aggregated gradients remain finite and that the overall training dynamics stay stable across global cycles.
% ~\eqref{eq:avgd_combconv} indicates that SSD-FL yields average first-order stationary points that are strictly bounded. 

%%%%%%%% some important notes on the graph cheeger inequalities -- from our friend next door who works on more graph theoretic aspects 
% {\color{blue} 
% Two points:
% (i) Need to add an additional result on loss gap between clustered and native
% This will tie in nicely with the spectral clustering approach in Algorithm~\ref{alg:spectral_routine}.

% (ii) Remove the separate subroutine for Cheeger's inequality
% No longer necessary if we just use the spectral gap directly in the bound (and later in Algorithm~\ref{alg:spectral_routine}).}

\section{Cluster Formation}
\label{sec:cluster_formation}

The theoretical results on convergence and consensus of SSD-FL in Sec.~\ref{sec:theory} assumed a general case with $1 \leq S \leq N$ total clusters to partition the network of $N$ devices.
Now, we leverage those results to develop SSD-FL's cluster formation algorithm, determining both an optimal number of clusters $S = \vert \mathcal{S} \vert$ and the constituent devices therein, i.e., $\mathcal{N}_s, \forall s \in \mathcal{S}$. 
The key cluster formation steps are summarized in Algorithm~\ref{alg:cluster_form}.

\subsection{Conductance criteria} \label{ssec:get_conduct_criteria}
To develop a conductance criteria for cluster formation, we revisit Lemma~\ref{thm:intra_cons}.
For the first global round $k=0$, the intra-cluster consensus gap $\Vert \Delta^{\tilde{k},0}_s \Vert = 0$ for any and all possible clusters $s \in \mathcal{S}$ and $S \in \{1, \cdots, N\}$ (as the optimal number of clusters $S$ is unknown), specifically because network devices are initialized with the same local ML model parameters so that $\theta^{0,0}_i = \theta^{0,0}_j$ $\forall i,j \in \mathcal{N}$. 
Setting $\Vert \Delta^{\mathsf{tol}} \Vert$ as the limit on the tolerable consensus gap across all $\tilde{k} \in \tilde{\mathcal{K}}$, we can then obtain the following by rearranging Lemma~\ref{thm:intra_cons} for any cluster $s \in \mathcal{S}$ and $S \in \{1, \cdots, N\}$
\begin{equation} \label{eq:delta_tolerable_init}
    % \Vert \Delta^{\mathsf{tol}}_S \Vert \leq \frac{2\eta B \sqrt{\floor{N/S}}}{1 - \eta - \lambda_2(\tilde{\boldsymbol{A}})}, 
    \eta + \frac{2 \eta B \sqrt{\floor{N/S}}}{\Vert \Delta^{\mathsf{tol}} \Vert} \leq 1 - \lambda_2(\tilde{\boldsymbol{A}}_s).
\end{equation} 
Noting that $I - \tilde{\boldsymbol{A}}_s$ is equivalent to the normalized Laplacian for any cluster $s \in \mathcal{S}$ as a result of Assumption~\ref{ass:adjacency_matrix}, we can then leverage Cheeger's inequality~\cite{chung1997spectral}, which states that
\begin{equation} \label{eq:cheeger_upper}
    % \frac{1}{2} (1 - \lambda_2(\tilde{\boldsymbol{A}}_s)) \leq \Phi^{\min}_{S}, 
    \frac{(\Phi^{\min}_S)^2}{2} \leq 1 - \lambda_2(\tilde{\boldsymbol{A}}_s)
\end{equation}
where $\Phi^{\min}_{S}$ denotes the minimum conductance threshold  for $S \in \{1, \cdots, N\}$.
By inspection of~\eqref{eq:delta_tolerable_init} and~\eqref{eq:cheeger_upper}, we have that
\begin{equation} \label{eq:final_cheeger_conduct}
    % \Phi^{\min}_S = \frac{\eta}{2} + \frac{\eta B \sqrt{\floor{N/S}}}{\Vert \Delta^{\mathsf{tol}}_S \Vert}.
    \Phi^{\min}_S = \sqrt{2\eta + \frac{4 \eta B \sqrt{\floor{N/S}}}{\Vert \Delta^{\mathsf{tol}} \Vert}}.   
\end{equation}
To summarize, given some set of clusters $\mathcal{S}$,~\eqref{eq:final_cheeger_conduct} adapts a minimum conductance threshold $\Phi^{\min}_S$ inversely proportional to the maximum tolerable intra-cluster consensus gap $\Delta^{\mathsf{tol}}_S$. 
Moreover, $\Phi^{\min}_S$ changes with the number of clusters, as larger $S$ in~\eqref{eq:final_cheeger_conduct} reduces the conductance requirement for each cluster $s \in \mathcal{S}$. This is intended because more clusters results in fewer devices per cluster (on average), which in turn reduces the likelihood of more divergent datasets (as compared to clusters with more devices). 
% means less divergent datasets 
% and smaller clusters are expected to have less divergent datasets. 

% {\color{blue} add a sentence that this bound changes with the number of clusters. more clusters, yields a smaller conductance requirement. this is ok because intra-cluster regimes enable smaller clusters with fewer devices are expected to have less divergence datasets and therefore more an easier time reaching consensus/convergence.}
% that depends on the number of clusters $S$ as well as the 

%% previous notes on cheeger's inequalities in zold_contents.tex

\begin{algorithm}[t] 
    \caption{\textsc{Cluster Formation in SSD-FL}}
    \label{alg:cluster_form}
    \begin{algorithmic}[1] 
    \STATE \textbf{Input:} Network graph $G = (\mathcal{N}, \boldsymbol{A})$, intra-cluster duration $\tau_a$, inter-cluster duration $\tau_r$, learning rate $\eta$, bound $B$, maximum tolerable consensus gap $\Vert \Delta^{\mathsf{tol}}_S \Vert$, and effective smoothness coefficients $\gamma^{\mathsf{eff}}_s$ and $\gamma^{\mathsf{eff}}$. 
    \STATE \textbf{Output:} Optimal set of clusters $\mathcal{S}^{*}$. 
    \STATE Initialize sets of candidate partitions $\tilde{\mathcal{S}} = \{ \{G\} \}$, estimated average first-order stationary points $\tilde{\mathcal{H}} = \{ \}$, and minimum conductance thresholds $\boldsymbol{\Phi}^{\min} = \{ \}$. % and minimum graph partition conductances $\hat{\boldsymbol{\Phi}}^{\min} = \{ \}$. 
    \STATE Initialize $\hat{\mathcal{S}} = \{ G \}$ as the starting partition of the original network $G$. 
    \WHILE{$\vert \hat{\mathcal{S}} \vert \leq N$}
        \STATE Determine conductance threshold $\Phi^{\min}_{\vert \hat{\mathcal{S}} \vert}$ via~\eqref{eq:final_cheeger_conduct}. 
        % \FOR{ $s \in \hat{\mathcal{S}}$ }
        %     \STATE Determine effective intra-cluster gradient bound, $\tilde{h} = \frac{\tilde{L}_s(\hat{\boldsymbol{\theta}}^{\tilde{k},0}_s) + 
        %     \frac{\alpha \tau_a \eta^2}{2} \Gamma_s}{\eta - \frac{\hat{\alpha}_s \eta^2}{2} \Gamma_s}$ from~\eqref{eq:intra_conv_statement}.  
        %     \STATE 
        % \ENDFOR
        % \STATE Determine the average (over the clusters $\hat{\mathcal{S}}$) intra-cluster first-order stationary point from Theorem~\ref{thm:intra_conv} in~\eqref{eq:intra_conv_statement}. 
        % \FOR{$s \in \hat{\mathcal{S}}$}
        %     \STATE Estimate $\alpha_s^o$ and $\alpha_s^d$ via the processes in~\eqref{eq:optim_het_alpha_s1}-\eqref{eq:data_het_alpha_s2}, and thereafter obtain $\alpha_s$. 
        %     \STATE Determine the bound for effective intra-cluster first-order gradient, i.e., $\frac{\tilde{L}_s(\hat{\boldsymbol{\theta}}^{\tilde{k},0}_s) +         \frac{\alpha \tau_a \eta^2}{2} \Gamma_s}{\eta - \frac{\hat{\alpha}_s \eta^2}{2} \Gamma_s}$ from Theorem~\ref{thm:intra_conv}. 
        % \ENDFOR 
        \STATE Track the average (over clusters $\hat{\mathcal{S}}$) intra-cluster first-order stationary point from Theorem~\ref{thm:intra_conv} in~\eqref{eq:intra_conv_statement}, i.e., $\tilde{\mathcal{H}} \leftarrow \tilde{\mathcal{H}} \cup \left\{ \frac{1}{\hat{S}} \sum_{s \in \hat{\mathcal{S}}}  \frac{\tilde{L}_s(\hat{\boldsymbol{\theta}}^{\tilde{k},0}_s) + 
        \frac{\alpha \tau_a \eta^2}{2} \Gamma_s}{\eta - \frac{\hat{\alpha}_s \eta^2}{2} \Gamma_s} \right\}$. This average relies on $\alpha_s^o$ and $\alpha_s^d$, $\forall s \in \hat{\mathcal{S}}$, estimates via the processes in~\eqref{eq:optim_het_alpha_s1}-\eqref{eq:data_het_alpha_s2} to obtain $\alpha_s$. 
        
        % \FOR{Each subgraph $s \in \hat{\mathcal{S}}$}
        %     \STATE Compute conductance $\Phi(\tilde{G}_s)$ via~\eqref{eq:graph_conduct_defn}.
        % \ENDFOR 
        \STATE Sort clusters $s \in \hat{\mathcal{S}}$ in ascending conductance, i.e., let $\mathcal{Q} = \{s^{(1)}, \dots, s^{(|\mathcal{S}|)}\} \gets \operatorname{sort}_{s \in \hat{\mathcal{S}}}(\Phi(\tilde{G}_s))$. 
        %% the conductance checker during partitioning
        % \STATE Get best partition $P^{*} = \{s_a^{*}, s_b^{*} \}$ from \textsc{Spectral Partitioning} in Algorithm~\ref{alg:spectral_routine}. 
        % \STATE \textsc{Spectral Partitioning} in Algorithm~\ref{alg:spectral_routine} returns updated partition of clusters $\hat{\mathcal{S}}$. 
        \STATE $\hat{\mathcal{S}} \gets \textsc{SpectralPartitioning}(\Phi^{\min}_S, \mathcal{Q})$.  %Update partition of clusters in $\hat{\mathcal{S}}$ via 
        
        % %% adjust Q with P, then reinitialize \hat{S} 
        % %% S = S + 1
        % %% then run intra-cluster estimates to finish
        % \STATE Select $(s^{a*}, s^{b*}) = \arg\min_{(s^{a}, s^{b})} \Phi(s^{a}, s^{b})$. 
        % \IF{$\Phi(s^{a*}, s^{b*}) \le \Phi^{\min}_S$}
        %     \STATE Replace $s^{(0)}$ in $\mathcal{S}$ with $\{s^{a*}, s^{b*}\}$.
        %     \STATE $S \gets S + 1$.
        % \ELSE
        %     \STATE Mark $s^{(0)}$ as non-partitionable.
        %     \STATE Attempt next smallest-conductance subgraph $s^{(1)}$.
        %     % \If{all $s^{(i)} \in \mathcal{S}$ are non-partitionable}
        %     %     \STATE Force partition on $s^{(0)}$ with $(s^{a*}, s^{b*}) = \arg\min_{(s^{a}, s^{b})} \Phi(s^{a}, s^{b})$.
        %     %     \STATE Replace $s^{(0)}$ in $\mathcal{S}$ with $\{s^{a*}, s^{b*}\}$.
        %     %     \STATE $S \gets S + 1$.
        %     % \EndIf
        % \ENDIF
        \STATE Update set of all candidate partitions, $\tilde{\mathcal{S}} \gets \tilde{\mathcal{S}} \cup \{\hat{\mathcal{S}}\}$.
    \ENDWHILE

    % \FOR{each clustering $\mathcal{S} \in \tilde{\mathcal{S}}$}
    %     \FOR{each cluster $s \in \mathcal{S}$}
    %         \STATE Estimate intra-cluster stationary point $\Vert \nabla \tilde{L}_s \Vert^2$ using Theorem~\ref{thm:intra_conv}.
    %     \ENDFOR
    %     \STATE Compute average stationary point $\bar{L}_{\mathcal{S}} = \frac{1}{|\mathcal{S}|} \sum_{s \in \mathcal{S}} \Vert \nabla \tilde{L}_s \Vert^2$.
    % \ENDFOR

    \STATE Find the optimal cluster $\mathcal{S}^{*}$ that satisfies the conductance requirements in $\Phi^{\min}_S$ with minimum average first-order stationary point, i.e., 
    $\mathcal{S}^{*} = \arg\min_{\mathcal{S} \in \tilde{\mathcal{S}}} \tilde{\mathcal{H}}$ subject to $\min_{s \in \mathcal{S}} \Phi(\tilde{G}_s) \geq \Phi^{\min}_S$.
    \STATE \textbf{Return} $\mathcal{S}^{*}$.
    \end{algorithmic}
\end{algorithm}

\subsection{Integration of heterogeneous optimizers} \label{ssec:het_optim_alg} 
Within any cluster $s \in \mathcal{S}$, the internal cluster heterogeneity influences relative gradient norm amplification ${\alpha}_s$ and $\hat{\alpha}_s$ as per Assumption~\ref{ass:grad_var}, and, in turn, $\hat{\alpha}_s$ greatly influences the resulting convergence results in Theorem~\ref{thm:intra_conv} and~\ref{thm:combined_conv}. 
%\alpha_s = rate of growth of the variance of effective intra-cluster stochastic gradients 
As the two primary forms of D2D heterogeneity are at the data-level and optimizer-level, we define $\alpha_s = \alpha_{s}^{o} + \alpha_{s}^{d}$ where $\alpha_s^o$ and $\alpha_s^d$ are the optimizer-induced and data-induced gradient norm amplification coefficients respectively. %% it is a bit unclear if this should be linear, but linear worked experimentally

Since both $\alpha_s^o$ and $\alpha_s^d$ measure internal cluster differences, we obtain them via D2D pairwise comparisons. 
Towards optimizer-induced heterogeneity, we first obtain
\begin{equation} \label{eq:optim_het_alpha_s1}
    \beta_s^o = \frac{1}{N_s^2}  \sum_{i \in \mathcal{N}_s} \sum_{j \in \mathcal{N}_s} \zeta_1 \mathbf{1}_{[\text{opt}_i \neq \text{opt}_j]} + \zeta_2 \Vert \mu_i - \mu_j \Vert + \zeta_3 \Vert \rho_i - \rho_j \Vert, 
\end{equation}
where $\zeta_1$, $\zeta_2$, and $\zeta_3$ are scaling coefficients for differentials in optimizer, proximal parameters $\mu_i$ and $\mu_j$, and momentum parameters $\rho_i$ and $\rho_j$. 
Subsequently, we linearly scale $\beta^o_s$ to obtain $\alpha^o_s$ as follows
\begin{equation} \label{eq:optim_het_alpha_s2}
    \alpha_s^{o} = \frac{\beta^o_s}{\zeta_1 + \zeta_2 + \zeta_3} (\alpha^{o,\max} -\alpha^{o,\min}) + \alpha^{o, \min}, 
\end{equation} 
where $\alpha^{o,\max}$ and $\alpha^{o, \min}$ denote the max and min contributions to optimizer heterogeneity scaling in $\alpha^o_s$, respectively. %$\alpha_s$ 
On the other hand, the data heterogeneity estimation relies on a combination of empirical average Jensen-Shannon divergence (JSD)~\cite{fuglede2004jensen} of relative frequencies and empirical energy distance (EED)~\cite{rizzo2016energy} of a sample of raw data from device-level datasets. % and empirical energy distance~\cite{rizzo2016energy}. 
We express this as
\begin{equation} \label{eq:data_het_alpha_s}
    \beta_s^{d} = \frac{1}{ \vert \tilde{\mathcal{E}}_s \vert} \sum_{\substack{i,j \in \mathcal{N}_s \\ (i,j) \in \tilde{\mathcal{E}}_s}} 
    \left( \text{JSD} (\mathcal{Y}_i \Vert \mathcal{Y}_j) + \frac{1}{wz} \text{EED}( \hat{\mathcal{D}}_i, \hat{\mathcal{D}}_j ) \right),
\end{equation}
where $\vert \tilde{\mathcal{E}}_s \vert$ denotes the cardinality of $\tilde{\mathcal{E}}_s$, $\text{JSD}$ represents the Jensen-Shannon divergence, $\mathcal{Y}_i$ is the relative frequency of labels within device $i$'s dataset $\mathcal{D}_i$, $wz$ represents the total data features, 
$\text{EED}( \hat{\mathcal{D}}_i, \hat{\mathcal{D}}_j ) = \frac{2}{\hat{D}_i \hat{D}_j} \sum_{h \in \hat{\mathcal{D}}_i, m \in \hat{\mathcal{D}_j}} \Vert x_h - x_m \Vert - \frac{1}{\hat{D}_i^2} \sum_{h,m \in \hat{D}_i} \Vert x_h - x_m \Vert - \frac{1}{\hat{D}_j^2} \sum_{h,m \in \hat{D}_j} \Vert x_h - x_m \Vert$ as the squared empirical energy distance from~\cite{rizzo2016energy}, and $\hat{\mathcal{D}}_i$ denotes a randomly chosen batch of data of size $\hat{D}_i$ from device $i$. Note that $\vert \hat{\mathcal{D}}_i \vert = \vert \hat{\mathcal{D}}_j \vert$, for any $(i,j) \in \tilde{\mathcal{E}}_s$.

The structure of~\eqref{eq:data_het_alpha_s} in that both JSD and EED are used to estimate pairwise and total cluster similarities is because SSD-FL aims to avoid wholesale D2D data sharing. 
Instead, JSD enables SSD-FL to measure devices' differences in distribution, as relative frequency in labels can act as a proxy for empirical dataset distribution. 
Simultaneously, EED on a randomly chosen subset of data $\hat{\mathcal{D}}_i$ and $\hat{\mathcal{D}}_j$ still enables a measure of the nominal differences between devices' datasets,  especially as it subtracts the internal gap in devices' local datasets. 
% % In particular, as relative frequencies act as a proxy for empirical distribution, JSD can thus measure devices' differences in labels. 
% % is chosen to compare devices' set of labels as it is a measure of differences in likelihood. 
% By contrast, devices' data feature similarities are nominal differences and, as such, we aim to compare the distances between individual samples (and the features therein).
% Here, EED, by subtracting the internal gap in devices' local datasets, can extract a better measure of the nominal differences between devices' datasets. %more robust measure of the nominal differences between devices' datasets. 
Next, to obtain $\alpha^{d}_s$, we scale $\beta^{d}_s$ linearly as in~\eqref{eq:optim_het_alpha_s2}, obtaining
\begin{equation} \label{eq:data_het_alpha_s2}
    \alpha_s^d = \beta^{d}_s (\alpha^{d,\max} - \alpha^{d,\min}) + \alpha^{d,\min}
\end{equation}
where $\alpha^{d,\max}$ and $\alpha^{d, \min}$ denote the max and min contributions to D2D data heterogeneity scaling in $\alpha^d_s$, respectively.

\begin{algorithm}[t] 
    \caption{\textsc{Spectral Partitioning}}
    \label{alg:spectral_routine}
    \begin{algorithmic}[1] 
    \STATE \textbf{Input:} Conductance threshold $\Phi^{\min}_{\hat{S}}$ and sorted clusters $\boldsymbol{Q} = \{s^{(1)}, \dots, s^{(|\mathcal{S}|)}\}$.
    \STATE \textbf{Output:} Updated and partitioned cluster set $\boldsymbol{Q}$. 
    % partition $P^{*} = \{s_a^{*}, s_b^{*} \}$.
    \FOR{Each cluster $s \in \boldsymbol{Q}$}
        \STATE Compute Fiedler eigenvector $\nu_2(s)$ and obtain sorted indices $\pi = \mathrm{argsort}(\nu_2(s))$.
        \STATE Initialize minimum conductance of possible partitions $\tilde{\Phi}^{\min}_{Q} \leftarrow 0$. 
        % and $P \leftarrow \{ \}$.
        \FOR{Index $n = 1$ to $|s|$} 
            \STATE Define two candidate subsets: $s^{a}_n = \{ \pi(1), \ldots, \pi(n) \}$ and $s^{b}_n = s \setminus s^{a}_n$. 
            \STATE Compute minimum conductance: $\Phi_n = \min \{ \Phi(\tilde{G}_{s^{a}_n}), \Phi(\tilde{G}_{s^{b}_n}) \}$. 
            \IF{$\Phi_n > \tilde{\Phi}^{\min}_{Q}$}
                \STATE Update minimum conductance of possible partitions,  $\hat{\Phi}^{\min}_{\hat{S}} \leftarrow \Phi_n$.
                \STATE Update intermediary best candidate partition, $\mathcal{P} \leftarrow \{ s^{a}_n, s^{b}_n \}$.
            \ENDIF 
        \ENDFOR 
        %% check relative to conductance threshold
        \IF{ $\hat{\Phi}^{\min}_{\hat{S}} \geq \Phi^{\min}_{\hat{S}}$}
            % \STATE Determine viable candidate $P^{*} = P$. 
            \STATE Update cluster set: $\boldsymbol{Q} \gets (\boldsymbol{Q} \setminus \{s\}) \cup \mathcal{P}$. %$\{s^{a}_n, s^{b}_n\}$.
            \STATE \textbf{return} $\boldsymbol{Q}$.
        \ELSIF{ $\hat{\Phi}^{\min}_{\hat{S}} < \Phi^{\min}_{\hat{S}}$ and $s$ is $s^{(1)}$}
            %% save the  parameters 
            \STATE Save the best candidate partition, $\tilde{\mathcal{P}} \leftarrow \mathcal{P}$.
        \ENDIF 
    \ENDFOR
    \STATE Update cluster set: $\boldsymbol{Q} \gets (\boldsymbol{Q} \setminus \{s^{(1)}\}) \cup \tilde{\mathcal{P}}$. %$\{s^{a}_n, s^{b}_n\}$.
    \STATE \textbf{return} $\boldsymbol{Q}$. 
\end{algorithmic}
\end{algorithm}

\subsection{Combined cluster formation} \label{ssec:alg_piece_together}
%% basically just summarize the steps (and explain why) for the algorithm 1 above

At a high level, SSD-FL’s cluster formation, summarized in Algorithm~\ref{alg:cluster_form}, iteratively partitions the network based on spectral structure and expected ML model training convergence (i.e., Theorem~\ref{thm:intra_conv}). 
Specifically, SSD-FL iteratively increases the number of clusters $\vert S \vert$ from $1$ to $N$, the size of the network. 
Starting with the original network graph $G = (\mathcal{N}, \boldsymbol{A})$, we denote the current partition of the network as $\hat{\mathcal{S}}$, and thus start with $\hat{\mathcal{S}} = \{G \}$ (and single cluster as $\vert \hat{\mathcal{S}} \vert = 1$). 
SSD-FL then computes the conductance $\Phi(\tilde{G}_s)$ of each subgraph $s \in \hat{\mathcal{S}}$ using~\eqref{eq:graph_conduct_defn}, and simultaneously determines the minimum conductance threshold $\Phi^{\min}_S$ from the Cheeger-based bound in~\eqref{eq:final_cheeger_conduct}. 
% This threshold establishes the minimum connectivity required to ensure bounded intra-cluster consensus error during training. 

For each iteration, SSD-FL evaluates the current cluster set $\hat{\mathcal{S}}$ via Theorem~\ref{thm:intra_conv} (and the $\alpha_s$ estimation process from Sec.~\ref{ssec:het_optim_alg}) to obtain an average effective intra-cluster first-order stationary point, stored in $\tilde{\mathcal{H}}$. 
Simultaneously, SSD-FL ranks the clusters within the current cluster set, i.e., $s \in \hat{\mathcal{S}}$, in ascending order of their conductance, forming a sorted set $\mathcal{Q} = \{ s^{(1)}, \dots, s^{(|\hat{\mathcal{S}}|)} \}$. 
In this way, the least-connected (and hence most separable as well as weakest internal consensus) clusters are examined first. 
We next apply the spectral partitioning process, detailed in Algorithm~\ref{alg:spectral_routine}. 
In this process, each cluster $s \in \mathcal{Q}$, starting with $s^{(1)}$, is partitioned by analyzing the Fiedler eigenvector $\nu_2(s)$ of its normalized Laplacian matrix, which corresponds to $\mathbf{I}_s - \tilde{\mathbf{A}}_s$ as a result of Assumption~\ref{ass:adjacency_matrix}. 
Within the Fiedler vector $\nu_2(s)$, devices with similar eigenvector values are more connected, while those with large gaps indicate weaker connectivity~\cite{chung1997spectral}.  
SSD-FL sweeps through $\nu_2(s)$, identifying the partition $\mathcal{P} = \{s^a, s^b\}$ of $s$ with the largest minimum conductance. 
If partition $\mathcal{P}$ has conductance over threshold $\Phi^{\min}_S$, then the set $\mathcal{Q}$ is updated as $\mathcal{Q} = (\hat{\mathcal{S}} \setminus \{ s \}) \cup \{ s^a, s^b \}$.
Otherwise, SSD-FL proceeds to the next smallest conductance cluster $s^{(n)}$ in $\mathcal{Q}$. % $\forall s^{(n)} \in \mathcal{Q}$
However, if no partition satisfies the threshold, then the original cluster with the smallest conductance, i.e., $s^{(1)}$, will be partitioned following the above rules. 
This post-partition cluster candidate is then stored as the new $\hat{\mathcal{S}}$ and in the candidate set $\tilde{\mathcal{S}}$. 
% The updated cluster configuration is stored in the candidate set  along with its corresponding average stationary-point estimate $\tilde{\mathcal{H}}$. 

SSD-FL continues the above process iteratively, until $\vert \hat{\mathcal{S}} \vert = N$, or, in other words, there is a candidate partition of every feasible size for a network with $N$ devices.
Among these possible partitions $\mathcal{S} \in \tilde{\mathcal{S}}$, SSD-FL determines $\mathcal{S}^{*} = \arg\min_{\mathcal{S} \in \tilde{\mathcal{S}}} \tilde{\mathcal{H}} \text{ s.t. } \Phi(\mathcal{S}) \ge \Phi^{\min}_S$. 
As such, $\mathcal{S}^{*}$ corresponds to the set of clusters that (i) maintains sufficient intra-cluster connectivity and (ii) yields the lowest average effective first-order gradient.
Therefore, SSD-FL's cluster formation is based on both graph topology and estimated ML performance. 
% across the network. 
% This integrated spectral–optimization clustering ensures that SSD-FL forms communication-efficient and learning-stable clusters tailored to both graph topology and training dynamics.

% For each partition point, i.e., for all $S \in \{1, \cdots, N\}$, SSD-FL will estimate the effective first-order stationary point from Theorem~\ref{thm:intra_conv}, leveraging the gradient norm amplification factor $\alpha_s$ from Sec.~\ref{ssec:het_optim_alg} and taking the average for all the clusters $s \in \mathcal{S}$ given each $S \in \{ 1, \cdots, N \}$. 
% Finally, the set of clusters $\mathcal{S}$ with the lowest average first-order stationary point and meeting the minimum conductance threshold $\Phi^{\min}_S$ is set to be the optimal set of clusters $\mathcal{S}^{*}$. 

% alpha_s and \Gamma_s doing all the heavy lifting really

% In terms of intuition, you're basically saying that order of importance for clusters: (i) graph connectivity/conductance requirement, (ii) average intra-cluster divergence -> intra-cluster convergence bound -> average global intra-cluster convergence bound

\section{Experimental Evaluation}
\label{sec:experiments}

In the following, we evaluate the performance of the proposed SSD-FL methodology across four dimensions, organized to highlight its core advantages, progressively. 
To this end, we present the experimental setup in Sec.~\ref{ssec:exp_setup}.
Then, we first examine the impact of inter-cluster period $\tau_r$ in Sec.~\ref{ssec:taur_exps} and intra-cluster period $\tau_a$ in Sec.~\ref{ssec:taua_exps}, as these results most directly demonstrate the impact of careful and deliberate cluster formation, which is our central contribution. 
Subsequently, we evaluate the scalability of SSD-FL relative to baselines via varying network size in Sec.~\ref{ssec:diff_net_sizes}, before concluding with performance across various network graph architectures in Sec.~\ref{ssec:diff_net_styles} in order to establish SSD-FL's general robustness. 
These experiments are performed for with and without heterogeneous optimizers, though the homogeneous SGD optimizer experiments are left to Appendix~\ref{app_sec:exps} for conciseness. 
Similarly, additional experiments on link formation probabilities and on SSD-FL's intra-cluster convergence bound are also available in Appendix~\ref{app_sec:exps}.

\begin{figure}[t]
    \centering
    \includegraphics[width=0.98\linewidth]{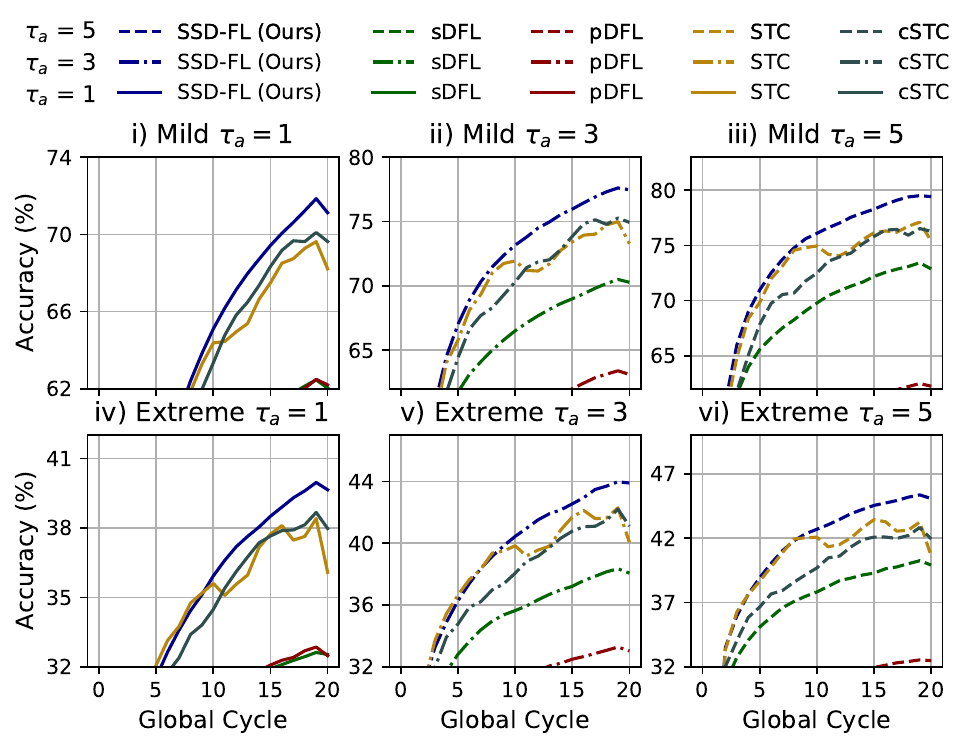}
    \caption{Varying intra-cluster period $\tau_a$ for FMNIST. 
    SSD-FL's advantage over baselines grows with $\tau_a$, and remains consistent across both mild and extreme heterogeneity settings.}
    % As intra-cluster periods get longer, i.e., larger $\tau_a$, SSD-FL improves in performance relative to baselines.}
    \label{fig:het_taua_fmnist} 
    \vspace{-4mm}
\end{figure}

\subsection{Experimental setup} \label{ssec:exp_setup}
%% ml model architure + ml relevant things, such as datasets, optimizer split
% data quantity per device,
% add a sentence - simple ML model architecture as the point is the validate the above theoretical results, and as such obtaining SOTA or near SOTA accuracies is not currently our intention 
The experiments are performed on FMNIST~\cite{xiao2017online}
and CIFAR10~\cite{krizhevsky2009learning}, with their respective training datasets of size $60000$ and $50000$ samples evenly partitioned across the network devices. 
The exact partition depends on the notion of data heterogeneity across the network, and, here, we consider mild and extreme non-i.i.d. scenarios, which correspond to cases where each device has data drawn from $3$ or $1$ label of the full dataset. % %To examine the degree of local dataset heterogeneity across the network, %within the main manuscript 
Moreover, as experiments involve heterogeneous ML optimizers at devices, we randomly assign each device an optimizer and, for proximal and momentum optimizers, we further assign a proximal parameter $\mu_i$ drawn uniformly at random from $\{5\times10^{-5} , 1\times 10^{-4}\}$ or a momentum parameter $\rho_i$ drawn uniformly at random from $\{0.8, 0.85\}$. 
Finally, the ML models used are five layer CNNs, with output channel dimension $32$, $64$, $128$, $128$, and $256$ sequentially, followed by a single linear layer. 
These more traditional neural networks are used because our goal is primarily the proof-of-concept of SSD-FL for further exploration of clustering (and network structure manipulation more generally) in decentralized FL settings and, as such, obtaining state-of-the-art (SOTA) or near SOTA accuracies are not our intention.

%% physical graph information 
%% network architecture (erdos-renyi + link formation probabilities), underlying number of nodes, total training time, default intra-cluster and inter-cluster periods 
Unless otherwise stated, the underlying networks are based on Erdős–Rényi random graphs~\cite{gilbert1959random} with link formation probability $10\%$ and size $30$ devices. Additionally, all experiments are run for $20$ total global cycles, with intra-cluster period $\tau_a = 3$ and inter-cluster period of $\tau_r = 1$.
%% adjacency matrix weights (how to get doubly stochastic and symmetric weights)
For all experiments, networks' adjacency matrices, including the inter-cluster and intra-cluster graph matrices $\boldsymbol{A}$ and $\tilde{\boldsymbol{A}}$, are based off of Metropolis-Hastings weights~\cite{xiao2006distributed}. 
% Metropolis-hastings weights~\cite{xiao2006distributed}, for $i \neq j$, weight is $1 / (1 + \max(d_i,d_j))$ and $i = j$ ensures doubly stochastic. 
% experiment runtime 
%% add a few sentences about the baselines
To contextualize performance, we examine SSD-FL relative to four classes of baseline decentralized FL methodologies: (i) synchronous (sDFL)~\cite{liu2022decentralized}, (ii) periodic (pDFL)~\cite{sun2022decentralized}, (iii) stochastic (STC)~\cite{koloskova2020unified}, and (iv) clustered stochastic (cSTC), which determines the total number of clusters randomly and thereafter follows stochastic~\cite{koloskova2020unified}. %augments stochastic decentralized FL with randomly determined number of clusters 
Moreover, for fairness, these baseline decentralized FL methodologies will have dedicated training rounds and additional D2D network communications adhering to $\tau_a$ and $\tau_r$ respectively. 
% this strcture, with dedicated training rounds and additional D2D network communications adhering to $\tau_a$ and $\tau_r$ respectively. 
Finally, regarding SSD-FL's cluster formation parameters, we use $\Delta^{\mathsf{tol}}_S = 10$ for $S \in \{1, \cdots, N\}$, and $\alpha = 0.1$. 
To derive $\alpha_s$, we use an equal weighting in~\eqref{eq:optim_het_alpha_s1} with $\zeta_1$, $\zeta_2$, and $\zeta_3 = 1$, while, for the min-max scalings in~\eqref{eq:optim_het_alpha_s2} and~\eqref{eq:data_het_alpha_s2}, we use $\alpha^{o,\max} = 0.2$ and $\alpha^{o,\min} = 0$ as well as $\alpha^{d,\max} = 0.2$ and $\alpha^{d,\min} = 0$, respectively.

\begin{figure}[t]
    \centering
    \includegraphics[width=0.98\linewidth]{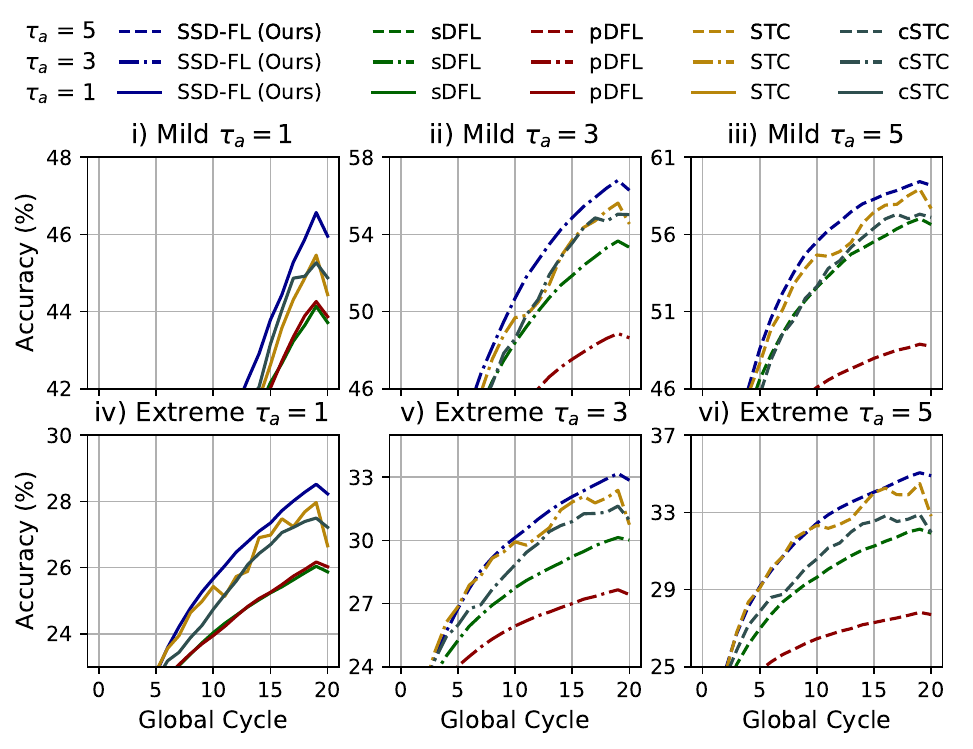}
    \caption{Varying intra-cluster period $\tau_a$ for CIFAR10. 
    The performance gap between SSD-FL and baselines is more pronounced under extreme heterogeneity.}
    % As intra-cluster periods get longer, i.e., larger $\tau_a$, SSD-FL improves in performance relative to baselines.}
    \label{fig:het_taua_cifar}
    \vspace{-4mm}
\end{figure}

%%% fmnist table -- these are results for \tau_r experiments
\begin{table*}[t]
\caption{Examining the average global cycles needed to reach various accuracy threshold on FMNIST and for networks with heterogeneous ML optimizers at devices.
SSD-FL's advantage accumulates for higher accuracy thresholds. 
% While for low accuracy thresholds, all methods can fairly comparable, SSD-FL gains a notable advantage for larger accuracy threshold points. 
Dashes indicate thresholds that were not reached.}
{\footnotesize
\begin{tabularx}{0.99\textwidth}{m{3.5em} *{16}{X} }
\toprule[.2em]
\multirow{3}{*}{\textbf{Method}} 
& \multicolumn{8}{c}{\textbf{$\tau_r = 3$}}
& \multicolumn{8}{c}{\textbf{$\tau_r = 5$}}
\\
\cmidrule(lr){2-9} \cmidrule(lr){10-17}
& \multicolumn{4}{c}{\textbf{Mild non-i.i.d. acc}} 
& \multicolumn{4}{c}{\textbf{Extreme non-i.i.d. acc}}
& \multicolumn{4}{c}{\textbf{Mild non-i.i.d. acc}} 
& \multicolumn{4}{c}{\textbf{Extreme non-i.i.d. acc}}
\\
\cmidrule(lr){2-5} \cmidrule(lr){6-9} \cmidrule(lr){10-13} \cmidrule(lr){14-17}
& 51\% & 58\% & 65\% & 72\%
& 30\% & 35\% & 40\% & 45\% 
& 51\% & 58\% & 65\% & 72\%
& 30\% & 35\% & 40\% & 45\%
\\
\midrule 
SSD-FL & 2.89 & 4.16 & 6.56  & 11.75 & 2.63 & 4.93 & 8.98  & 17.75 & 2.64 & 3.68 & 5.59  & 9.84  & 2.44 & 4.35 & 7.42  & 13.52 \\
sDFL & 3.07 & 4.61 & 7.40  & 14.70 & 2.77 & 5.55 & 11.08 & -- & 2.72 & 3.83 & 5.87  & 10.66 & 2.45 & 4.34 & 7.90  & 15.22 \\
pDFL & 3.06 & 4.58 & 7.43  & 13.80 & 2.73 & 5.32 & 11.01 & -- & 2.70 & 3.83 & 5.88  & 10.51 & 2.44 & 4.35 & 7.78  & 15.36 \\
STC & 3.21 & 4.93 & 7.51  & 16.06 & 2.94 & 5.78 & 13.86 & -- & 3.02 & 4.75 & 6.65  & 13.96 & 2.67 & 5.55 & 13.78 & -- \\
cSTC & 3.15 & 4.66 & 7.51  & 14.30 & 2.92 & 5.84 & 12.15 & 19.12 & 2.95 & 4.23 & 6.29  & 12.11 & 2.80 & 5.51 & 10.36 & 17.55 \\
\bottomrule
\end{tabularx}
\label{tab:acc_thresholds_fmnist}
}
\end{table*}

%%% cifar10 table
\begin{table*}[t]
\caption{Average global cycles that decentralized FL methodologies need to reach or exceed accuracy thresholds on CIFAR10 when devices employ heterogeneous ML optimizers. 
SSD-FL, similar to the case in Table~\ref{tab:acc_thresholds_fmnist}, continues to demonstrate faster convergence for higher thresholds.} % Dashes indicate thresholds that were not reached.}
{\footnotesize
\begin{tabularx}{0.99\textwidth}{m{3.5em} *{16}{X} }
\toprule[.2em]
\multirow{3}{*}{\textbf{Method}} 
& \multicolumn{8}{c}{\textbf{$\tau_r = 3$}}
& \multicolumn{8}{c}{\textbf{$\tau_r = 5$}}
\\
\cmidrule(lr){2-9} \cmidrule(lr){10-17}
& \multicolumn{4}{c}{\textbf{Mild non-i.i.d. acc}} 
& \multicolumn{4}{c}{\textbf{Extreme non-i.i.d. acc}}
& \multicolumn{4}{c}{\textbf{Mild non-i.i.d. acc}} 
& \multicolumn{4}{c}{\textbf{Extreme non-i.i.d. acc}}
\\
\cmidrule(lr){2-5} \cmidrule(lr){6-9} \cmidrule(lr){10-13} \cmidrule(lr){14-17}
& 51\% & 58\% & 65\% & 72\%
& 30\% & 35\% & 40\% & 45\% 
& 51\% & 58\% & 65\% & 72\%
& 30\% & 35\% & 40\% & 45\%
\\
\midrule 
SSD-FL & 6.41 & 9.26 & 12.87 & 17.73 & 2.86 & 5.88 & 11.21 & -- & 5.90 & 8.57 & 11.63 & 15.88 & 2.60 & 4.84 & 8.91 & 15.38 \\
sDFL & 6.68 & 9.71 & 13.79 & 18.64 & 2.91 & 6.42 & 12.42 & -- & 6.20 & 8.90 & 12.11 & 16.61 & 2.63 & 5.13 & 9.91 & 16.98 \\
pDFL & 6.61 & 9.57 & 13.62 & 18.99 & 3.00 & 6.47 & 12.60 & -- & 5.99 & 8.66 & 11.79 & 16.16 & 2.46 & 5.02 & 9.42 & 17.31 \\
STC & 6.98 & 11.10 & 15.14 & --    & 3.40 & 8.47 & 17.38 & -- & 6.62 & 10.87 & 15.07 & 19.86 & 3.37 & 6.67 & 15.31 & -- \\
cSTC & 6.55 & 10.36 & 14.52 & 19.37 & 3.06 & 7.64 & 14.06 & -- & 6.25 & 9.89  & 13.85 & 17.92 & 2.92 & 7.52 & 13.28 & -- \\
\bottomrule
\end{tabularx}
\label{tab:acc_thresholds_cifar10}
}
\end{table*}

%%%% intra-cluser variations
%% general trends 
% ssd_fl increased gap to baselines as tau_a gets larger 
% gives some numbers from extreme FMNIST and CIFAR10
\subsection{Intra-cluster duration $\tau_a$} \label{ssec:taua_exps}
First, we examine the impact of intra-cluster period $\tau_a$ on SSD-FL and the various decentralized FL baselines in Fig.~\ref{fig:het_taua_fmnist}-\ref{fig:het_taua_cifar}, with $\tau_r=1$ to isolate the effects of $\tau_a$ and a network of $N=10$ devices. 
The intra-cluster period enables us to assess whether cluster formation offers value, specifically as longer local training periods within clusters (i.e., larger $\tau_a$) should benefit methods with more careful and deliberate cluster formation, while highlighting the drift and instability that result from random or no clustering.

This intuition is confirmed across both datasets and heterogeneity levels. 
At $\tau_a$ = 1, all methods perform comparably, with SSD-FL holding only a modest edge over the best baseline. 
However, as $\tau_a$ grows larger, SSD-FL pulls progressively further ahead. For example when $\tau_a = 5$ in extreme non-i.i.d. scenarios, SSD-FL leads the best baseline STC by roughly $4\%$ on FMNIST ($46\%$ vs $42\%$) and roughly $2\%$ on CIFAR10 ($35\%$ vs $33\%$), with the separation visible not just in final accuracy but throughout the convergence trajectory. 
The fact that this gap emerges and widens with $\tau_a$ rather than remaining constant suggests that SSD-FL's cluster formation is translating longer intra-cluster training periods into more useful model updates than the baselines.

Beyond final accuracies, SSD-FL also offers notably smoother and faster convergence curves relative to STC and cSTC across both datasets. 
Unlike STC and cSTC, both of which exhibit more erratic/noisy convergence behavior, SSD-FL converges steadily throughout, reflecting the intra-cluster stability induced by Algorithm~\ref{alg:cluster_form}.
Moreover, while sDFL and pDFL do offer smooth convergence curves, their accuracies are far lower than those obtained by SSD-FL, for example by roughly $9\%$ and $4\%$ on FMNIST and CIFAR10 in extreme non-i.i.d. settings at $\tau_a = 5$. 
Taken together, these points suggest that SSD-FL, via careful cluster formation, is able to effectively lead to intra-cluster stability (i.e., reduced intra-cluster differences), which in turn produces more useful local ML model updates and easier inter-cluster propagation across global rounds.

\subsection{Inter-cluster period $\tau_r$} \label{ssec:taur_exps} 
Next, we examine the impact of inter-cluster period $\tau_r$ on convergence speed in Tables~\ref{tab:acc_thresholds_fmnist} and~\ref{tab:acc_thresholds_cifar10}, by measuring the average number of global cycles needed to reach various accuracy thresholds on random graphs with $N=10$ devices and $\tau_a = 1$. 
Rather than final accuracy alone, convergence speed highlights the practical importance of both communication efficiency and training effectiveness, especially in large-scale edge/fog networks. 
Moreover, these experiments also examine the impact of changing $\tau_r \in [1,3,5]$, though the tables for $\tau_r = 1$ are left to Appendix~\ref{app_sec:exps} as their takeaways are similar to those in Tables~\ref{tab:acc_thresholds_fmnist} and~\ref{tab:acc_thresholds_cifar10}. 

% While the final performance accuracies were compared in Sec.~\ref{ssec:diff_net_styles} and~\ref{ssec:diff_net_sizes}, the convergence speeds among decentralized FL methodologies are also an important consideration. 
% In Tables~\ref{tab:acc_thresholds_fmnist} and~\ref{tab:acc_thresholds_cifar10}, we examine the number of global cycles $k$ needed to reach accuracy thresholds over random graphs. 
% Simultaneously, these experiments assess the impact of changing $\tau_r$, and are performed with networks of size $N=10$ nodes and $\tau_a=1$. %, the latter in particular to isolate the effect of $\tau_r$. 

For FMNIST in Table~\ref{tab:acc_thresholds_fmnist}, we see that SSD-FL nearly always requires fewer global rounds to reach the accuracy thresholds than the decentralized FL baselines. %yields faster convergence rates 
Moreover, the gap in global rounds needed between SSD-FL and the baselines increases with higher accuracy thresholds. 
On FMNIST under mild non-i.i.d. with $\tau_r = 3$, SSD-FL requires $12\%$ fewer global rounds than the best performing baseline pDFL to reach $65\%$ accuracy ($6.56$ vs $7.51$), a gap that widens to $15\%$ saving fewer rounds at $72\%$ accuracy ($11.75$ vs $13.80$).
Meanwhile, under extreme non-i.i.d. settings with $\tau_r = 3$, SSD-FL's advantage becomes more pronounced, requiring $18\%$ fewer rounds than pDFL to reach $40\%$ accuracy ($8.98$ vs $11.01$), and, alongside cSTC, is the only one of two methods to reach the $45\%$ threshold. 
Similarly, these trends continue to hold on CIFAR10 with $\tau_r = 3$ in Table~\ref{tab:acc_thresholds_cifar10}. 
Under mild non-i.i.d. settings, SSD-FL reaches $72\%$ accuracy in $17.73$ rounds vs $18.64$ for the best performing baseline sDFL. 
These savings become more pronounced in extreme non-i.i.d. settings, where SSD-FL requires $14\%$ fewer rounds than pDFL (best performing baseline) to reach $40\%$ accuracy ($11.21$ vs $12.60$). 

As $\tau_r$ increases to $5$, the absolute gap between SSD-FL and the best performing baselines become smaller. 
For example, on FMNIST and mild non-i.i.d. settings, SSD-FL's lead over pDFL at $72\%$ accuracy decreases from $2.05$ to $0.67$ global rounds. 
This is expected, however, as larger $\tau_r$ means more inter-cluster synchronization steps, which gives all methods more opportunities for global synchronization.

% For example, in mild non-i.i.d., SSD-FL and STC have a gap of $0.32$ global rounds at $51\%$ accuracy which grows to $4.31$ global rounds at $72\%$ accuracy. 
% As the inter-cluster periods get longer, i.e., larger $\tau_r$, the gap in global rounds needed between SSD-FL and the baselines decreases. 
% For instance, in mild non-i.i.d., SSD-FL uses $2.55$ fewer global rounds than cSTC to reach $72\%$ accuracy when $\tau_r = 3$, which decreases to $2.27$ when $\tau_r=5$.
% For the special case of pDFL with $\tau_r=5$ and extreme non-i.i.d., SSD-FL does use more global rounds to obtain $30\%$ and $35\%$ accuracies but pulls ahead after $40\%$ accuracy, requiring fewer global rounds than pDFL thereafter.
% From Table~\ref{tab:acc_thresholds_cifar10}, these insights for SSD-FL continue to hold for CIFAR10. 

% The trends are identical for the CIFAR10 experiments in Table~\ref{tab:acc_thresholds_cifar10}. However, here SSD-FL actually gets outperformed initially by pDFL, but rapidly increases in convergence speed, subsequently using fewer global rounds than pDFL for $35\%$ ... accuracy thresholds. 

While these previous experiments established SSD-FL's advantages in terms of controllable training hyper-parameters, we next evaluate its adaptability to various fixed network properties, such as network size, architecture, and link formation probabilities (in Appendix~\ref{app_ssec:linkp_exps}), which are defined by the network environments rather than something controlled by network operators.

%%%% network size
\subsection{Network size} \label{ssec:diff_net_sizes} 
We next examine the impact of network size from $N=10$ to $N=50$ for random graphs using both FMNIST in Fig.~\ref{fig:het_ns_fmnist}, and CIFAR10 in Fig.~\ref{fig:het_ns_cifar}. 
This experiment assesses the scalability benefits offered by SSD-FL, specifically that careful and deliberate cluster formation yields consistent advantages as edge/fog networks grow larger.

\begin{figure}[t]
    \centering
    \includegraphics[width=0.96\linewidth]{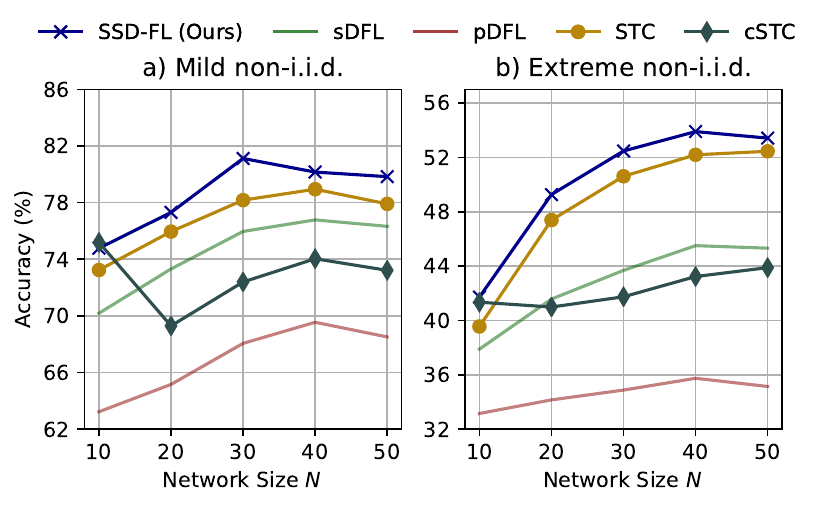}
    \caption{Varying network size from $N=10$ to $50$ with Erdős–Rényi random graphs on FMNIST. SSD-FL consistently yields better or equal performance relative to decentralized FL baselines.}
    \label{fig:het_ns_fmnist}
    \vspace{-4mm}
\end{figure}

Across both datasets and heterogeneity levels, SSD-FL consistently outperforms all baselines and maintains a stable performance gap as networks grow in size. 
% Specifically, across these experiments, SSD-FL outperforms the best performing baselines by roughly $1-3\%$ regardless of the network size, $N$. 
% On FMNIST under mild non-i.i.d. conditions, SSD-FL outperforms the best baseline sDFL by roughly 1-3\% across all network sizes, with the gap remaining stable rather than shrinking as N grows — a particularly encouraging sign for large-scale deployment. 
% Under extreme non-i.i.d., the advantage is similarly consistent, with SSD-FL outperforming the best baseline by a stable margin throughout. 
While these gains appears modest, this stable final accuracy advantage across network sizes compounds with the results from Sec.~\ref{ssec:taua_exps} and~\ref{ssec:taur_exps}, the latter of which demonstrates \textit{much} faster convergence in settings with the more practically relevant case of $\tau_r > 1$.
Thus, SSD-FL allows network operators to save on communication rounds while achieving higher final accuracies relative to existing methodologies in larger edge/fog settings.

% Overall, trends in both datasets are similar, with SSD-FL consistently outperforming the decentralized FL baselines and returning stable improvements in performance as the size of network grows. 
% Moreover, SSD-FL maintains a similar gap to the baselines for larger networks, aside from the cSTC baseline for which the gap in performance grows with network size. % over larger networks. 
% However, in contrast to other network topologies such as preferential attachment and small world graphs discussed in Sec.~\ref{ssec:diff_net_styles}, the STC baseline performs comparatively well on random graphs and thus SSD-FL has only a modest improvement here. %the network size grows
% % likely due to the uniform connectivity mitigating gradient drift.
% % While SSD-FL continues to hold an advantage, this margin remains modest yet persistent, indicating robustness to topological symmetry.

Among the baselines, cSTC is the only one that also employs clustering, making it a particularly valuable point of comparison. 
While it starts comparably to SSD-FL at $N=10$ in the mild non-i.i.d. scenario of Fig.~\ref{fig:het_ns_fmnist} (both near $74.5\%$), its performance stalls as $N$ grows, falling roughly $7\%$ behind SSD-FL by $N=50$ ($73\%$ vs $80\%$). 
In extreme non-i.i.d. settings, this gap grows, with cSTC trailing SSD-FL by approximately $10\%$ on FMNIST and $4\%$ on CIFAR10 at $N=50$. 
This shows that careless or random cluster formation can actually compound the difficulties of large-scale decentralized FL rather than helping them.  
By contrast, SSD-FL's stable scaling behavior shows that principled cluster formation (per Algorithm~\ref{alg:cluster_form}) offers value in larger and more complex network graphs.
Interestingly, when networks employ homogeneous SGD optimizers at devices, cSTC performs at a comparable level to the STC baseline, with further details provided in Appendix~\ref{app_sec:exps}.

% While the cSTC baseline has a big performance gap, relative to SSD-FL and STC, we want to highlight that this is due to improper clustering. 
% %is primarily due to randomized clustering in networks with heterogeneous ML optimizers at devices. 
% % this may be due to the impact of improper clustering.
% Specifically, as cSTC is a heuristic baseline with random number of clusters (and subsequently follows the STC methodology), we see via inspection with STC that improper or careless cluster formation can yield significant performance penalties. 
% Interestingly, when networks employ homogeneous SGD optimizers at devices, cSTC performs at a comparable level to the STC baseline, but those details are left to Appendix~\ref{app_sec:exps}. 

\begin{figure}[t]
    \centering
    \includegraphics[width=0.96\linewidth]{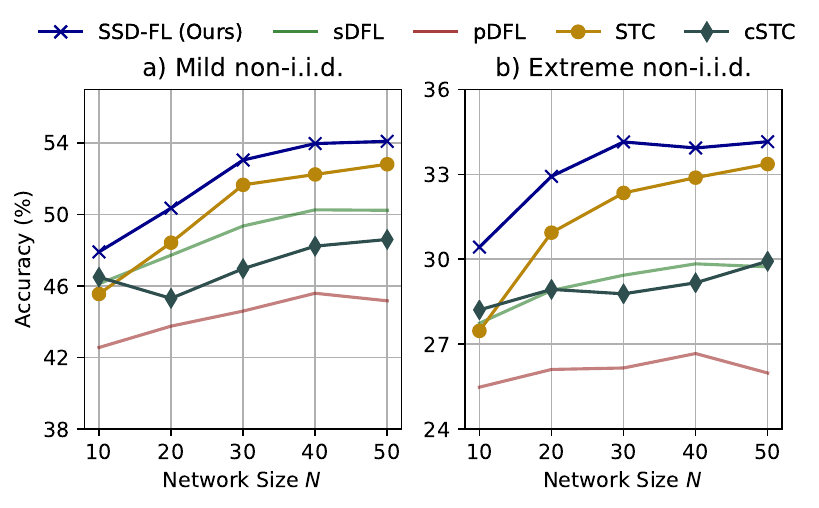}
    \caption{Varying network size from $N=10$ to $50$ with Erdős–Rényi random graphs on CIFAR10. SSD-FL maintains a consistent performance gap across various network sizes.}
    \label{fig:het_ns_cifar}
    \vspace{-4mm}
\end{figure}

%%%% initial/underlying graph network architecture
\subsection{Global network architectures} \label{ssec:diff_net_styles}
% to assess the influence of clustering

We compare SSD-FL with these decentralized FL baselines over multiple global network architectures, each with a unique underlying rule guiding its set of D2D links.
In particular, we evaluate over (i) Erdős–Rényi random graph (RNG)~\cite{gilbert1959random}, in which any two devices $i,j \in \mathcal{N}$ have a fixed probability, $10\%$, to have link between them, (ii) Barabási–Albert preferential attachment (PrefA)~\cite{barabasi1999emergence}, where we set each device to iteratively connect to one other devices with probability proportional to their current degree, (iii) random geometric graph (RGeo)~\cite{penrose2003random}, where devices are placed uniformly at random in a unit-sized Euclidean space and links are established between those within a $0.2$ radius, (iv) Watts-Strogatz small world~\cite{watts1998collective}, for which we choose to have each device with $3$ links to neighboring devices and a $20\%$ chance to reconnect these links randomly, and (iv) complete graphs (Comp)~\cite{erdos1947some}, in which all devices $i \in \mathcal{N}$ are connected. % to each other. 

\begin{figure}[t]
    \centering
    \includegraphics[width=0.98\linewidth]{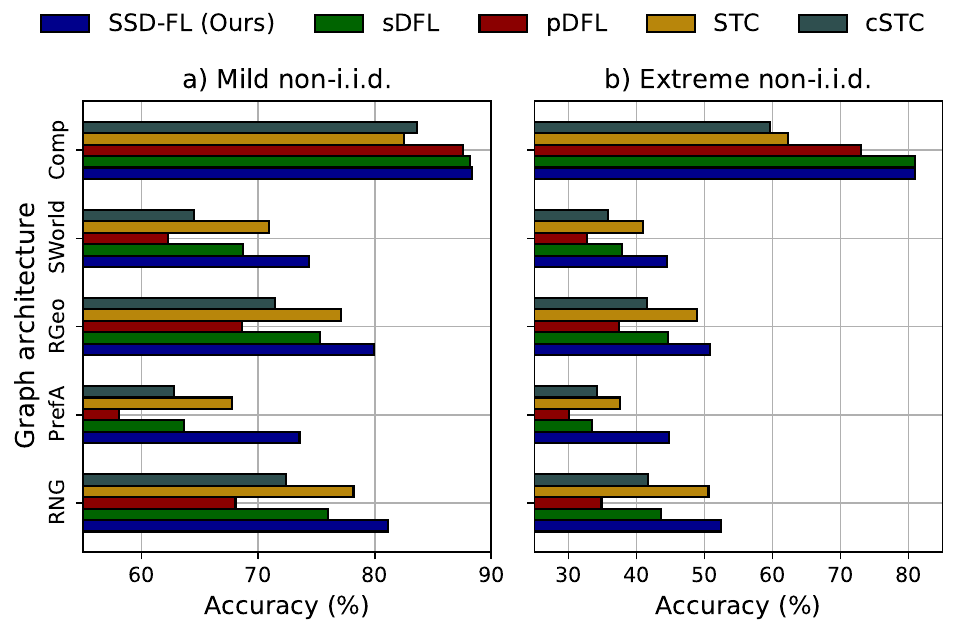}
    \caption{Evaluation of decentralized FL baselines for various network architectures on FMNIST. SSD-FL yields the best performances with the exception of complete networks, for which it identifies a single cluster as optimal, reducing to sDFL.}
    \label{fig:het_ns2_fmnist}
    \vspace{-4mm}
\end{figure}

\begin{figure}[t]
    \centering
    \includegraphics[width=0.98\linewidth]{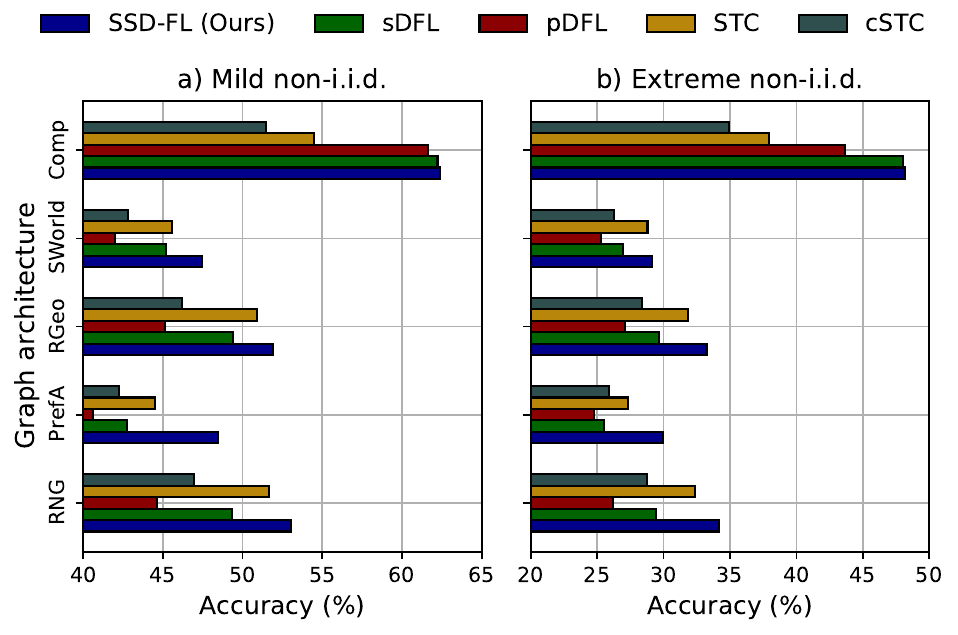}
    \caption{Evaluation of decentralized FL baselines for various network architectures on CIFAR10. Results mirror the FMNIST findings in Fig.~\ref{fig:het_ns2_fmnist}, including the special case of complete networks.}
    \label{fig:het_ns2_cifar}
    \vspace{-4mm}
\end{figure}

Across all non-trivial topologies, SSD-FL consistently outperforms the baselines on both FMNIST in Fig.~\ref{fig:het_ns2_fmnist} and CIFAR10 in Fig.~\ref{fig:het_ns2_cifar}. 
The advantages are most pronounced on preferential attachment and small world graphs, where SSD-FL leads the best performing baseline STC by roughly $7\%$ on FMNIST in mild non-i.i.d. settings ($73.5\%$ vs $67.7\%$ on PrefA), and, similarly, by roughly $7\%$ in extreme non-i.i.d. scenarios ($44.8\%$ vs $37.6\%$ on PrefA). 
Meanwhile, on random graphs, SSD-FL maintains a more modest but consistent advantage of roughly $3\%$ over STC under mild non-i.i.d. ($81.1\%$ vs $78.2\%$), with a larger gap of roughly $2\%$ under extreme non-i.i.d. ($52.5\%$ vs $50.6\%$). 
Since these takeaways on FMNIST are similar to those for CIFAR10 in Fig.~\ref{fig:het_ns2_cifar}, these results collectively suggest that SSD-FL's cluster formation is able to exploit the underlying structure of diverse network topologies, yielding consistent improvements regardless of how the deployment network is formed.

For the special case of complete graphs, SSD-FL and sDFL achieve near identical performance on both datasets and heterogeneity levels ($88.3\%$ vs $88.2\%$ on FMNIST mild non-i.i.d., $80.9\%$ vs $81.0\%$ on FMNIST extreme non-i.i.d.), with SSD-FL forming a single cluster as Algorithm~\ref{alg:cluster_form} correctly identifies that partitioning is unnecessary. 
In sparse networks, clustering trades global connectivity for local density, accelerating intra-cluster convergence enough to justify the reduction in active links. 
In a complete graph however, this trade-off breaks down as the network is already maximally connected, and so partitioning offers no local convergence benefit while incurring consensus delays/costs. 
% while strictly increasing the path length required for global consensus. 
Rather than being a limitation, this result highlights that SSD-FL demonstrates nuance in its cluster formation by clustering only when helpful.

\section{Conclusion}
\label{sec:conclusion}

In this paper, we have introduced SSD-FL, a serverless, semi-decentralized framework for FL, bridging the gap among centralized, semi-decentralized, and decentralized FL. 
To do so, our methodology introduces intra-cluster and inter-cluster regimes, which together form global rounds, and subsequently showed the convergence and consensus properties for such a framework with general clusters. 
Thereafter, we leveraged these theoretical bounds to optimize cluster formation via spectral properties of the network. 
Meanwhile, experiments across various graph topologies as well as different levels of device data and ML optimizer heterogeneity showed that SSD-FL would consistently outperform baseline decentralized FL methodologies. 
Future work can explore time-varying clusters and theoretical extensions for directed topologies, in which asymmetric D2D communications can further complicate convergence, consensus, and overall decision making.

% \newpage 
\bibliographystyle{IEEEtran}
\bibliography{References}

% --------------------------------------------------
% appendicies start here

\newpage
\clearpage
\begingroup
\let\clearpage\relax
\onecolumn
\appendices %% ieee format
% \appendix %% acm format

% --------------------------------------------------
% actual contents begin here

\section*{\centering Appendix}
\section*{\centering Table of Contents}

\noindent
\begin{tabular}{@{}p{0.95\linewidth}r@{}} 

\textbf{Appendix~\ref{app_ssec:smoothness_reg_loss}: Proof of Proposition~\ref{thm:smoothness_reg_loss}}          & \pageref{app_ssec:smoothness_reg_loss} \\[1.6em]
\textbf{Appendix~\ref{app_ssec:re_smooth}: Proof of Corollary~\ref{coro:re_smooth}}    & \pageref{app_ssec:re_smooth} \\[1.6em]
\textbf{Appendix~\ref{app_ssec:intra_conv}: Proof of Theorem~\ref{thm:intra_conv}}        & \pageref{app_ssec:intra_conv} \\[1.6em]
\textbf{Appendix~\ref{app_ssec:intra_cons}: Proof of Lemma~\ref{thm:intra_cons}}        & \pageref{app_ssec:intra_cons} \\[1.6em]
\textbf{Appendix~\ref{app_ssec:inter_cons}: Proof of Lemma~\ref{thm:inter_cons}}        & \pageref{app_ssec:inter_cons} \\[1.6em]
\textbf{Appendix~\ref{app_ssec:comb_conv}: Proof of Theorem~\ref{thm:combined_conv}}        & \pageref{app_ssec:comb_conv} \\[1.6em]

% %%% optim sections
% \textbf{Appendix~\ref{app_sec:st_lf_solution}: Solution to Optimization}  & \pageref{app_sec:st_lf_solution} \\ [0.2em]
% \hspace{2em} \ref{app_ssec:gp} \hspace{1em} Geometric Programming \dotfill              & \pageref{app_ssec:gp} \\ [0.2em]
% \hspace{2em} \ref{app_ssec:tx_2gp} \hspace{1em} Optimization Problem Transformation \dotfill & \pageref{app_ssec:tx_2gp} \\ [0.2em]
% \hspace{2em} \ref{app_ssec:restatement_optim} \hspace{1em} Summary of Resulting Optimization Formulation\dotfill & \pageref{app_ssec:restatement_optim} \\[1.6em]

%%% more experiments sections
\textbf{Appendix~\ref{app_sec:exps}: Additional Experiments}    & \pageref{app_sec:exps} \\ [0.2em]
\hspace{2em} \ref{app_ssec:linkp_exps} \hspace{1em} Varying Link Probabilities \dotfill           & \pageref{app_ssec:linkp_exps} \\ [0.2em]
\hspace{2em} \ref{app_ssec:all_sgds} \hspace{1em} Homogeneous SGD Optimizers \dotfill       & \pageref{app_ssec:all_sgds} \\ [0.2em]
\hspace{2em} \ref{app_ssec:intra_grad_evals} \hspace{1em} Normalized Intra-Cluster Gradients \dotfill  & \pageref{app_ssec:intra_grad_evals} \\ [1.6em]
%%% 

% \label{app_sec:exps} 
% \label{app_ssec:linkp_exps}
% \label{app_ssec:all_sgds}
% \subsection{Normalized Intra-Cluster Gradients} \label{app_ssec:intra_grad_evals}

\end{tabular}

\newpage

\section{Proof of Proposition~\ref{thm:smoothness_reg_loss}}
\label{app_ssec:smoothness_reg_loss} 

\reglsmooth*

\begin{proof}
% Given any cluster $s \in \mathcal{S}$, devices $i \in \mathcal{N}_s$ have locally $\gamma_i$-smooth loss functions (per Assumption~\ref{ass:smoothness},  
Recall that, via the definition of regularized cluster loss functions, we can expand $\Vert \nabla \tilde{L}_s(\hat{\boldsymbol{\theta}}^{\tilde{k},q_1}_s) - \nabla \tilde{L}_s(\hat{\boldsymbol{\theta}}^{\tilde{k},q_2}_s) \Vert$ as follows:
\begin{align}
    & \Vert \nabla \tilde{L}_s(\hat{\boldsymbol{\theta}}^{\tilde{k},q_1}_s) - \nabla \tilde{L}_s(\hat{\boldsymbol{\theta}}^{\tilde{k},q_2}_s) \Vert \\
    & \overset{(a)}{=} \Bigg\Vert  \nabla \left( \sum_{i \in \mathcal{N}_s} L_i(\boldsymbol{\theta}^{\tilde{k},q_1}_i) - L_i(\boldsymbol{\theta}^{\tilde{k},q_2}_i) \right)
    + 
    \nabla \left( \boldsymbol{\theta}^{\tilde{k},q_1}_i \sum_{p=0}^{q_1-1} \rho^{q_1-p}_i \nabla L_i(\boldsymbol{\theta}^{\tilde{k},p}_i) - \boldsymbol{\theta}^{\tilde{k},q_2}_i \sum_{p=0}^{q_2 - 1} \rho^{q_2-p}_i \nabla L_i(\boldsymbol{\theta}^{\tilde{k},p}_i) \right) \\ \nonumber
    & + \nabla \left( \sum_{i \in \mathcal{N}_s} \frac{\mu_i}{2} \left( \Vert \boldsymbol{\theta}^{\tilde{k},q_1}_i - \boldsymbol{\theta}^{\tilde{k},0}_i \Vert^2 
    - \Vert \boldsymbol{\theta}^{\tilde{k},q_2}_i - \boldsymbol{\theta}^{\tilde{k},0}_i \Vert^2 \right) \right)
    + 
    \nabla \frac{1}{2\eta} \left( \Vert \hat{\boldsymbol{\theta}}^{\tilde{k},q_1}_s \Vert^2_{\boldsymbol{I}_s - \boldsymbol{A}_s} - \Vert \hat{\boldsymbol{\theta}}^{\tilde{k},q_2}_s \Vert^2_{\boldsymbol{I}_s - \boldsymbol{A}_s}  \right)
    \Bigg\Vert \\ 
    & \label{eq:prf_thm_smooth_1}
    \overset{(b)}{\leq} 
    \underbrace{ \left\Vert \nabla \mathcal{L}_s(\hat{\boldsymbol{\theta}}^{\tilde{k},q_1}_s) - \nabla \mathcal{L}_s(\hat{\boldsymbol{\theta}}^{\tilde{k},q_2}_s) \right\Vert }_{(i)}
    +
    \underbrace{ \left\Vert \sum_{p = 0}^{q_1 - 1} \boldsymbol{\rho}_s^{q_1-p} \odot \nabla \mathcal{L}_s(\hat{\boldsymbol{\theta}}^{\tilde{k},p}_s) 
    -  
    \sum_{p = 0}^{q_2 - 1} \boldsymbol{\rho}_s^{q_2-p} \odot \nabla \mathcal{L}_s(\hat{\boldsymbol{\theta}}^{\tilde{k},p}_s)
    \right\Vert}_{(ii)} \\ \nonumber 
    & + 
    \underbrace{ \left\Vert 
    \boldsymbol{\mu}_s \odot \left( \hat{\boldsymbol{\theta}}^{\tilde{k},q_1}_s - \hat{\boldsymbol{\theta}}^{\tilde{k},0}_s \right) - \boldsymbol{\mu}_s \odot \left( \hat{\boldsymbol{\theta}}^{\tilde{k},q_2}_s - \hat{\boldsymbol{\theta}}^{\tilde{k},0}_s \right) 
    \right\Vert }_{(iii)}
    +
    \underbrace{ \left\Vert \frac{1}{\eta} \left( \boldsymbol{I}_s - \boldsymbol{A}_s \right) \left( \hat{\boldsymbol{\theta}}^{\tilde{k},q_1}_s - \hat{\boldsymbol{\theta}}^{\tilde{k},q_2}_s \right)
    \right\Vert}_{(iv)},
\end{align}
where $\nabla \mathcal{L}_s(\hat{\boldsymbol{\theta}}^{\tilde{k},q}_s) = [\nabla L_i (\boldsymbol{\theta}^{\tilde{k},q}_i) ]_{i \in \mathcal{N}_s}$, $(a)$ is from the expanded definition of the regularized effective cluster loss functions, and $(b)$ applies the gradient to the scalars and then uses the triangle inequality. 
We next bound each of the four terms $(i), (ii), (iii),$ and $(iv)$ in~\eqref{eq:prf_thm_smooth_1}, starting with term $(i)$ as follows:
\begin{align} 
    & \left\Vert \nabla \mathcal{L}_s(\hat{\boldsymbol{\theta}}^{\tilde{k},q_1}_s) - \nabla \mathcal{L}_s(\hat{\boldsymbol{\theta}}^{\tilde{k},q_2}_s) \right\Vert 
    \overset{(c)}{=} \left( \sum_{i \in \mathcal{N}_s} \left( \nabla L_i(\boldsymbol{\theta}^{\tilde{k},q_1}_i) - \nabla L_i(\boldsymbol{\theta}^{\tilde{k},q_2}_i) \right)^2 \right)^{1/2} \\
    & \overset{(d)}{\leq} \left( \sum_{i \in \mathcal{N}_s} \gamma_i^2 \left( \boldsymbol{\theta}^{\tilde{k},q_1}_i - \boldsymbol{\theta}^{\tilde{k},q_2}_i \right)^2 \right)^{1/2} \\
    & \overset{(e)}{\leq} \left(\max_{i \in \mathcal{N}_s} \gamma_i \right) \left( \sum_{i \in \mathcal{N}_s} \left( \boldsymbol{\theta}^{\tilde{k},q_1}_i - \boldsymbol{\theta}^{\tilde{k},q_2}_i \right)^2 \right)^{1/2}  \\
    & \overset{(f)}{=} \hat{\gamma}_s \left \Vert \hat{\boldsymbol{\theta}}^{\tilde{k},q_1}_s - \hat{\boldsymbol{\theta}}^{\tilde{k},q_2}_s \right \Vert, 
\end{align}
where $(c)$ follows from the definition of the Euclidean norm, $(d)$ leverages the smoothness assumption in Assumption~\ref{ass:smoothness}, $(e)$ extracts the largest smoothness coefficient $\hat{\gamma}_s = \max_{i \in \mathcal{N}_s} \gamma_i$, and $(f)$ re-applies the equivalent form of the Euclidean norm. 
Next, for term $(ii)$ in~\eqref{eq:prf_thm_smooth_1}, we have that
\begin{align}
    &  \left\Vert \sum_{p = 0}^{q_1 - 1} \boldsymbol{\rho}_s^{q_1-p} \odot \nabla \mathcal{L}_s(\hat{\boldsymbol{\theta}}^{\tilde{k},p}_s) 
    -  
    \sum_{p = 0}^{q_2 - 1} \boldsymbol{\rho}_s^{q_2-p} \odot \nabla \mathcal{L}_s(\hat{\boldsymbol{\theta}}^{\tilde{k},p}_s)
    \right\Vert \\
    & \overset{(g)}{=} \left\Vert \sum_{p = q_2}^{q_1 - 1} \boldsymbol{\rho}_s^{q_1-p} \odot \nabla \mathcal{L}_s(\hat{\boldsymbol{\theta}}^{\tilde{k},p}_s) 
    + 
    \sum_{p = 0}^{q_2 - 1} \left( \boldsymbol{\rho}_s^{q_1-p} -  \boldsymbol{\rho}_s^{q_2-p} \right) \odot \nabla \mathcal{L}_s(\hat{\boldsymbol{\theta}}^{\tilde{k},p}_s)
    \right\Vert \\
    & \overset{(h)}{\leq} \sum_{p = q_2}^{q_1 - 1} \left\Vert \boldsymbol{\rho}_s^{q_1-p} \odot \nabla \mathcal{L}_s(\hat{\boldsymbol{\theta}}^{\tilde{k},p}_s)  \right\Vert + \sum_{p = 0}^{q_2 - 1} \left \Vert \left( \boldsymbol{\rho}_s^{q_1-p} - \boldsymbol{\rho}_s^{q_2-p} \right) \odot \nabla \mathcal{L}_s(\hat{\boldsymbol{\theta}}^{\tilde{k},p}_s)
    \right\Vert \\ 
    & \overset{(i)}{\leq} \sum_{p = q_2}^{q_1 - 1} \left( \sum_{i \in \mathcal{N}_s} \left( \rho_i^{q_1 - p}\nabla L_i(\boldsymbol{\theta}^{\tilde{k},q_1}_i) \right)^2 \right)^{1/2} 
    + 
    \sum_{p=0}^{q_2-1} \left( \sum_{i \in \mathcal{N}_s} \left( \left( \rho_i^{q_1 - p} - \rho_i^{q_2-p} \right) \nabla L_i(\boldsymbol{\theta}^{\tilde{k},q_1}_i) \right)^2 \right)^{1/2} \\
    & \overset{(j)}{\leq} \sum_{p=q_2}^{q_1 -1} \left( \sum_{i \in \mathcal{N}_s} \left( \nabla L_i(\boldsymbol{\theta}^{\tilde{k},q_1}_i) \right)^2 \right)^{1/2} 
    + \sum_{p=0}^{q_2-1} \left( \sum_{i \in \mathcal{N}_s} \left( \nabla L_i(\boldsymbol{\theta}^{\tilde{k},q_1}_i) \right)^2 \right)^{1/2} \\
    & \overset{(k)}{\leq} \sum_{p=0}^{\tau_a-1} \left( N_s B^2 \right)^{1/2} = \tau_a B \sqrt{N_s},
\end{align}
where $(g)$ aligns the summations over the iterations, $(h)$ results from the triangle inequality, $(i)$ expands the Euclidean distance, $(j)$ uses the fact that $\rho_i^{q_1-p}$ and $\rho^{q_1-p}_i - \rho^{q_2-p}_i \leq 1$, and $(k)$ uses the fact that $q_1 \leq \tau_a$ and subsequently leverages Assumption~\ref{ass:bounded_grad} to obtain $\Vert \nabla L_i(\boldsymbol{\theta}_i) \Vert < B$. 
Next, we bound the difference of proximal terms (i.e., term $(iii)$ in~\eqref{eq:prf_thm_smooth_1}) as follows:
\begin{align}
    & \left\Vert 
    \boldsymbol{\mu}_s \odot \left( \hat{\boldsymbol{\theta}}^{\tilde{k},q_1}_s - \hat{\boldsymbol{\theta}}^{\tilde{k},0}_s \right) - \boldsymbol{\mu}_s \odot \left( \hat{\boldsymbol{\theta}}^{\tilde{k},q_2}_s - \hat{\boldsymbol{\theta}}^{\tilde{k},0}_s \right) 
    \right\Vert 
    \overset{(l)}{=} \left\Vert \boldsymbol{\mu}_s \odot \left( \hat{\boldsymbol{\theta}}^{\tilde{k},q_1}_s - \hat{\boldsymbol{\theta}}^{\tilde{k},q_2}_s \right)\right\Vert 
    \\
    & \overset{(m)}{\leq} 
    \left( \sum_{i \in \mathcal{N}_s} \left( \mu_i (\boldsymbol{\theta}^{\tilde{k},q_1}_i - \boldsymbol{\theta}^{\tilde{k},q_2}_i) \right)^2 \right)^{1/2}
    \\
    & \overset{(n)}{\leq}
    \left( \sum_{i \in \mathcal{N}_s} \left( (\boldsymbol{\theta}^{\tilde{k},q_1}_i - \boldsymbol{\theta}^{\tilde{k},q_2}_i) \right)^2 \right)^{1/2} \\
    & \overset{(o)}{\leq}
    \left \Vert \hat{\boldsymbol{\theta}}^{\tilde{k},q_1}_s - \hat{\boldsymbol{\theta}}^{\tilde{k},q_2}_s  \right \Vert  
\end{align}
where $(l)$ cancels out $\pm \hat{\boldsymbol{\theta}}^{\tilde{k},0}_s$, $(m)$ expands the definition of Euclidean distance, $(n)$ uses the fact that $\mu_i \leq 1$ so that $\max_i \mu_i \leq 1$, and $(o)$ is the definition of Euclidean distance. 
Finally, for term $(iv)$ in~\eqref{eq:prf_thm_smooth_1}, we have that
\begin{align}
    & \left\Vert \frac{1}{\eta} \left( \boldsymbol{I}_s - \boldsymbol{A}_s \right) \left( \hat{\boldsymbol{\theta}}^{\tilde{k},q_1}_s - \hat{\boldsymbol{\theta}}^{\tilde{k},q_2}_s \right)
    \right\Vert 
    \overset{(q)}{\leq} \frac{1}{\eta} \Vert \boldsymbol{I}_s - \boldsymbol{A}_s\Vert  \left\Vert \hat{\boldsymbol{\theta}}^{\tilde{k},q_1}_s - \hat{\boldsymbol{\theta}}^{\tilde{k},q_2}_s \right\Vert \\
    & \overset{(p)}{\leq} \frac{1}{\eta} \lambda_{\max}(\boldsymbol{I}_s - \boldsymbol{A}_s) \left\Vert \hat{\boldsymbol{\theta}}^{\tilde{k},q_1}_s - \hat{\boldsymbol{\theta}}^{\tilde{k},q_2}_s \right\Vert \\
    & \overset{(r)}{=} \frac{1}{\eta} \left(1 - \lambda_{N_s}(\boldsymbol{A}_s)\right) \left\Vert \hat{\boldsymbol{\theta}}^{\tilde{k},q_1}_s - \hat{\boldsymbol{\theta}}^{\tilde{k},q_2}_s \right\Vert, 
\end{align}
where $(q)$ converts the expression into two terms, $(p)$ uses the fact that the 2-norm of a matrix (spectral norm) is the largest eigenvalue of said matrix, $(r)$ simplifies the expression of $\lambda_{\max}$.

Finally, combining the bounds for terms $(i)$, $(ii)$, $(iii)$, and $(iv)$ in~\eqref{eq:prf_thm_smooth_1} yields
\begin{align} \label{eq:prf_thm_smooth_2}
    & \left\Vert \nabla \tilde{L}_s(\hat{\boldsymbol{\theta}}^{\tilde{k},q_1}_s) - \nabla \tilde{L}_s(\hat{\boldsymbol{\theta}}^{\tilde{k},q_2}_s) \right\Vert 
    \leq 
    \hat{\gamma}_s \left \Vert \hat{\boldsymbol{\theta}}^{\tilde{k},q_1}_s - \hat{\boldsymbol{\theta}}^{\tilde{k},q_2}_s \right \Vert
    + \tau_a B \sqrt{N_s}
    + \left \Vert \hat{\boldsymbol{\theta}}^{\tilde{k},q_1}_s - \hat{\boldsymbol{\theta}}^{\tilde{k},q_2}_s  \right \Vert \nonumber \\
    &
    + \frac{1}{\eta} \left(1 - \lambda_{N_s}(\boldsymbol{A}_s)\right) \left\Vert \hat{\boldsymbol{\theta}}^{\tilde{k},q_1}_s - \hat{\boldsymbol{\theta}}^{\tilde{k},q_2}_s \right\Vert \\
    & 
    = \left(\hat{\gamma}_s + 1 + \frac{1}{\eta} \left( 1- \lambda_{N_s}(\boldsymbol{A}_s) \right) \right)  \left\Vert \hat{\boldsymbol{\theta}}^{\tilde{k},q_1}_s - \hat{\boldsymbol{\theta}}^{\tilde{k},q_2}_s \right\Vert
    + 
    \tau_a B \sqrt{N_s}, 
\end{align}
which completes the proof for the gap between intra-cluster full gradients. 

Leveraging the same logic for the gap between full gradients across the entire network (i.e., at the global level), we can obtain
\begin{align}
    & \left\Vert \nabla \tilde{L}(\boldsymbol{\theta}^{\tilde{k},q_1}) - \nabla \tilde{L}(\boldsymbol{\theta}^{\tilde{k},q_2}) \right\Vert 
    \leq 
    \hat{\gamma} \left \Vert {\boldsymbol{\theta}}^{\tilde{k},q_1} - {\boldsymbol{\theta}}^{\tilde{k},q_2} \right \Vert
    + 
    \tau_a B \sqrt{N} 
    + 
    \left \Vert {\boldsymbol{\theta}}^{\tilde{k},q_1} - {\boldsymbol{\theta}}^{\tilde{k},q_2} \right \Vert \nonumber \\ 
    & 
    + 
    \frac{1}{\eta}\left( 1- \lambda_{N}(\tilde{\boldsymbol{A}}) \right)
    \left \Vert {\boldsymbol{\theta}}^{\tilde{k},q_1} - {\boldsymbol{\theta}}^{\tilde{k},q_2} \right \Vert \\
    & 
    = \left(\hat{\gamma} + 1 + \frac{1}{\eta} \left( 1- \lambda_{N}(\tilde{\boldsymbol{A}}) \right) \right)  \left\Vert {\boldsymbol{\theta}}^{\tilde{k},q_1} - {\boldsymbol{\theta}}^{\tilde{k},q_2} \right\Vert
    + 
    \tau_a B \sqrt{N},
\end{align}
where $\hat{\gamma} = \max_{i \in \mathcal{N}} \gamma_i$. 
\end{proof}

\newpage

\section{Proof of Corollary~\ref{coro:re_smooth}}
\label{app_ssec:re_smooth} 

\resmooth*

\begin{proof}
Define a variable $\tilde{t} \in [0,1]$ so that we can parameterize a line segment from $\hat{\boldsymbol{\theta}}^{\tilde{k},q_2}_s$ to $\hat{\boldsymbol{\theta}}^{\tilde{k},q_1}_s$ as follows:
\begin{equation} \label{eq:prf_coro_resmooth_1}
    \hat{\boldsymbol{\theta}}_s(\tilde{t}) = \hat{\boldsymbol{\theta}}^{\tilde{k},q_2}_s + \tilde{t} \left( \hat{\boldsymbol{\theta}}^{\tilde{k},q_1}_s - \hat{\boldsymbol{\theta}}^{\tilde{k},q_2}_s \right). 
\end{equation}
Exploiting~\eqref{eq:prf_coro_resmooth_1}, we can express intra-cluster regularized loss functions as functions of $\tilde{t}$,
% Analyzing the difference between intra-cluster regularized loss functions thus yields: 
obtaining arithmetic of intra-cluster regularized loss functions as follows:
\begin{equation} \label{eq:prf_coro_resmooth_2}
    \tilde{L}_s(\hat{\boldsymbol{\theta}}^{\tilde{k},q_1}_s) - \tilde{L}_s(\hat{\boldsymbol{\theta}}^{\tilde{k},q_2}_s) \equiv \tilde{L}_s(\hat{\boldsymbol{\theta}}_s(1)) - \tilde{L}_s(\hat{\boldsymbol{\theta}}_s(0)). 
\end{equation}
Using the fundamental theorem of calculus, we further convert the right hand side of~\eqref{eq:prf_coro_resmooth_2} as follows:
\begin{align}
    & \tilde{L}_s(\hat{\boldsymbol{\theta}}_s(1)) - \tilde{L}_s(\hat{\boldsymbol{\theta}}_s(0)) \\
    & \overset{(a)}{=} \int^1_0 \frac{d}{d \tilde{t}} \tilde{L}_s(\hat{\boldsymbol{\theta}}_s(\tilde{t})) d\tilde{t} \\
    & \label{eq:prf_coro_resmooth_3}
    \overset{(b)}{=} \int^1_0 \nabla \tilde{L}_s(\hat{\boldsymbol{\theta}}_s(\tilde{t}))^T \left( \hat{\boldsymbol{\theta}}_s(1) - \hat{\boldsymbol{\theta}}_s(0)
    \right) d\tilde{t} 
    \equiv \int^1_0 \nabla \tilde{L}_s(\hat{\boldsymbol{\theta}}_s(\tilde{t}))^T \left( \hat{\boldsymbol{\theta}}^{\tilde{k},q_1}_s - \hat{\boldsymbol{\theta}}^{\tilde{k},q_2}_s
    \right) d\tilde{t}
\end{align}
where $(a)$ results from the fundamental theorem of calculus, and $(b)$ follows from the chain rule applied onto $\tilde{L}_s(\hat{\boldsymbol{\theta}}_s(\tilde{t}))$ and subsequently~\eqref{eq:prf_coro_resmooth_1}. 
Combining~\eqref{eq:prf_coro_resmooth_2} and~\eqref{eq:prf_coro_resmooth_3} yields
\begin{align}
    & 
    \tilde{L}_s(\hat{\boldsymbol{\theta}}^{\tilde{k},q_1}_s) 
    \overset{(c)}{=} \tilde{L}_s(\hat{\boldsymbol{\theta}}^{\tilde{k},q_2}_s) 
    + 
    \int^1_0 \nabla \tilde{L}_s(\hat{\boldsymbol{\theta}}_s(\tilde{t}))^T \left( \hat{\boldsymbol{\theta}}_s(1) - \hat{\boldsymbol{\theta}}_s(0)
    \right) d\tilde{t}  \\
    & \label{eq:prf_coro_resmooth_4}
    \overset{(d)}{=}  \tilde{L}_s(\hat{\boldsymbol{\theta}}^{\tilde{k},q_2}_s) 
    + \nabla \tilde{L}_s(\hat{\boldsymbol{\theta}}^{\tilde{k},q_2}_s)^T \left( \hat{\boldsymbol{\theta}}^{\tilde{k},q_1}_s - \hat{\boldsymbol{\theta}}^{\tilde{k},q_2}_s \right) 
    + 
    \int^1_0 \left(\nabla \tilde{L}_s(\hat{\boldsymbol{\theta}}_s(\tilde{t})) - \nabla \tilde{L}_s(\hat{\boldsymbol{\theta}}^{\tilde{k},q_2}_s) \right)^T  
    \left( \hat{\boldsymbol{\theta}}^{\tilde{k},q_1}_s - \hat{\boldsymbol{\theta}}^{\tilde{k},q_2}_s
    \right) d\tilde{t},
\end{align}
where $(c)$ re-arranges the combination of~\eqref{eq:prf_coro_resmooth_2} and~\eqref{eq:prf_coro_resmooth_3}, and $(d)$ introduces $\pm \nabla \tilde{L}_s(\hat{\boldsymbol{\theta}}^{\tilde{k},q_2}_s) \left( \hat{\boldsymbol{\theta}}^{\tilde{k},q_1}_s - \hat{\boldsymbol{\theta}}^{\tilde{k},q_2}_s \right)$. 
Next, we focus on bounding the integral in~\eqref{eq:prf_coro_resmooth_4} as follows:
\begin{align}
    & \int^1_0 \left(\nabla \tilde{L}_s(\hat{\boldsymbol{\theta}}_s(\tilde{t})) - \nabla \tilde{L}_s(\hat{\boldsymbol{\theta}}^{\tilde{k},q_2}_s) \right)^T 
    \left( \hat{\boldsymbol{\theta}}^{\tilde{k},q_1}_s - \hat{\boldsymbol{\theta}}^{\tilde{k},q_2}_s
    \right) d\tilde{t} 
    \overset{(e)}{\leq} 
    \int^1_0 \left\Vert \nabla \tilde{L}_s(\hat{\boldsymbol{\theta}}_s(\tilde{t})) - \nabla \tilde{L}_s(\hat{\boldsymbol{\theta}}^{\tilde{k},q_2}_s)  \right\Vert 
    \left\Vert  \hat{\boldsymbol{\theta}}^{\tilde{k},q_1}_s - \hat{\boldsymbol{\theta}}^{\tilde{k},q_2}_s
    \right\Vert 
    d\tilde{t}  \\
    & \overset{(f)}{\leq} 
    \int_0^1 
    \left[\left( \hat{\gamma}_s + 1 + \frac{1}{\eta} \left( 1 - \lambda_{N_s}(\boldsymbol{A}_s) \right)\right) 
    \left\Vert \hat{\boldsymbol{\theta}}_s(\tilde{t}) - \hat{\boldsymbol{\theta}}^{\tilde{k},q_2}_s \right\Vert + \tau_a B \sqrt{N_s} \right] 
    \left\Vert  \hat{\boldsymbol{\theta}}^{\tilde{k},q_1}_s - \hat{\boldsymbol{\theta}}^{\tilde{k},q_2}_s
    \right\Vert 
    d\tilde{t}
    \\
    & \overset{(g)}{=} 
    \int_0^1 
    \left[\left( \hat{\gamma}_s + 1 + \frac{1}{\eta} \left( 1 - \lambda_{N_s}(\boldsymbol{A}_s) \right)\right) 
    \left\Vert \tilde{t} \left( \hat{\boldsymbol{\theta}}^{\tilde{k},q_1}_s - \hat{\boldsymbol{\theta}}^{\tilde{k},q_2}_s \right) \right\Vert \right] 
    \left\Vert  \hat{\boldsymbol{\theta}}^{\tilde{k},q_1}_s - \hat{\boldsymbol{\theta}}^{\tilde{k},q_2}_s
    \right\Vert     
    d \tilde{t}
    + \tau_a B \sqrt{N_s} \left\Vert  \hat{\boldsymbol{\theta}}^{\tilde{k},q_1}_s - \hat{\boldsymbol{\theta}}^{\tilde{k},q_2}_s
    \right\Vert \\
    &  \label{eq:prf_coro_resmooth_5}
    \overset{(h)}{=} 
    \frac{1}{2} \left( \hat{\gamma}_s + 1 + \frac{1}{\eta} \left( 1 - \lambda_{N_s}(\boldsymbol{A}_s) \right)\right) 
    \left\Vert  \hat{\boldsymbol{\theta}}^{\tilde{k},q_1}_s - \hat{\boldsymbol{\theta}}^{\tilde{k},q_2}_s
    \right\Vert^2
    + 
    \tau_a B \sqrt{N_s} \left\Vert  \hat{\boldsymbol{\theta}}^{\tilde{k},q_1}_s - \hat{\boldsymbol{\theta}}^{\tilde{k},q_2}_s
    \right\Vert 
\end{align}
where $(e)$ is from the Cauchy-Schwarz inequality, $(f)$ uses the regularized cluster loss gradient gap derived in Theorem~\ref{thm:smoothness_reg_loss}, $(g)$ substitutes the definition of $\hat{\boldsymbol{\theta}}_s(\tilde{t})$ from~\eqref{eq:prf_coro_resmooth_1}, and $(h)$ expands the integral. 
Finally, combining~\eqref{eq:prf_coro_resmooth_4} and~\eqref{eq:prf_coro_resmooth_5} yields the result as follows:
\begin{equation}
\begin{aligned}
    &
    \tilde{L}_s(\hat{\boldsymbol{\theta}}^{\tilde{k},q_1}_s) 
    \leq \tilde{L}_s(\hat{\boldsymbol{\theta}}^{\tilde{k},q_2}_s) 
    + \nabla \tilde{L}_s(\hat{\boldsymbol{\theta}}^{\tilde{k},q_2}_s)^T \left( \hat{\boldsymbol{\theta}}^{\tilde{k},q_1}_s - \hat{\boldsymbol{\theta}}^{\tilde{k},q_2}_s \right) \\
    &
    + \frac{1}{2} \left( \hat{\gamma}_s + 1 + \frac{1}{\eta} \left( 1 - \lambda_{N_s}(\boldsymbol{A}_s) \right)\right) 
    \left\Vert  \hat{\boldsymbol{\theta}}^{\tilde{k},q_1}_s - \hat{\boldsymbol{\theta}}^{\tilde{k},q_2}_s
    \right\Vert^2
    + 
    \tau_a B \sqrt{N_s} \left\Vert  \hat{\boldsymbol{\theta}}^{\tilde{k},q_1}_s - \hat{\boldsymbol{\theta}}^{\tilde{k},q_2}_s
    \right\Vert. 
\end{aligned}
\end{equation}

\end{proof}
\newpage

\section{Proof of Theorem~\ref{thm:intra_conv}}
\label{app_ssec:intra_conv} 

\intraConv*

\begin{proof}
Leveraging the result of Corollary~\ref{coro:re_smooth} and combining with the intra-cluster ML model update rule from~\eqref{eq:expand_cluster_gd} yields
% We start by recalling the intra-cluster ML model parameter update rule from~\eqref{eq:expand_cluster_gd}
\begin{align}
    &
    \tilde{L}_s(\hat{\boldsymbol{\theta}}^{\tilde{k},q+1}_s) 
    \leq \tilde{L}_s(\hat{\boldsymbol{\theta}}^{\tilde{k},q}_s) 
    + \nabla \tilde{L}_s(\hat{\boldsymbol{\theta}}^{\tilde{k},q}_s)^T \left( -\eta \nabla \tilde{F}_s(\hat{\boldsymbol{\theta}}^{\tilde{k},q}_s) \right) \\ \nonumber
    &
    + \frac{1}{2} \left( \hat{\gamma}_s + 1 + \frac{1}{\eta} \left( 1 - \lambda_{N_s}(\boldsymbol{A}_s) \right)\right) 
    \left\Vert  -\eta \nabla \tilde{F}_s(\hat{\boldsymbol{\theta}}^{\tilde{k},q}_s) 
    \right\Vert^2
    + 
    \tau_a B \sqrt{N_s} \left\Vert  -\eta \nabla \tilde{F}_s(\hat{\boldsymbol{\theta}}^{\tilde{k},q}_s) 
    \right\Vert \\
    & \label{eq:prf_intraconv_1}
    \overset{(a)}{\leq} 
    \tilde{L}_s(\hat{\boldsymbol{\theta}}^{\tilde{k},q}_s) 
    + \nabla \tilde{L}_s(\hat{\boldsymbol{\theta}}^{\tilde{k},q}_s)^T \left( -\eta \nabla \tilde{F}_s(\hat{\boldsymbol{\theta}}^{\tilde{k},q}_s) \right) 
    % \\ \nonumber
    + \frac{1}{2} \left( \hat{\gamma}_s + 1 + \frac{1}{\eta} \left( 1 - \lambda_{N_s}(\boldsymbol{A}_s) \right) + 
    \tau_a B \sqrt{N_s} \right) 
    \left\Vert  -\eta \nabla \tilde{F}_s(\hat{\boldsymbol{\theta}}^{\tilde{k},q}_s) 
    \right\Vert^2,
\end{align}
where $(a)$ follows immediately since $\Vert -\eta \nabla \tilde{F}_s(\hat{\boldsymbol{\theta}}^{\tilde{k},q}_s) \Vert^2 > \Vert - \eta \nabla \tilde{F}_s(\hat{\boldsymbol{\theta}}^{\tilde{k},q}_s) \Vert$. 
Re-arranging~\eqref{eq:prf_intraconv_1} and taking the expectation yields
\begin{align}
    & \tilde{L}_s(\hat{\boldsymbol{\theta}}^{\tilde{k},q+1}_s)  - \tilde{L}_s(\hat{\boldsymbol{\theta}}^{\tilde{k},q}_s) 
    \overset{(b)}{\leq}
    - \eta \nabla \tilde{L}_s(\hat{\boldsymbol{\theta}}^{\tilde{k},q}_s)^T \mathbb{E} \left[ \nabla \tilde{F}_s(\hat{\boldsymbol{\theta}}^{\tilde{k},q}_s) \right] \nonumber \\
    & 
    + 
    \frac{\eta^2}{2} \left( \hat{\gamma}_s + 1 + \frac{1}{\eta} \left( 1 - \lambda_{N_s}(\boldsymbol{A}_s) \right) + 
    \tau_a B \sqrt{N_s} \right) 
    \mathbb{E} \left[ \left\Vert \nabla \tilde{F}_s(\hat{\boldsymbol{\theta}}^{\tilde{k},q}_s) 
    \right\Vert^2 \right] 
    \\
    & \overset{(c)}{\leq}
    - \eta \left\Vert \nabla
    \tilde{L}_s(\hat{\boldsymbol{\theta}}^{\tilde{k},q}_s)
    \right\Vert^2 
    + \frac{\eta^2}{2} 
    \left( \hat{\gamma}_s + 1 + \frac{1}{\eta} \left( 1 - \lambda_{N_s}(\boldsymbol{A}_s) \right) + 
    \tau_a B \sqrt{N_s} \right) 
    \left( \alpha + \hat{\alpha}_s \left\Vert 
    \nabla \tilde{L}_s(\hat{\boldsymbol{\theta}}^{\tilde{k},q}_s)
    \right\Vert^2 \right) \\
    & 
    \label{eq:prf_intraconv_2}
    \overset{(d)}{=}
    \left( -\eta  +\frac{\hat{\alpha}_s \eta^2 }{2} \left( \hat{\gamma}_s + 1 + \frac{1}{\eta} \left( 1 - \lambda_{N_s}(\boldsymbol{A}_s) \right) + 
    \tau_a B \sqrt{N_s} \right) \right) 
    \left\Vert 
    \nabla \tilde{L}_s(\hat{\boldsymbol{\theta}}^{\tilde{k},q}_s)
    \right\Vert^2  \nonumber \\
    & 
    + \frac{\alpha \eta^2}{2} \left( \hat{\gamma}_s + 1 + \frac{1}{\eta} \left( 1 - \lambda_{N_s}(\boldsymbol{A}_s) \right) + 
    \tau_a B \sqrt{N_s} \right), 
\end{align}
where $(b)$ is the result of re-arrangement, $(c)$ leverages Assumption~\ref{ass:grad_var} and the fact that $\nabla \tilde{F}_s(\hat{\boldsymbol{\theta}}^{\tilde{k},q}_s)$ is the unbiased estimate of $\nabla \tilde{L}_s(\hat{\boldsymbol{\theta}}^{\tilde{k},q_2}_s)$, and $(d)$ simplifies the algebra. 
Further re-arrangement of~\eqref{eq:prf_intraconv_2} yields
\begin{equation}
\begin{aligned} \label{eq:prf_intraconv_3}
    & \left( \eta  - \frac{\hat{\alpha}_s \eta^2 }{2} \left( \hat{\gamma}_s + 1 + \frac{1}{\eta} \left( 1 - \lambda_{N_s}(\boldsymbol{A}_s) \right) + 
    \tau_a B \sqrt{N_s} \right) \right) 
    \left\Vert 
    \nabla \tilde{L}_s(\hat{\boldsymbol{\theta}}^{\tilde{k},q}_s)
    \right\Vert^2
    \leq 
    \tilde{L}_s(\hat{\boldsymbol{\theta}}^{\tilde{k},q}_s)  - \tilde{L}_s(\hat{\boldsymbol{\theta}}^{\tilde{k},q+1}_s)  \\
    & + 
    \frac{\alpha \eta^2}{2} \left( \hat{\gamma}_s + 1 + \frac{1}{\eta} \left( 1 - \lambda_{N_s}(\boldsymbol{A}_s) \right) + 
    \tau_a B \sqrt{N_s} \right).
\end{aligned}
\end{equation}
Finally, dividing both sides of~\eqref{eq:prf_intraconv_3} by the coefficient on $\left\Vert \nabla \tilde{L}_s(\hat{\boldsymbol{\theta}}^{\tilde{k},q}_s) \right\Vert^2$ and summing over all instances $q \in \tilde{k}$ yields the result as follows:
\begin{equation}
\begin{aligned}
    \label{eq:prf_intraconv_4}
    & \sum_{q=0}^{\tau_a-1} \left\Vert 
    \nabla \tilde{L}_s(\hat{\boldsymbol{\theta}}^{\tilde{k},q}_s)
    \right\Vert^2 
    \overset{(e)}{\leq}
    \frac{\tilde{L}_s(\hat{\boldsymbol{\theta}}^{\tilde{k},0}_s)  - \tilde{L}_s(\hat{\boldsymbol{\theta}}^{\tilde{k},\tau_a}_s)}
    {\left( \eta  - \frac{\hat{\alpha}_s \eta^2 }{2} \left( \hat{\gamma}_s + 1 + \frac{1}{\eta} \left( 1 - \lambda_{N_s}(\boldsymbol{A}_s) \right) + 
    \tau_a B \sqrt{N_s} \right) \right) }\\ 
    & + \frac{\frac{\alpha \tau_a \eta^2}{2} \left( \hat{\gamma}_s + 1 + \frac{1}{\eta} \left( 1 - \lambda_{N_s}(\boldsymbol{A}_s) \right) + 
    \tau_a B \sqrt{N_s} \right)}
    {\left( \eta  - \frac{\hat{\alpha}_s \eta^2 }{2} \left( \hat{\gamma}_s + 1 + \frac{1}{\eta} \left( 1 - \lambda_{N_s}(\boldsymbol{A}_s) \right) + 
    \tau_a B \sqrt{N_s} \right) \right) } \\
    & \overset{(f)}{\leq} \frac{
    \tilde{L}_s(\hat{\boldsymbol{\theta}}^{\tilde{k},0}_s) + 
    \frac{\alpha \tau_a \eta^2}{2} \left( \hat{\gamma}_s + 1 + \frac{1}{\eta} \left( 1 - \lambda_{N_s}(\boldsymbol{A}_s) \right) + 
    \tau_a B \sqrt{N_s} \right)}
    {\left( \eta  - \frac{\hat{\alpha}_s \eta^2 }{2} \left( \hat{\gamma}_s + 1 + \frac{1}{\eta} \left( 1 - \lambda_{N_s}(\boldsymbol{A}_s) \right) + 
    \tau_a B \sqrt{N_s} \right) \right) }, 
\end{aligned}
\end{equation}
where $(f)$ is from the fact that $\tilde{L}_s(\cdot) \geq 0$. 
Finally, note that the $(e)$ step in~\eqref{eq:prf_intraconv_4} requires that 
\begin{align} \label{eq:prf_intraconv_5}
    & \eta > \frac{\hat{\alpha}_s \eta^2 }{2} \left( \hat{\gamma}_s + 1 + \frac{1}{\eta} \left( 1 - \lambda_{N_s}(\boldsymbol{A}_s) + \tau_a B \sqrt{N_s} \right) \right)
    % & 1 > \frac{\hat{\alpha}_s \eta }{2} \left( \hat{\gamma}_s + 1 + \frac{1}{\eta} \left( 1 - \lambda_{N_s}(\boldsymbol{A}_s) \right) \right) \\
    \rightarrow \eta < \frac{2}{\hat{\alpha}_s \left( \hat{\gamma}_s + 1 + \frac{1}{\eta} \left( 1 - \lambda_{N_s}(\boldsymbol{A}_s) \right) + \tau_a B \sqrt{N_s} \right) },
\end{align}
after re-arranging. 

% which, after simplification, yields
% \begin{equation}
%     \eta < \frac{2}{\hat{\alpha}_s \left( \hat{\gamma}_s + 1 + \frac{1}{\eta} \left( 1 - \lambda_{N_s}(\boldsymbol{A}_s) \right) + \tau_a B \sqrt{N_s} \right) }.
% \end{equation} 

%% a simple extension (sum over all clusters)...
% At the global level (i.e., across all clusters $s \in \mathcal{S})$,~\eqref{eq:prf_intraconv_5} 
% \begin{equation}
%     \sum_{s \in \mathcal{S}} \sum_{q=0}^{\tau_a-1} \left\Vert 
%     \nabla \tilde{L}_s(\hat{\boldsymbol{\theta}}^{\tilde{k},q}_s)
%     \right\Vert^2 
%     \leq 
%     \sum_{s \in \mathcal{S}} \left[ \frac{
%     \tilde{L}_s(\hat{\boldsymbol{\theta}}^{\tilde{k},0}_s) + 
%     \frac{\alpha \tau_a \eta^2}{2} \left( \hat{\gamma}_s + 1 + \frac{1}{\eta} \left( 1 - \lambda_{N_s}(\boldsymbol{A}_s) \right) + 
%     \tau_a B \sqrt{N_s} \right)}
%     {\left( \eta  - \frac{\hat{\alpha}_s \eta^2 }{2} \left( \hat{\gamma}_s + 1 + \frac{1}{\eta} \left( 1 - \lambda_{N_s}(\boldsymbol{A}_s) \right) + 
%     \tau_a B \sqrt{N_s} \right) \right) }
%     \right].
% \end{equation}

\end{proof}

\newpage

\section{Proof of Lemma~\ref{thm:intra_cons}}
\label{app_ssec:intra_cons} 

\intraCons*

\begin{proof}
Combining the intra-cluster ML model parameters update rule in~\eqref{eq:expand_cluster_gd} and the full form of $\tilde{F}_s(\cdot)$ via~\eqref{eq:tildeg_cluster_wide} and~\eqref{eq:cluster_2gether_update} yields
\begin{align} \label{eq:prf_intracons_1}
    \hat{\boldsymbol{\theta}}^{\tilde{k},q+1}_s = \boldsymbol{A}_s \hat{\boldsymbol{\theta}}^{\tilde{k},q}_s
    - 
    \eta \left( \boldsymbol{G}_s(\hat{\boldsymbol{\theta}}^{\tilde{k},q}_s) + 
    \sum_{p=0}^{q-1} \boldsymbol{\rho}_s^{q-p} \odot \boldsymbol{G}_s(\hat{\boldsymbol{\theta}}^{\tilde{k},p}_s) + \boldsymbol{\mu}_s \odot \left( \hat{\boldsymbol{\theta}}^{\tilde{k},q}_s - \hat{\boldsymbol{\theta}}^{\tilde{k},0}_s \right) \right).
\end{align}
To analyze $\Delta^{\tilde{k},q+1}_s$, we first express it in an equivalent form 
\begin{align} \label{eq:prf_intracons_2}
    \Delta^{\tilde{k},q+1}_{s} = \overline{\boldsymbol{\theta}}^{\tilde{k},q+1}_s \boldsymbol{1}_s - \hat{\boldsymbol{\theta}}^{\tilde{k},q+1}_s 
    = \boldsymbol{P}_s \hat{\boldsymbol{\theta}}^{\tilde{k},q+1}_s, 
\end{align}
where $\boldsymbol{P}_s = \frac{1}{N_s} \boldsymbol{1}_s \boldsymbol{1}_s^T - \boldsymbol{I}_s$. 
Combining~\eqref{eq:prf_intracons_1} and~\eqref{eq:prf_intracons_2} then applying the triangle inequality enables the following expansion of $\Vert \Delta^{\tilde{k},q+1}_s \Vert$
\begin{align} \label{eq:prf_intracons_3}
    & \left \Vert \Delta^{\tilde{k},q+1}_s \right\Vert 
    \leq 
    \underbrace{ \left\Vert \boldsymbol{P}_s \left( \boldsymbol{A}_s \hat{\boldsymbol{\theta}}^{\tilde{k},q}_s \right) \right\Vert}_{(i)} 
    +
    \underbrace{ \left\Vert \boldsymbol{P}_s \left(\eta \boldsymbol{G}_s(\hat{\boldsymbol{\theta}}^{\tilde{k},q}_s) \right) \right\Vert}_{(ii)} \nonumber \\
    & 
    + 
    \underbrace{ \left\Vert \boldsymbol{P}_s \left(\eta  \sum_{p=0}^{q-1} \boldsymbol{\rho}_s^{q-p} \odot \boldsymbol{G}_s(\hat{\boldsymbol{\theta}}^{\tilde{k},p}_s) \right) \right\Vert}_{(iii)}
    +
    \underbrace{ \left\Vert \boldsymbol{P}_s \left(\eta  \boldsymbol{\mu}_s \odot \left( \hat{\boldsymbol{\theta}}^{\tilde{k},q}_s - \hat{\boldsymbol{\theta}}^{\tilde{k},0}_s \right) \right) \right\Vert}_{(iv)}. 
\end{align}
As $\boldsymbol{A}_s$ is doubly stochastic per Assumption~\ref{ass:adjacency_matrix}, we exploit commutativity of the constituents of term $(i)$ in~\eqref{eq:prf_intracons_3} as follows:
\begin{align} 
    & \left\Vert \boldsymbol{P}_s \left( \boldsymbol{A}_s \hat{\boldsymbol{\theta}}^{\tilde{k},q}_s \right) \right\Vert = \left\Vert \boldsymbol{A}_s \left( \boldsymbol{P}_s \hat{\boldsymbol{\theta}}^{\tilde{k},q}_s \right) \right\Vert \\
    & \overset{(a)}{\leq} \left\Vert \boldsymbol{A}_s \Delta^{\tilde{k},q}_s  \right\Vert \\
    & 
    \label{eq:prf_intracons_4}
    \overset{(b)}{\leq} \lambda_2(\boldsymbol{A}_s) \left \Vert \Delta^{\tilde{k},q}_s \right\Vert, 
\end{align}
where $(a)$ uses the definition of $\Delta^{\tilde{k},q}_s$, and $(b)$ bounds the spectral norm of $\boldsymbol{A}_s$ by its largest feasible eigenvalue, assuming $\Delta^{\tilde{k},q}_s \perp 1$. 
Next, for term $(ii)$, we bound as follows:
\begin{align} 
    & \left\Vert \boldsymbol{P}_s \left(\eta \boldsymbol{G}_s(\hat{\boldsymbol{\theta}}^{\tilde{k},q}_s) \right) \right\Vert \\
    & \overset{(c)}{=} \eta \left( \sum_{j \in \mathcal{N}_s} \left( \sum_{i \in \mathcal{N}_s} \frac{1}{N_s} g_i(\boldsymbol{\theta}^{\tilde{k},q}_i) - g_j(\boldsymbol{\theta}^{\tilde{k},q}_j)\right)^2 \right)^{1/2} \\
    & \overset{(d)}{\leq} \eta \left( \sum_{j \in \mathcal{N}_s} \left( \left\Vert \sum_{i \in \mathcal{N}_s} \frac{1}{N_s} g_i(\boldsymbol{\theta}^{\tilde{k},q}_i) \right\Vert + \left\Vert g_j(\boldsymbol{\theta}^{\tilde{k},q}_j) \right\Vert \right)^2 \right)^{1/2}  \\
    & \overset{(e)}{\leq} \eta \left( \sum_{j \in \mathcal{N}_s}(2B)^2 \right)^{1/2} \\
    & 
    \label{eq:prf_intracons_5}
    \overset{(f)}{=} 2 \eta B \sqrt{N_s},
\end{align}
where $(c)$ uses the definition of the Euclidean distance, $(d)$ follows from a triangle inequality, $(e)$ relies on triangle inequality and Assumption~\ref{ass:bounded_grad}, and $(f)$ simplifies the result of $(e)$. 
Similarly, for term $(iii)$, we have that
\begin{align} 
    & \left\Vert \boldsymbol{P}_s \left(\eta  \sum_{p=0}^{q-1} \boldsymbol{\rho}_s^{q-p} \odot \boldsymbol{G}_s(\hat{\boldsymbol{\theta}}^{\tilde{k},p}_s) \right) \right\Vert \\
    & \overset{(g)}{\leq} \eta q \left\Vert \boldsymbol{P}_s \boldsymbol{G}_s(\hat{\boldsymbol{\theta}}^{\tilde{k},p}_s) \right\Vert\\
    & 
    \label{eq:prf_intracons_6}
    \overset{(h)}{\leq} 2 \eta (\tau_a - 1) B \sqrt{N_s},
\end{align}
where $(g)$ follows from triangle inequalities, $\rho_i < 1$, and the properties of the Hadamard product, 
%% \rho_i \leq 1 and \sum_p^q-1 \rho_i^{q-p} \leq \sum_p^{q-1} 1 \leq q, which you then extract
and $(h)$ is from similar steps as that of $(c)-(f)$ above in~\eqref{eq:prf_intracons_5} and $q \leq \tau_a - 1$. 
Finally, for term $(iv)$ in~\eqref{eq:prf_intracons_3}, we bound via the following:
\begin{align} 
    & \left\Vert \boldsymbol{P}_s \left(\eta  \boldsymbol{\mu}_s \odot \left( \hat{\boldsymbol{\theta}}^{\tilde{k},q}_s - \hat{\boldsymbol{\theta}}^{\tilde{k},0}_s \right) \right) \right\Vert \\
    & \overset{(i)}{\leq} \eta  \left\Vert \boldsymbol{P}_s \left( \hat{\boldsymbol{\theta}}^{\tilde{k},q}_s - \hat{\boldsymbol{\theta}}^{\tilde{k},0}_s \right) \right\Vert \\
    & \overset{(j)}{=} \eta  \left\Vert \Delta^{\tilde{k},q}_s - \Delta^{\tilde{k},0}_s \right\Vert \\
    % & \overset{(k)}{\leq} \eta \left( \left\Vert \Delta^{\tilde{k},q}_s \right\Vert + \left\Vert \Delta^{\tilde{k},0}_s \right\Vert\right)
    & 
    \label{eq:prf_intracons_7}
    \overset{(k)}{\leq} \eta \left\Vert \Delta^{\tilde{k},q}_s \right\Vert,   
\end{align}
where $(i)$ is from $\mu_i \leq 1$, $(j)$ leverages the definition of $\Delta^{\tilde{k},q}_s$, and $(k)$ exploits the fact that $\Delta^{\tilde{k},0}_s = 0$. 
Finally, combining~\eqref{eq:prf_intracons_4}-\eqref{eq:prf_intracons_7} into~\eqref{eq:prf_intracons_3} yields
\begin{align} \label{eq:prf_intracons_8}
    &\left \Vert \Delta^{\tilde{k},q+1}_s \right\Vert 
    \leq 
    (\lambda_2(\boldsymbol{A}_s)+\eta) \left \Vert \Delta^{\tilde{k},q}_s \right\Vert
    +
    2 \eta \tau_a B \sqrt{N_s}.
\end{align}
Expanding~\eqref{eq:prf_intracons_8} recursively yields: 
\begin{align}
    & \left \Vert \Delta^{\tilde{k},q+1}_s \right\Vert 
    \overset{(l)}{\leq}
    (\lambda_2(\boldsymbol{A}_s)+\eta)^{q} \left \Vert \Delta^{\tilde{k},0}_s \right\Vert
    +
    2 \eta \tau_a B \sqrt{N_s} \sum_{p=0}^{q} (\lambda_2(\boldsymbol{A}_s)+\eta)^p
    \\
    % & \overset{(m)}{=} 2 \eta \tau_a B \sqrt{N_s} \sum_{p=0}^{q} (\lambda_2(\boldsymbol{A}_s)+\eta)^p \\
    % & \overset{(n)}{\leq} \frac{2 \eta \tau_a B \sqrt{N_s}}{1 - \eta - \lambda_2(\boldsymbol{A}_s)}, 
    & \overset{(m)}{=} (\lambda_2(\boldsymbol{A}_s)+\eta)^{q} \left \Vert \Delta^{\tilde{k},0}_s \right\Vert
    + \frac{2 \eta \tau_a B \sqrt{N_s}}{1 - \eta - \lambda_2(\boldsymbol{A}_s)}, 
    \label{eq:prf_intracons_9} 
\end{align}
where $(l)$ expands the recursion in~\eqref{eq:prf_intracons_7}, and $(m)$ bounds the finite geometric sum by the infinite geometric sum and requires that $\eta < 1 - \lambda_2(\boldsymbol{A}_s)$. 
%$(m)$ exploits the fact that $\Delta^{\tilde{k},0}_s = 0$, and $(n)$ bounds the finite geometric sum by the infinite geomtric sum.
% Finally, note that the $(n)$ step in~\eqref{eq:prf_intracons_9} requires that $\eta < 1 - \lambda_2(\boldsymbol{A}_s)$, which concludes the proof. 
Finally, noting that $q \leq \tau_a - 1$ then yields
\begin{align}
    \left \Vert \Delta^{\tilde{k},\tau_a}_s \right\Vert \leq (\lambda_2(\boldsymbol{A}_s)+\eta)^{\tau_a-1} \left \Vert \Delta^{\tilde{k},0}_s \right\Vert
    + \frac{2 \eta \tau_a B \sqrt{N_s}}{1 - \eta - \lambda_2(\boldsymbol{A}_s)}.
\end{align}

\end{proof}

% \newpage

\section{Proof of Lemma~\ref{thm:inter_cons}}
\label{app_ssec:inter_cons} 

\interCons*

\begin{proof}
Via the global update rule~\eqref{eq:cluster_2gether_update}, we have that
\begin{align}
    & \left\Vert \hat{\Delta}^{\hat{k},q+1} \right\Vert 
    \overset{(a)}{=} 
    \left\Vert \boldsymbol{A} \left( \overline{\boldsymbol{\theta}}^{\hat{k},q} \boldsymbol{1} - {\boldsymbol{\theta}}^{\tilde{k},q} \right) \right\Vert 
    \overset{(b)}{\leq} \lambda_2(\boldsymbol{A}) \left\Vert \hat{\Delta}^{\hat{k},q} \right\Vert 
    \overset{(c)}{\leq} \lambda_2(\boldsymbol{A})^{q} \left\Vert \hat{\Delta}^{\hat{k},0} \right\Vert,
\end{align}
where $(a)$ is the result of~\eqref{eq:cluster_2gether_update}, $(b)$ bounds the spectral norm of $\boldsymbol{A}_s$ by its largest feasible eigenvalue, assuming $\hat{\Delta}^{\hat{k},q} \perp 1$, and $(c)$ expands the recursion.

\end{proof}

\newpage

\section{Proof of Theorem~\ref{thm:combined_conv}}
\label{app_ssec:comb_conv} 

\combConv*

\begin{proof}
% {\color{blue} GO CORRECT THE NOTATION - IT SHOULD BE $\tilde{L}$ not $L$!!!}
Given any global round $k \in \mathcal{K}$, we sum over the global gradients as follows:
\begin{equation} \label{eq:combconv1}
    \sum_{q=0}^{\tau_a + \tau_r - 1} \Vert \nabla \tilde{L}(\boldsymbol{\theta}^{k,q})\Vert^2 
    = 
    \underbrace{\sum_{q=0}^{\tau_a - 1} \Vert \nabla \tilde{L}(\boldsymbol{\theta}^{k,q})\Vert^2}_{(i)}
    + 
    \underbrace{\sum_{q=\tau_a}^{\tau_r + \tau_a - 1} \Vert \nabla \tilde{L}(\boldsymbol{\theta}^{k,q})\Vert^2}_{(ii)},
\end{equation}
where $(i)$ is the intra-cluster regime $\tilde{k}$ with $q \in \{0, \cdots, \tau_a-1\}$, and $(ii)$ is the inter-cluster regime $\hat{k}$ with $q \in \{\tau_a, \cdots, \tau_r + \tau_a - 1 \}$. 
We bound the two components of~\eqref{eq:combconv1} separately, starting with the intra-cluster regime component in~\eqref{eq:combconv1}$(i)$
\begin{align}
    & \sum_{q=0}^{\tau_a - 1} \Vert \nabla \tilde{L}(\boldsymbol{\theta}^{k,q})\Vert^2 
    \overset{(a)}{=} \sum_{s \in \mathcal{S}} \sum_{q=0}^{\tau_a -1} \Vert \nabla \tilde{L}_s(\hat{\boldsymbol{\theta}}^{k,q}_s) \Vert^2
    \\
    & \overset{(b)}{\leq} 
    \sum_{s \in \mathcal{S}}
    \frac{\tilde{L}_s(\hat{\boldsymbol{\theta}}^{k,0}_s) + 
    \frac{\alpha \tau_a \eta^2}{2} \Gamma_s}{\eta - \frac{\hat{\alpha}_s \eta^2}{2} \Gamma_s} \\
    & \overset{(c)}{\leq} 
    \frac{\sum_{s \in \mathcal{S}} \left( \tilde{L}_s(\hat{\boldsymbol{\theta}}^{k,0}_s) + 
    \frac{\alpha \tau_a \eta^2}{2} \Gamma_s \right)}{\eta - \frac{\hat{\alpha} \eta^2}{2} \Gamma} \\
    & \label{eq:combconv2}
    \overset{(d)}{\leq} 
    \frac{ \tilde{L}(\boldsymbol{\theta}^{k,0}) + \sum_{s \in \mathcal{S}} 
    \frac{\alpha \tau_a \eta^2}{2} \Gamma_s }{\eta - \frac{\hat{\alpha} \eta^2}{2} \Gamma}    
\end{align}
where $(a)$ is from $\Vert \nabla \tilde{L} (\boldsymbol{\theta}^{k,q}) \Vert^2 = \left( \nabla \tilde{L} (\boldsymbol{\theta}^{k,q}) \right)^T \nabla \tilde{L} (\boldsymbol{\theta}^{k,q}) = \sum_{s \in \mathcal{S}} \left( \nabla \tilde{L}_s (\hat{\boldsymbol{\theta}}^{k,q}_s) \right)^T \nabla \tilde{L}_s (\hat{\boldsymbol{\theta}}^{k,q}_s)$, %$ = \sum_{s \in \mathcal{S}} \Vert \nabla \tilde{L}_s (\hat{\boldsymbol{\theta}}^{k,q}_s) \Vert^2$, 
$(b)$ follows from Theorem~\ref{thm:intra_conv},
$(c)$ uses the fact that $\Gamma \geq \Gamma_s$ and $\hat{\alpha}_s \geq \hat{\alpha}$ so that $\eta - \frac{\hat{\alpha}_s\eta^2}{2}\Gamma_s \geq \eta - \frac{\hat{\alpha} \eta^2}{2}\Gamma $,
and $(d)$ uses the definition of $\tilde{L}(\cdot)$ from~\eqref{eq:eff_g_loss}. 
Next, for term $(ii)$ in~\eqref{eq:combconv1}, we start by leveraging Proposition~\ref{thm:smoothness_reg_loss}, as follows: 
\begin{align} \label{eq:combconv3}
    & \Vert \nabla \tilde{L}(\boldsymbol{\theta}^{k,q}) - \nabla \tilde{L}(\boldsymbol{\theta}^{k,q-1})\Vert 
    \leq \left(\hat{\gamma} + 1 + \frac{1}{\eta}(1-\lambda_N({\boldsymbol{A}})) \right) \Vert \boldsymbol{\theta}^{k,q} - \boldsymbol{\theta}^{k,q-1} \Vert + \tau_a B \sqrt{N}.
\end{align}
Applying the triangle inequality to the left hand side of~\eqref{eq:combconv3} and rearranging yields
\begin{align} \label{eq:combconv4}
    \Vert \nabla \tilde{L}(\boldsymbol{\theta}^{k,q}) \Vert \leq \Vert \nabla \tilde{L}(\boldsymbol{\theta}^{k,q-1}) \Vert + \left(\hat{\gamma} + 1 + \frac{1}{\eta}(1-\lambda_N({\boldsymbol{A}})) \right) \Vert \boldsymbol{\theta}^{k,q} - \boldsymbol{\theta}^{k,q-1} \Vert + \tau_a B \sqrt{N}
\end{align}
Next, we exploit the definition of $\hat{\Delta}^{k,q}$ in Lemma~\ref{thm:inter_cons} to obtain
% in~\eqref{eq:defn_mean_theta_diffs} to obtain 
\begin{equation} \label{eq:combconv5}
    \boldsymbol{\theta}^{k,q} = \overline{\boldsymbol{\theta}}^{k,q} \boldsymbol{1} - \hat{\Delta}^{k,q} 
\end{equation}
and 
\begin{equation} \label{eq:combconv6}
    \tilde{\boldsymbol{A}} \hat{\Delta}^{k,q} = \tilde{\boldsymbol{A}} \overline{\boldsymbol{\theta}}^{k,q} \boldsymbol{1} - \tilde{\boldsymbol{A}} \boldsymbol{\theta}^{k,q} \overset{(e)}{=} \overline{\boldsymbol{\theta}}^{k,q} \boldsymbol{1} - \boldsymbol{\theta}^{k,q+1},
\end{equation}
where~\eqref{eq:combconv6} holds only for $q \in \{\tau_a, \cdots, \tau_a + \tau_r - 1 \}$ and $(e)$ is from the fact that ${\boldsymbol{A}}$ is doubly stochastic. 
Taking the difference between~\eqref{eq:combconv5} and~\eqref{eq:combconv6} gives
\begin{equation} \label{eq:combconv7}
    \boldsymbol{\theta}^{k,q+1} - \boldsymbol{\theta}^{k,q} = \hat{\Delta}^{k,q} - {\boldsymbol{A}} \hat{\Delta}^{k,q} 
    = (\boldsymbol{I} - {\boldsymbol{A}}) \hat{\Delta}^{k,q},
\end{equation}
which can be bounded above by 
\begin{align} \label{eq:combconv8}
    & \Vert \boldsymbol{\theta}^{k,q+1} - \boldsymbol{\theta}^{k,q} \Vert 
    \overset{(f)}{\leq} \Vert \boldsymbol{I} - {\boldsymbol{A}} \Vert  \Vert \hat{\Delta}^{k,q} \Vert \\
    & 
    \label{eq:combconv9}
    \overset{(g)}{\leq}  \left( 1 - \lambda_N({\boldsymbol{A}}) \right) \lambda_2({\boldsymbol{A}})^{q-1} \Vert \hat{\Delta}^{k,0} \Vert, 
\end{align}
where $(f)$ is from norm of the right hand side of~\eqref{eq:combconv7}, and $(g)$ takes the largest eigenvalue of $\boldsymbol{I} - {\boldsymbol{A}}$ and leverages Lemma~\ref{thm:inter_cons}. 
Substituting~\eqref{eq:combconv9} into~\eqref{eq:combconv4} enables the following: 
\begin{align} \label{eq:combconv10}
    \Vert \nabla \tilde{L}(\boldsymbol{\theta}^{k,q}) \Vert 
    \leq \Vert \nabla \tilde{L}(\boldsymbol{\theta}^{k,q-1}) \Vert 
    + 
    \left(\hat{\gamma} + 1 + \frac{1}{\eta}(1-\lambda_N({\boldsymbol{A}})) \right) \left( 1 - \lambda_N({\boldsymbol{A}}) \right) \lambda_2({\boldsymbol{A}})^{q-1} \Vert \hat{\Delta}^{k,0} \Vert + \tau_a B \sqrt{N}, 
\end{align}
Expanding the recursive relationship in~\eqref{eq:combconv10} then yields
\begin{align} \label{eq:combconv11}
    \Vert \nabla \tilde{L}(\boldsymbol{\theta}^{k,q}) \Vert \leq \Vert \nabla \tilde{L}(\boldsymbol{\theta}^{k,0} )\Vert 
    + \left(\hat{\gamma} + 1 + \frac{1}{\eta}(1-\lambda_N({\boldsymbol{A}})) \right) \left( 1 - \lambda_N({\boldsymbol{A}}) \right)
    \sum_{p=0}^{q-1} \lambda_2({\boldsymbol{A}})^{p} \Vert \hat{\Delta}^{k,0} \Vert  
    + (q-1) \tau_a B \sqrt{N},
\end{align} 
and, after squaring both sides, 
\begin{align}
    & \Vert \nabla \tilde{L}(\boldsymbol{\theta}^{k,q}) \Vert^2 
    \leq \left\{ \Vert \nabla \tilde{L}(\boldsymbol{\theta}^{k,0} )\Vert 
    + \left(\hat{\gamma} + 1 + \frac{1}{\eta}(1-\lambda_N({\boldsymbol{A}})) \right) \left( 1 - \lambda_N({\boldsymbol{A}}) \right)
    \sum_{p=0}^{q-1} \lambda_2({\boldsymbol{A}})^{p} \Vert \hat{\Delta}^{k,0} \Vert  
    + (q-1) \tau_a B \sqrt{N} \right\}^2 \\
    & 
    \label{eq:combconv12}
    \overset{(h)}{\leq} 
    2 \Vert \nabla \tilde{L}(\boldsymbol{\theta}^{k,0} )\Vert^2 + 4 \left(\hat{\gamma} + 1 + \frac{1}{\eta}(1-\lambda_N({\boldsymbol{A}})) \right)^2 \left( 1 - \lambda_N({\boldsymbol{A}}) \right)^2 \left(\frac{1}{1 - \lambda_2({\boldsymbol{A}})} \Vert \hat{\Delta}^{k,0} \Vert \right)^2 + 4 \tau_a^2 B^2 N (q-1)^2, 
\end{align}
where $(h)$ follows from $(a+b)^2 \leq 2a^2 + 2 b^2$ applied twice and the fact that $\sum_{p=0}^{q-1} \lambda_2({\boldsymbol{A}})^p \leq \sum_{p=0}^{\infty} \lambda_2({\boldsymbol{A}})^p = \frac{1}{1- \lambda_2({\boldsymbol{A}})}$. 
Summing~\eqref{eq:combconv12} over $q \in \{ \tau_a, \cdots, \tau_a + \tau_r - 1 \}$ yields
\begin{align}
    & \sum_{q = \tau_a}^{\tau_a + \tau_r - 1} \Vert \nabla \tilde{L} (\boldsymbol{\theta}^{k,q}) \Vert^2 
    \overset{(i)}{\leq} 
    2 (\tau_r - 1) \Vert \nabla \tilde{L}(\boldsymbol{\theta}^{k,0} )\Vert^2 
    \nonumber \\ 
    & 
    + 4 (\tau_r - 1) \left(\hat{\gamma} + 1 + \frac{1}{\eta}(1-\lambda_N({\boldsymbol{A}})) \right)^2 \left( 1 - \lambda_N({\boldsymbol{A}}) \right)^2 \left(\frac{\Vert \hat{\Delta}^{k,0} \Vert}{1 - \lambda_2({\boldsymbol{A}})}  \right)^2
    + 4 \tau_a^2 B^2 N \sum_{q = \tau_a}^{\tau_a + \tau_r - 1} (q-1)^2, \\
    & 
    \label{eq:combconv13}
    \overset{(j)}{\leq} 
    2 (\tau_r - 1) \Vert \nabla \tilde{L}(\boldsymbol{\theta}^{k,0} )\Vert^2 
    + 4 (\tau_r - 1) \left(\hat{\gamma} + 1 + \frac{1}{\eta}(1-\lambda_N({\boldsymbol{A}})) \right)^2 \left( 1 - \lambda_N({\boldsymbol{A}}) \right)^2 \left(\frac{\Vert \hat{\Delta}^{k,0} \Vert}{1 - \lambda_2({\boldsymbol{A}})}  \right)^2 \nonumber \\
    & 
    + 4 \tau_a^2 B^2 N \tau_r (\tau_a + \tau_r - 1)^2,
\end{align}
where $(i)$ expands the summation over $q$ for non-$q$ dependent terms, and $(j)$ results from $\sum_{q=\tau_a}^{\tau_a + \tau_r - 1} (q-1)^2 \leq \sum_{q=\tau_a}^{\tau_a + \tau_r - 1} q^2 \leq \tau_r (\tau_a+\tau_r-1)^2$. 
Returning to~\eqref{eq:combconv1}, we combine the bounds for the intra-cluster and the inter-cluster terms as follows: 
\begin{align}
    & \sum_{q=0}^{\tau_a + \tau_r - 1} \Vert \nabla \tilde{L}(\boldsymbol{\theta}^{k,q})\Vert^2 
    \leq 
    \frac{ \tilde{L}(\boldsymbol{\theta}^{k,0}) + \sum_{s \in \mathcal{S}} 
    \frac{\alpha \tau_a \eta^2}{2} \Gamma_s }{\eta - \frac{\hat{\alpha} \eta^2}{2} \Gamma}  
    + 
    2 (\tau_r - 1) \Vert \nabla \tilde{L}(\boldsymbol{\theta}^{k,0} )\Vert^2 \nonumber \\
    & 
    + 4 (\tau_r - 1) \left(\hat{\gamma} + 1 + \frac{1}{\eta}(1-\lambda_N({\boldsymbol{A}})) \right)^2 \left( 1 - \lambda_N({\boldsymbol{A}}) \right)^2 \left(\frac{\Vert \hat{\Delta}^{k,0} \Vert}{1 - \lambda_2({\boldsymbol{A}})} \right)^2
    + 4 \tau_a^2 B^2 N \tau_r (\tau_a + \tau_r - 1)^2 \\
    & 
    \label{eq:combconv14}
    \overset{(k)}{\leq} 
    \frac{ 2\tau_r \tilde{L}(\boldsymbol{\theta}^{k,0}) + (\tau_a + 2 (\tau_r -1)) \sum_{s \in \mathcal{S}} 
    \frac{\alpha \eta^2}{2} \Gamma_s }{\eta - \frac{\hat{\alpha} \eta^2}{2} \Gamma}  \nonumber \\ 
    & 
    + 4 (\tau_r - 1) \left(\hat{\gamma} + 1 + \frac{1}{\eta}(1-\lambda_N({\boldsymbol{A}})) \right)^2 \left( 1 - \lambda_N({\boldsymbol{A}}) \right)^2 \left(\frac{\Vert \hat{\Delta}^{k,0} \Vert}{1 - \lambda_2({\boldsymbol{A}})} \right)^2
    + 4 \tau_a^2 B^2 N \tau_r (\tau_a + \tau_r - 1)^2,
\end{align}
where $(k)$ leverages Theorem~\ref{thm:intra_conv} with $\tau_a = 1$ and. 
Finally, noting that $\gamma^{\mathsf{eff}} = \hat{\gamma} + 1 + \frac{1}{\eta}(1-\lambda_N(\tilde{\boldsymbol{A}}))$ and re-arranging~\eqref{eq:combconv14} completes the proof. 

\end{proof}

\newpage

\section{Additional Experiments} \label{app_sec:exps} 
As indicated within the main manuscript, we further evaluate SSD-FL 
by varying link probabilities when the underlying network graph is an Erdős–Rényi random graph~\cite{gilbert1959random} for both heterogeneous and homogeneous device ML optimizers in Appendix~\ref{app_ssec:linkp_exps}.
Subsequently, we examine SSD-FL when network devices have homogeneous SGD optimizers in appendix~\ref{app_ssec:all_sgds} and further examine the properties of the bound in Theorem~\ref{thm:intra_conv} via investigating the variation in normalized intra-cluster gradients across datasets and local device ML optimizers in Appendix~\ref{app_ssec:intra_grad_evals}.

\subsection{Varying Link Probabilities} \label{app_ssec:linkp_exps}
We examine the impact of increasing link formation probabilities from $10\%$ to $50\%$ in random graphs in Fig.~\ref{fig:lp_ovr_hybrid} and~\ref{fig:lp_ovr_sgd}.
We do want to emphasize that, when link formation probability is $100\%$, the random graph has equivalent structure to the complete graphs shown in Sec.~\ref{ssec:diff_net_styles}.

For the case with heterogeneous ML optimizers at devices in Fig.~\ref{fig:lp_ovr_hybrid}, we see that SSD-FL either outperforms or matches the final accuracies of the baseline decentralized FL methodologies. 
Moreover, relative to the sDFL and pDFL methodologies, SSD-FL maintains a similar sized performance gap, roughly $4\%$ and $13\%$ respectively, regardless of the link formation probability and dataset. 
Experiments with homogeneous ML optimizers in Fig.~\ref{fig:lp_ovr_sgd} yield the similar takeaways. 

% To summarize Fig.~\ref{fig:lp_ovr_hybrid}, increased link probabilities has similar effects across the baselines as that 

% don't need to discuss the sgd trends, just say that general takeaways remain identical to that of heterogeneous ML optimizers

\begin{figure}[h]
    \centering
    \begin{subfigure}[t]{0.495\linewidth}
        \centering
        \includegraphics[width=\linewidth]{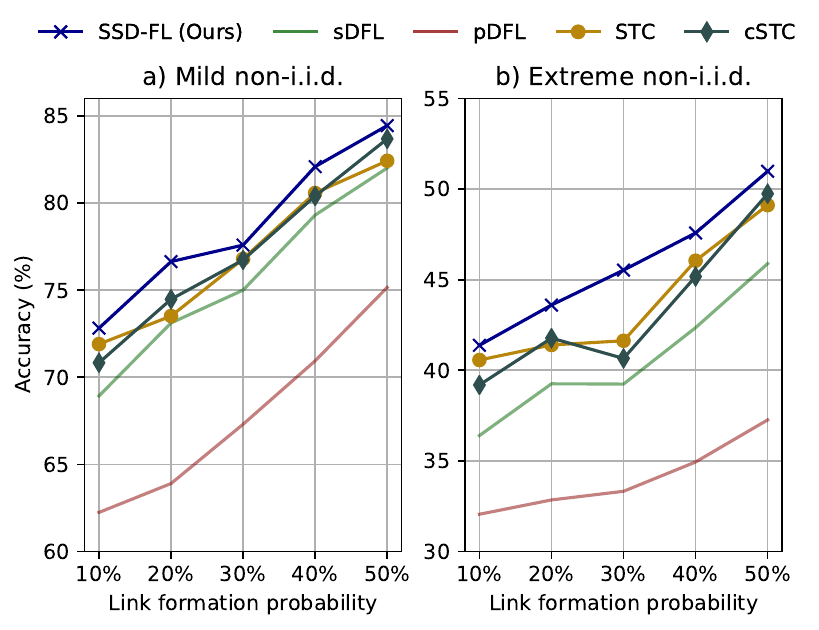}
        \caption{Experiments on the FMNIST dataset.}
        \label{fig:lp_fmnist_hybrid}
    \end{subfigure}
    \hfill
    \begin{subfigure}[t]{0.495\linewidth}
        \centering
        \includegraphics[width=\linewidth]{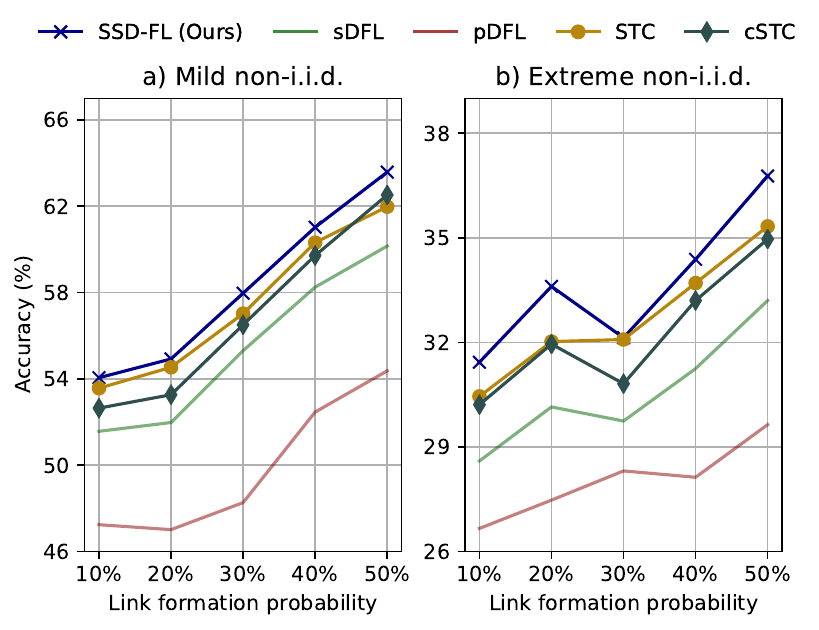}
        \caption{Experiments on the CIFAR10 dataset.}
        \label{fig:lp_cifar10_hybrid}
    \end{subfigure}
    \caption{Varying link formation probabilities from $10\%$ to $50\%$ for Erdős–Rényi random graphs. We evaluate both (a) FMNIST and (b) CIFAR10 datasets for networks with $30$ devices, $\tau_a = 3$, $\tau_r = 1$, and heterogeneous ML optimizers at devices.}
    \label{fig:lp_ovr_hybrid}
\end{figure}

\begin{figure}[h]
    \centering
    \begin{subfigure}[t]{0.495\linewidth}
        \centering
        \includegraphics[width=\linewidth]{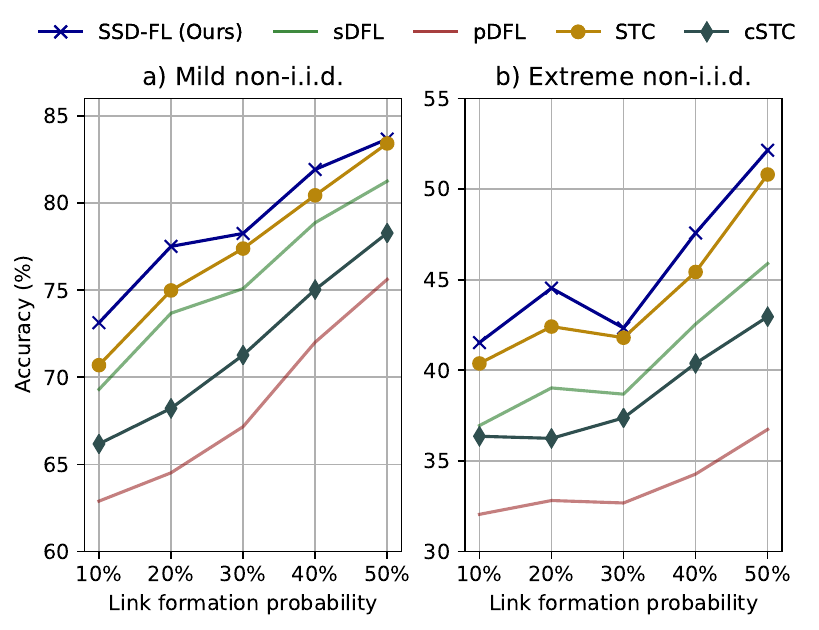}
        \caption{Experiments on the FMNIST dataset.}
        \label{fig:lp_fmnist_sgd}
    \end{subfigure}
    \hfill
    \begin{subfigure}[t]{0.495\linewidth}
        \centering
        \includegraphics[width=\linewidth]{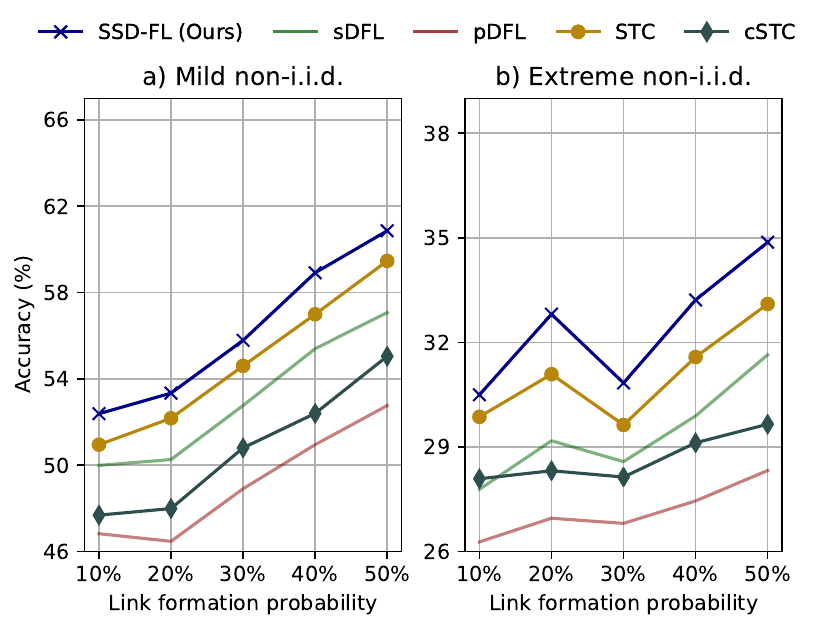}
        \caption{Experiments on the CIFAR10 dataset.}
        \label{fig:lp_cifar10_sgd}
    \end{subfigure}
    \caption{Varying link formation probabilities from $10\%$ to $50\%$ for Erdős–Rényi random graphs with homogeneous SGD optimizers at devices. The experimental setup is the same as that in Fig.~\ref{fig:lp_ovr_hybrid} aside from the choice of ML optimizers and, while the accuracies are lower, especially for CIFAR10 in Fig.~\ref{fig:lp_cifar10_sgd}, the key takeaways remain identical.}
    \label{fig:lp_ovr_sgd}
\end{figure}

\subsection{Homogeneous SGD Optimizers} \label{app_ssec:all_sgds}
We list the experimental results for experiments varying intra-cluster duration $\tau_a$ in Fig.~\ref{fig:sgd_taua}, inter-cluster period $\tau_r$ in Tables~\ref{tab:acc_thresholds_11}-\ref{tab:acc_thresholds_sgd_fmnist}, underlying network graph architectures in Fig.~\ref{fig:sgd_ns2_ovr}, and network size in Fig.~\ref{fig:sgd_ns_ovr}. 
In particular, regarding the inter-cluster period experiments, Tables~\ref{tab:acc_thresholds_11} and~\ref{tab:cifar10_sgd_11} show results for networks with homogeneous SGD and heterogeneous optimizers across their devices, but only for $\tau_r=1$. 
% show the global rounds needed to reach accuracy thresholds when $\tau_r =1$
While the exact numerical results and convergence curves may differ, the core takeaways remain the same as those from heterogeneous local ML optimizers. 

% \subsubsection{Intra-cluster duration $\tau_a$}
\begin{figure}[h]
    \centering
    \begin{subfigure}[t]{0.495\linewidth}
        \centering
        \includegraphics[width=\linewidth]{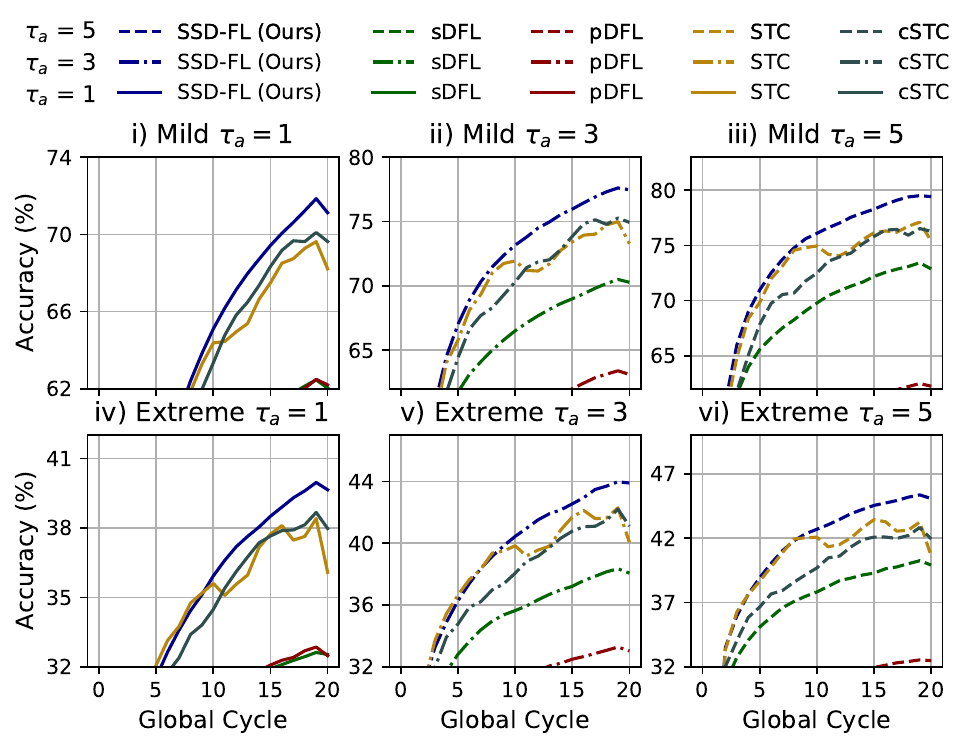}
        \caption{Experiments on the FMNIST dataset.}
        \label{fig:sgd_taua_fmnist}
    \end{subfigure}
    \hfill
    \begin{subfigure}[t]{0.495\linewidth}
        \centering
        \includegraphics[width=\linewidth]{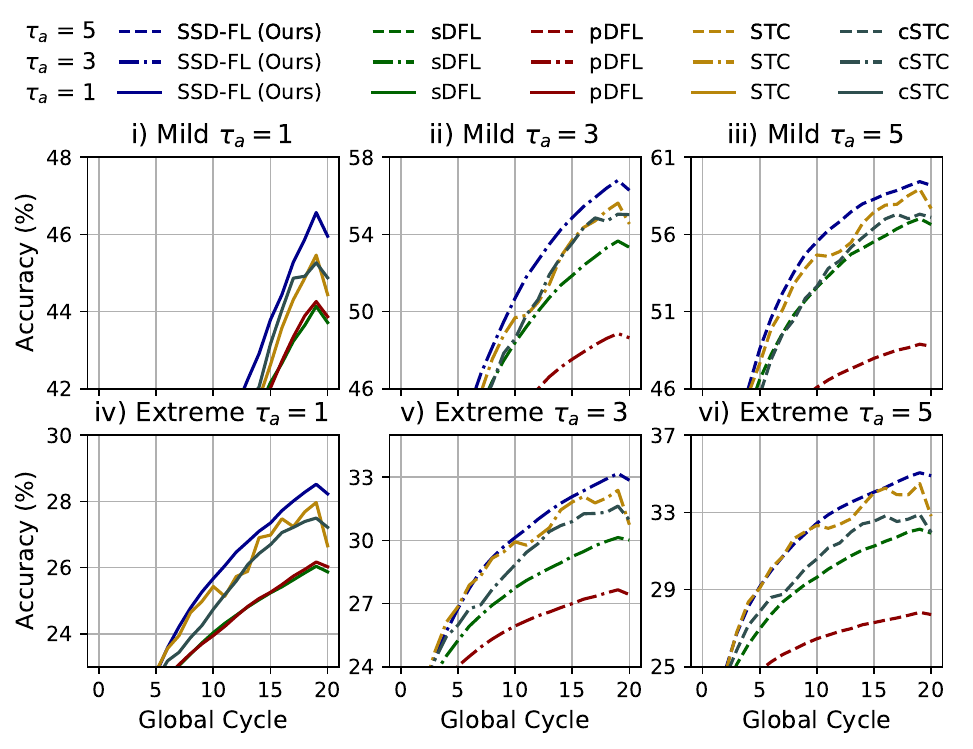}
        \caption{Experiments on the CIFAR10 dataset.}
        \label{fig:sgd_taua_cifar}
    \end{subfigure}
    \caption{Varying intra-cluster period $\tau_a$ for random networks with of size $N=10$ with homogeneous SGD optimizers at devices. Both FMNIST, in Fig.~\ref{fig:sgd_taua_fmnist}, and CIFAR10, in Fig.~\ref{fig:sgd_taua_cifar}, are investigated for mild and extreme non-i.i.d. scenarios.}
    \label{fig:sgd_taua}
\end{figure}

\begin{table}[H]
\caption{The average global cycles for methods to reach accuracy thresholds on FMNIST with $\tau_r=1$. Networks with both SGD and heterogeneous optimizers are investigated. Dashes indicate thresholds that were not reached.}
{\footnotesize
\begin{tabularx}{0.99\textwidth}{m{3.5em} *{16}{X} }
\toprule[.2em]
\multirow{3}{*}{\textbf{Method}} 
& \multicolumn{8}{c}{\textbf{SGD optimizers}}
& \multicolumn{8}{c}{\textbf{Hybrid optimizers}}
\\
\cmidrule(lr){2-9} \cmidrule(lr){10-17}
& \multicolumn{4}{c}{\textbf{Mild non-i.i.d. acc}} 
& \multicolumn{4}{c}{\textbf{Extreme non-i.i.d. acc}}
& \multicolumn{4}{c}{\textbf{Mild non-i.i.d. acc}} 
& \multicolumn{4}{c}{\textbf{Extreme non-i.i.d. acc}}
\\
\cmidrule(lr){2-5} \cmidrule(lr){6-9} \cmidrule(lr){10-13} \cmidrule(lr){14-17}
& 51\% & 58\% & 65\% & 72\%
& 30\% & 35\% & 40\% & 45\% 
& 51\% & 58\% & 65\% & 72\%
& 30\% & 35\% & 40\% & 45\%
\\
\midrule 
SSD-FL & 4.02 & 6.35 & 10.45 & 19.71 & 4.16 & 8.79 & 17.93 & -- & 3.50 & 5.45 & 9.24 & 19.46 & 3.48 & 7.11 & 16.71 & -- \\
ONE    & 5.29 & 10.16 & 18.65 & --    & 7.73 & --    & --    & -- & 5.01 & 10.77 & --   & --    & 8.79 & --   & --    & --  \\
ALL    & 5.39 & 9.95  & 19.30 & --    & 7.58 & 19.79 & --    & -- & 5.06 & 10.38 & --   & --    & 8.91 & --   & --    & -- \\
RGW    & 4.39 & 7.45  & 13.61 & --    & 4.67 & 13.07 & 19.94 & -- & 3.64 & 5.69  & 11.18 & --   & 3.68 & 9.34 & 19.37 & -- \\
RGP    & 5.02 & 8.59  & 17.11 & --    & 5.19 & 14.51 & --    & -- & 3.81 & 5.90  & 10.29 & --   & 3.71 & 8.97 & 18.45 & -- \\
\bottomrule
\end{tabularx}
\label{tab:acc_thresholds_11}
}
\end{table}

%%% sgd optimizers only 
\begin{table}[H]
\caption{Average global cycles when methods reach or exceed accuracy points on FMNIST when networks have homogeneous SGD optimizers at devices. Dashes indicate thresholds that were not reached.}
{\footnotesize
\begin{tabularx}{0.99\textwidth}{m{3.5em} *{16}{X} }
\toprule[.2em]
\multirow{3}{*}{\textbf{Method}} 
& \multicolumn{8}{c}{\textbf{$\tau_r = 3$}}
& \multicolumn{8}{c}{\textbf{$\tau_r = 5$}}
\\
\cmidrule(lr){2-9} \cmidrule(lr){10-17}
& \multicolumn{4}{c}{\textbf{Mild non-i.i.d. acc}} 
& \multicolumn{4}{c}{\textbf{Extreme non-i.i.d. acc}}
& \multicolumn{4}{c}{\textbf{Mild non-i.i.d. acc}} 
& \multicolumn{4}{c}{\textbf{Extreme non-i.i.d. acc}}
\\
\cmidrule(lr){2-5} \cmidrule(lr){6-9} \cmidrule(lr){10-13} \cmidrule(lr){14-17}
& 51\% & 58\% & 65\% & 72\%
& 30\% & 35\% & 40\% & 45\% 
& 51\% & 58\% & 65\% & 72\%
& 30\% & 35\% & 40\% & 45\%
\\
\midrule 
SSD-FL & 3.34 & 4.77 & 7.29  & 12.72 & 3.12 & 5.64 & 9.82  & 18.27 & 2.98 & 4.14 & 6.31  & 10.57 & 2.72 & 4.76 & 7.86  & 14.07 \\ 
ONE    & 3.57 & 5.23 & 8.29  & 14.79 & 3.37 & 6.51 & 11.91 & -- & 3.10 & 4.40 & 6.68  & 11.34 & 2.75 & 5.06 & 8.62  & 15.51 \\ 
ALL    & 3.56 & 5.27 & 8.42  & 14.53 & 3.40 & 6.47 & 12.02 & -- & 3.08 & 4.39 & 6.67  & 11.37 & 2.74 & 5.01 & 8.32  & 15.62 \\ 
RGW    & 3.72 & 6.22 & 9.05  & 18.16 & 3.66 & 8.20 & 14.80 & -- & 3.61 & 5.86 & 8.50  & 17.69 & 3.65 & 7.61 & 14.21 & -- \\ 
RGP    & 4.12 & 5.94 & 10.04 & 17.87 & 4.12 & 6.91 & 15.44 & -- & 3.84 & 5.42 & 9.23  & 15.91 & 3.65 & 5.87 & 11.84 & -- \\
\bottomrule
\end{tabularx}
\label{tab:acc_thresholds_sgd_fmnist}
}
\end{table}

\begin{table}[H]
\caption{The average global cycles for methods to reach accuracy thresholds on CIFAR10 with $\tau_r=1$. Networks with both SGD and heterogeneous optimizers are investigated. Dashes indicate thresholds that were not reached.}
\footnotesize
\begin{tabularx}{0.99\textwidth}{m{3.5em} *{16}{X} }
\toprule[.2em]
\multirow{3}{*}{\textbf{Method}} 
& \multicolumn{8}{c}{\textbf{SGD optimizers}}
& \multicolumn{8}{c}{\textbf{Hybrid optimizers}}
\\
\cmidrule(lr){2-9} \cmidrule(lr){10-17}
& \multicolumn{4}{c}{\textbf{Mild non-i.i.d. acc}} 
& \multicolumn{4}{c}{\textbf{Extreme non-i.i.d. acc}}
& \multicolumn{4}{c}{\textbf{Mild non-i.i.d. acc}} 
& \multicolumn{4}{c}{\textbf{Extreme non-i.i.d. acc}}
\\
\cmidrule(lr){2-5} \cmidrule(lr){6-9} \cmidrule(lr){10-13} \cmidrule(lr){14-17}
& 36\% & 40\% & 44\% & 48\%
& 22\% & 25.5\% & 29\% & 32.5\% 
& 36\% & 40\% & 44\% & 48\%
& 22\% & 25.5\% & 29\% & 32.5\%
\\
\midrule
SSD-FL & 8.35 & 13.12 & 18.91 & -- & 3.98 & 9.95 & -- & -- & 7.38 & 10.96 & 16.15 & -- & 3.79 & 8.29 & 16.76 & -- \\
ONE    & 8.74 & 14.35 & --    & -- & 4.62 & 17.77 & -- & -- & 7.12 & 12.29 & 19.92 & -- & 4.38 & 15.72 & --    & -- \\
ALL    & 8.50 & 14.68 & --    & -- & 4.78 & 17.47 & -- & -- & 7.29 & 12.13 & 19.26 & -- & 4.51 & 15.53 & --    & -- \\
RGW    & 8.78 & 13.51 & 18.99 & -- & 4.33 & 9.72  & -- & -- & 7.70 & 12.32 & 16.87 & -- & 3.85 & 9.82  & --    & -- \\
RGP    & 8.89 & 15.01 & --    & -- & 4.79 & 14.44 & -- & -- & 7.59 & 11.53 & 16.39 & -- & 3.57 & 11.16 & --    & -- \\
\bottomrule
\end{tabularx}
\label{tab:cifar10_sgd_11}
\end{table}

%% sgd table
\begin{table}[H]
\caption{Average global cycles required for methods to reach or exceed target accuracies on CIFAR-10 with homogeneous SGD optimizers at devices. Dashes indicate thresholds that were not reached.}
{\footnotesize
\begin{tabularx}{0.99\textwidth}{m{3.5em} *{16}{X} }
\toprule[.2em]
\multirow{3}{*}{\textbf{Method}} 
& \multicolumn{8}{c}{\textbf{$\tau_r = 3$}}
& \multicolumn{8}{c}{\textbf{$\tau_r = 5$}}
\\
\cmidrule(lr){2-9} \cmidrule(lr){10-17}
& \multicolumn{4}{c}{\textbf{Mild non-i.i.d. acc}} 
& \multicolumn{4}{c}{\textbf{Extreme non-i.i.d. acc}}
& \multicolumn{4}{c}{\textbf{Mild non-i.i.d. acc}} 
& \multicolumn{4}{c}{\textbf{Extreme non-i.i.d. acc}}
\\
\cmidrule(lr){2-5} \cmidrule(lr){6-9} \cmidrule(lr){10-13} \cmidrule(lr){14-17}
& 51\% & 58\% & 65\% & 72\%
& 30\% & 35\% & 40\% & 45\% 
& 51\% & 58\% & 65\% & 72\%
& 30\% & 35\% & 40\% & 45\%
\\
\midrule 
SSD-FL & 7.49 & 10.79 & 15.23 & -- & 3.26 & 7.00 & 13.06 & -- & 6.89 & 9.72  & 14.05 & 18.94 & 2.96 & 5.94 & 11.59 & 19.59 \\ 
ONE    & 7.63 & 11.31 & 16.26 & -- & 3.16 & 7.01 & 14.31 & -- & 7.01 & 10.12 & 14.23 & --    & 2.92 & 6.08 & 11.96 & -- \\
ALL    & 7.60 & 11.36 & 16.26 & -- & 3.10 & 7.25 & 14.65 & -- & 6.96 & 10.22 & 14.35 & --    & 2.93 & 6.32 & 11.94 & -- \\ 
RGW    & 7.96 & 11.28 & 17.24 & -- & 2.97 & 8.45 & 15.82 & -- & 7.88 & 11.04 & 16.91 & --    & 2.98 & 8.21 & 14.97 & -- \\ 
RGP    & 7.95 & 12.38 & 17.91 & -- & 4.68 & 9.52 & 19.53 & -- & 7.55 & 11.08 & 16.91 & --    & 4.46 & 8.62 & 17.39 & -- \\
\bottomrule
\end{tabularx}
\label{tab:acc_thresholds_sgd_fmnist}
}
\end{table}

\begin{figure}[H]
    \centering
    \begin{subfigure}[t]{0.495\linewidth}
        \centering
        \includegraphics[width=\linewidth]{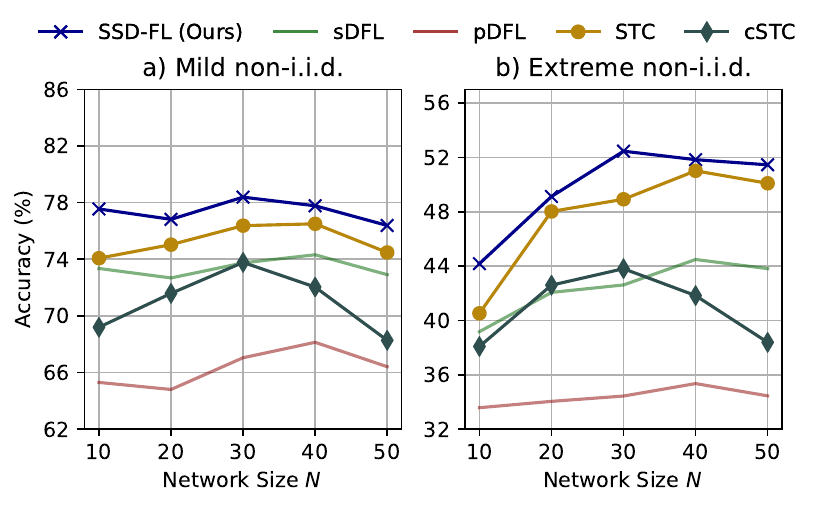}
        \caption{Experiments on the FMNIST dataset.}
        \label{fig:sgd_ns_fmnist}
    \end{subfigure}
    \hfill
    \begin{subfigure}[t]{0.495\linewidth}
        \centering
        \includegraphics[width=\linewidth]{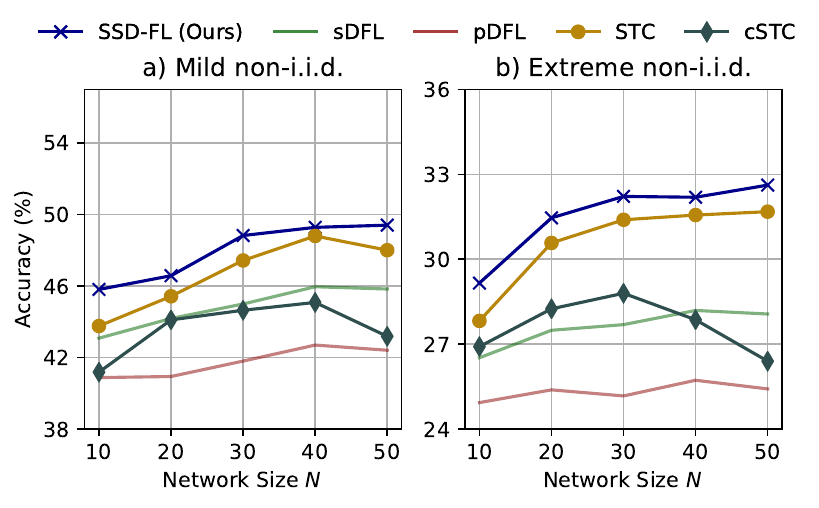}
        \caption{Experiments on the CIFAR10 dataset.}
        \label{fig:sgd_ns_cifar}
    \end{subfigure}
    \caption{Varying network size from $N=10$ to $N=50$ with Erdős–Rényi random graph architecture and homogeneous SGD optimizers at devices. While nominal final accuracies are lower than the case for heterogeneous ML optimizers at devices in Fig.~\ref{fig:het_ns_fmnist} and~\ref{fig:het_ns_cifar}, the main takeaways remain the same.}
    \label{fig:sgd_ns_ovr}
\end{figure}

%%% network architectures
\begin{figure}[H]
    \centering
    \begin{subfigure}[t]{0.495\linewidth}
        \centering
        \includegraphics[width=\linewidth]{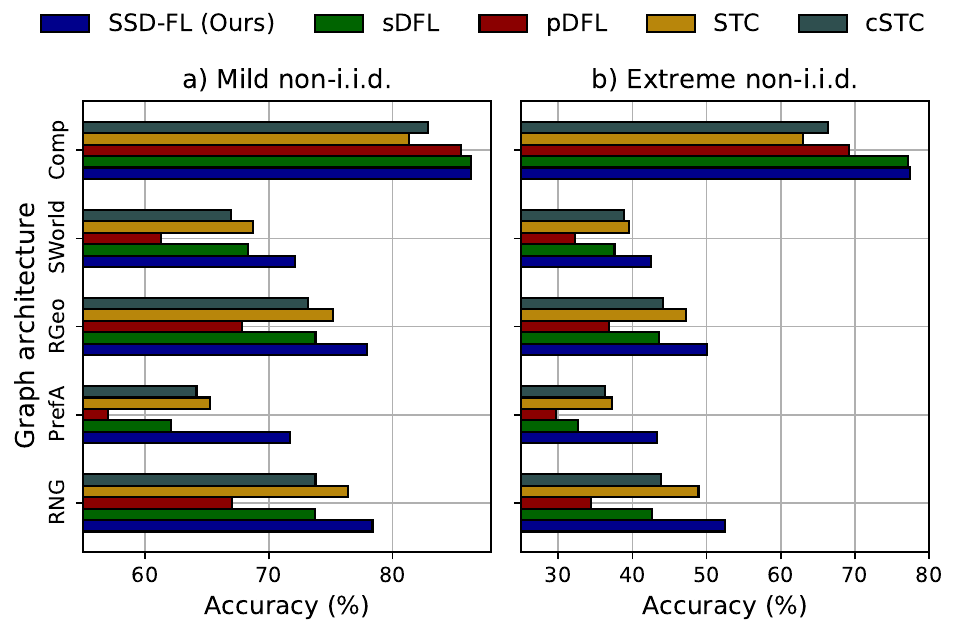}
        \caption{Experiments on the FMNIST dataset.}
        \label{fig:sgd_ns2_fmnist}
    \end{subfigure}
    \hfill
    \begin{subfigure}[t]{0.495\linewidth}
        \centering
        \includegraphics[width=\linewidth]{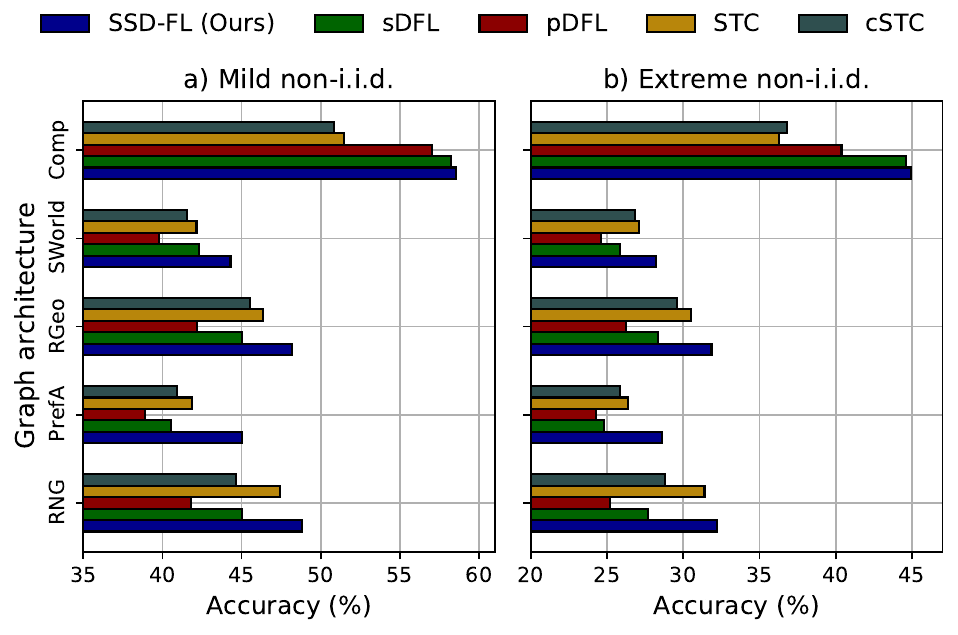}
        \caption{Experiments on the CIFAR10 dataset.}
        \label{fig:sgd_ns2_cifar}
    \end{subfigure}
    \caption{Evaluation of SSD-FL relative to decentralized FL baselines for various network architectures with homogeneous SGD optimizers at all devices. 
    Similar to the experiment involving heterogeneous ML optimizers at devices in Fig.~\ref{fig:het_ns2_fmnist} and~\ref{fig:het_ns2_cifar}, SSD-FL consistently demonstrates superior performance with complete networks being the exception.}
    \label{fig:sgd_ns2_ovr}
\end{figure}

%%%%%%%%%%%%%%%%%%%%%%%
% other section
%%%%%%%%%%%%%%%%%%%%%%%

\subsection{Normalized Intra-Cluster Gradients} \label{app_ssec:intra_grad_evals}
We also investigate the impact of heterogeneous/homogeneous ML optimizers and intra-cluster period $\tau_a$ on the average effective intra-cluster first order stationary point from Theorem~\ref{thm:intra_conv}.
While the nominal differences across datasets and optimizers are small in Fig.~\ref{fig:grad_hybrid} and~\ref{fig:grad_sgd}, there is an important point.
As $\tau_a$ grows, the almost parabolic nature of average gradients shifts, yielding a different optimal number of clusters. 
For instance, in Fig.~\ref{fig:grad_mild_hybrid}, we can see the optimal or minimum point shift from $S=4$ to $S=2$ as $\tau_a$ grows from $1$ to $5$. 
Moreover, we see similar takeaways when comparing heterogeneous and homogeneous ML optimizers or mild vs extreme non-i.i.d. data distributions - these factors lead to minor differences in estimated average effective gradient but they shift the minima and scaling of estimated gradient with respect to the number of clusters.

% We also want to note the difference between heterogeneous and homogeneous ML optimizers and their relationship with mild and extreme non-i.i.d. data distributions at devices in Fig.~\ref{fig:grad_hybrid} and~\ref{fig:grad_sgd}. 
% For mild non-i.i.d. data distributions, homogeneous SGD optimizers leads to a general increase in nominal average effective gradient. 

% Paragraph 1: differences as $\tau_a$ grows

% Paragraph 2: mild vs extreme

% Paragraph 3: heterogeneous vs homogeneous optimizers

%% heterogeneous optimizers
\begin{figure}[H]
    \centering
    \begin{subfigure}[t]{0.495\linewidth}
        \centering
        \includegraphics[width=\linewidth]{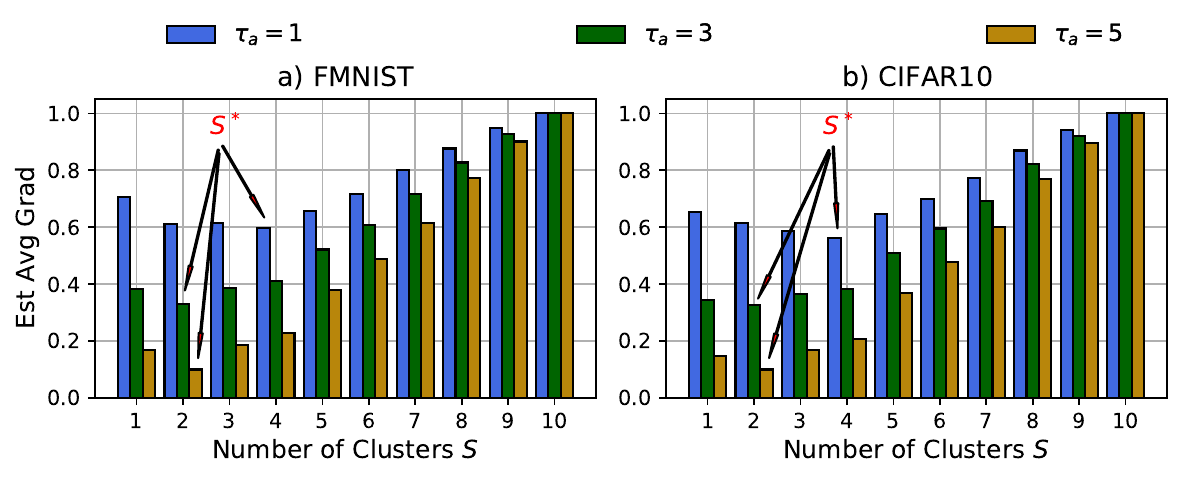}
        \caption{Experiments with mild non-i.i.d. datasets at devices.}
        \label{fig:grad_mild_hybrid}
    \end{subfigure}
    \hfill
    \begin{subfigure}[t]{0.495\linewidth}
        \centering
        \includegraphics[width=\linewidth]{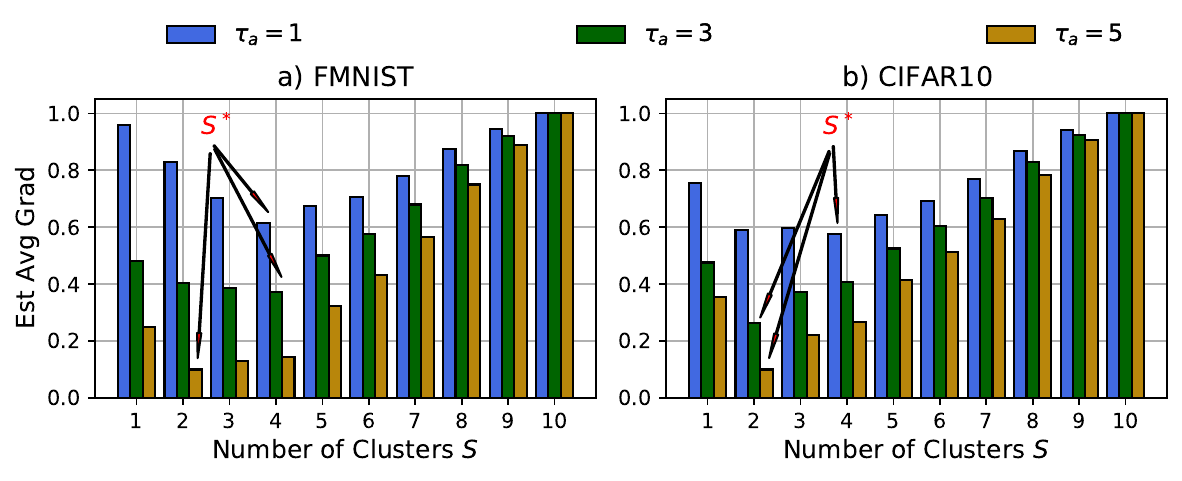}
        \caption{Experiments with extreme non-i.i.d. datasets at devices.}
        \label{fig:grad_extreme_hybrid}
    \end{subfigure}
    \caption{Average intra-cluster effective gradients from Theorem~\ref{thm:intra_conv} for networks with heterogeneous ML optimizers at devices. As the intra-cluster period $\tau_a$ increases, the average intra-cluster effective gradients decrease in relative magnitude. }
    \label{fig:grad_hybrid}
\end{figure}

%% sdgs only
\begin{figure}[H]
    \centering
    \begin{subfigure}[t]{0.495\linewidth}
        \centering
        \includegraphics[width=\linewidth]{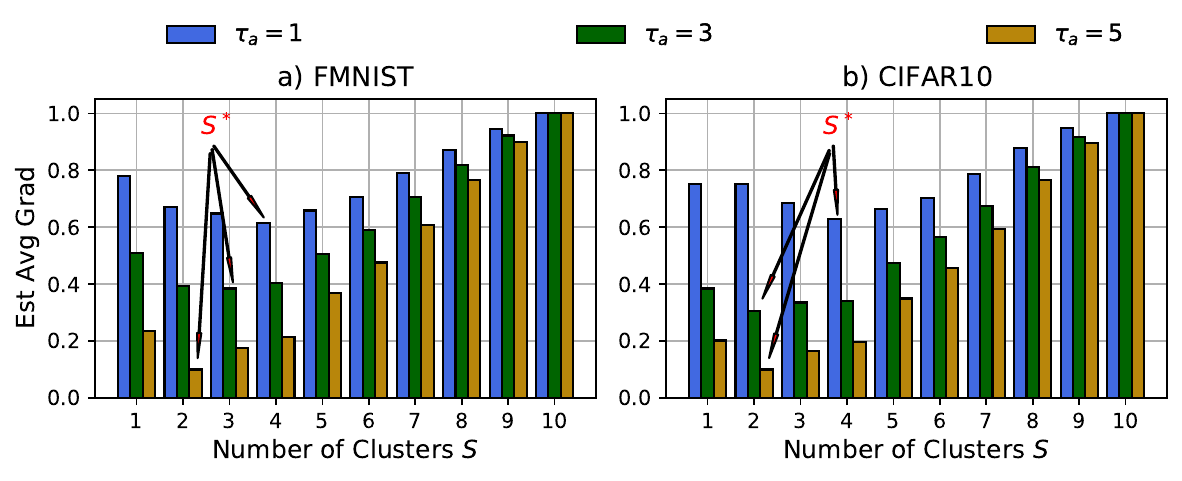}
        \caption{Experiments with mild non-i.i.d. datasets at devices.}
        \label{fig:grad_mild_sgd}
    \end{subfigure}
    \hfill
    \begin{subfigure}[t]{0.495\linewidth}
        \centering
        \includegraphics[width=\linewidth]{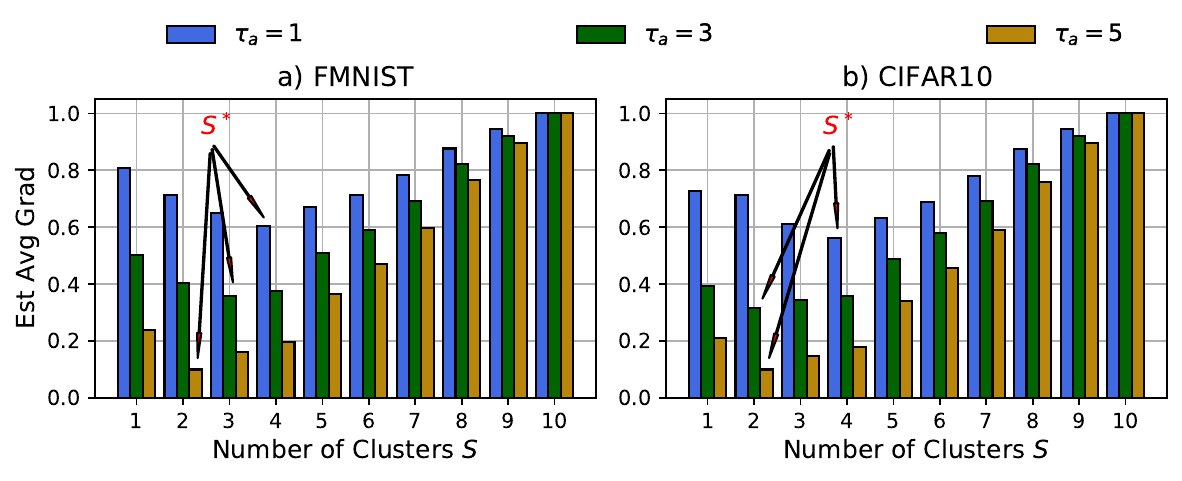}
        \caption{Experiments with extreme non-i.i.d. datasets at devices.}
        \label{fig:grad_extreme_sgd}
    \end{subfigure}
    \caption{The influence of intra-cluster period $\tau_a$ on the average intra-cluster effective gradients for networks with homogeneous SGD optimizers at devices.}
    \label{fig:grad_sgd}
\end{figure}

\newpage

%%% notes - uncomment when need a reference
%% these are self notes
% \input{proofs/modelnotes}
% \newpage

% \input{algos/optim}
% \newpage 

\end{document}